\newif \ifarxiv 
\newif \ifneurips

\arxivtrue
\neuripsfalse

\ifneurips
    \documentclass{article}

    \usepackage{microtype}
    \usepackage{graphicx}
    \usepackage{subcaption}
    \usepackage{booktabs} % for professional tables
    \usepackage{enumitem}
    \usepackage{algorithm}
    \usepackage{algpseudocode}
    \usepackage[
      colorlinks=true,
      linkcolor=blue,
      citecolor=blue,
      urlcolor=blue,
      filecolor=blue
    ]{hyperref}

    \usepackage[numbers,sort&compress]{natbib}
    \usepackage[final]{neurips_2026}

    \makeatletter
    \providecommand{\@trackname}{}
    \makeatother
    
    \usepackage{amsmath}
    \usepackage{amssymb}
    \usepackage{mathtools}
    \usepackage{amsthm}
    \usepackage{thmtools,thm-restate}
    \usepackage{bm}
    \usepackage{comment}
    \usepackage{breqn}
    \usepackage{xcolor}

    \usepackage[capitalize,noabbrev]{cleveref}
    
    \theoremstyle{plain}
    \newtheorem{theorem}{Theorem}[section]
    \newtheorem{proposition}[theorem]{Proposition}
    \newtheorem{lemma}[theorem]{Lemma}
    \newtheorem{corollary}[theorem]{Corollary}
    \theoremstyle{definition}
    \newtheorem{definition}[theorem]{Definition}
    \newtheorem{assumption}[theorem]{Assumption}
    \theoremstyle{remark}
    \newtheorem{remark}[theorem]{Remark}

\newcommand{\ax}[1]{{\color{red} \textbf{AX:} #1}}
\newcommand{\widebar}[1]{\mkern 1.5mu\overline{\mkern-1.5mu#1\mkern-1.5mu}\mkern 1.5mu}

\newcommand{\tr}{\mathrm{Tr}}
\newcommand{\dist}{\operatorname{dist}}

\def\mbb{\mathbb}
\def\mc{\mathcal}

\newcommand{\elu}{\operatorname{ELU}}
\newcommand{\gelu}{\operatorname{GELU}}
\newcommand{\silu}{\operatorname{SiLU}}
\newcommand{\relu}{\operatorname{ReLU}}
\newcommand{\rrelu}{\operatorname{RReLU}}
\newcommand{\leakyrelu}{\operatorname{Leaky-ReLU}}

    \title{Emergent Low-Rank Training Dynamics in MLPs with Smooth Activations}

   \author{
      Alec S. Xu \\
      University of Michigan \\
      \texttt{alecx@umich.edu} \\
      \And
      Can Yaras \\
      Princeton University \\
      \texttt{cy4900@princeton.edu} \\
      \And
      Matthew Asato \\
      University of Michigan \\
      \texttt{asamatt@umich.edu} \\
      \And
      Qing Qu \\
      University of Michigan \\
      \texttt{qingqu@umich.edu} \\
      \And 
      Laura Balzano \\
      University of Michigan \\
      \texttt{girasole@umich.edu}
    }
\fi

\ifarxiv 
    \documentclass{article}
    \usepackage{deepthink}
    \usepackage[utf8]{inputenc}

    \usepackage{amsmath, amsfonts, amssymb, amsthm}
    \usepackage{thmtools,thm-restate}
    \usepackage{geometry}
    \usepackage{algorithm}
    \usepackage{algpseudocode}
    \usepackage{enumitem}
    \usepackage[most]{tcolorbox}
    \usepackage{mathtools}
    \usepackage{hyperref}
    \hypersetup{colorlinks=true,linkcolor=blue, citecolor=blue}
    \usepackage[round]{natbib}
    \usepackage{subcaption}
    \usepackage{comment}
    \usepackage{bm}
    \usepackage{cleveref}
    \allowdisplaybreaks
    \usepackage{dsfont}
    \usepackage{booktabs} % for professional tables

    \theoremstyle{plain}
    \newtheorem{theorem}{Theorem}[section]
    \newtheorem{proposition}[theorem]{Proposition}
    \newtheorem{lemma}[theorem]{Lemma}
    
    \theoremstyle{definition}
    \newtheorem{definition}[theorem]{Definition}
    
    \theoremstyle{remark}
    
    \title{Emergent Low-Rank Training Dynamics in MLPs with Smooth Activations}

    \newcommand{\corrauth}{\textsuperscript{\ddag}}
    
    \authorblock{
      \href{https://alecxu00.github.io/}{\textbf{Alec S. Xu}}$^1$\corrauth,
      \href{https://canyaras.com}{\textbf{Can Yaras}}$^2$,
      \href{https://www.linkedin.com/in/matt-asato/}{\textbf{Matthew Asato}}$^1$, 
      \href{https://qingqu.engin.umich.edu/}{\textbf{Qing Qu}}$^1$,
      \href{https://web.eecs.umich.edu/~girasole/}{\textbf{Laura Balzano}}$^1$
    }
    
    \affiliation{
      $^1$ University of Michigan \quad $^2$ Princeton University
    }
    
    \authornote{
      \corrauth\ Corresponding author
    }
    
    \abstracttext{
        \noindent Recent empirical evidence has demonstrated that the training dynamics of large-scale deep neural networks occur within low-dimensional subspaces. While this has inspired new research into low-rank training, compression, and adaptation, theoretical justification for these dynamics in nonlinear networks remains limited. %compared to deep linear settings.
        To address this gap, this paper analyzes the learning dynamics of multi-layer perceptrons (MLPs) under gradient descent (GD). We demonstrate that the weight dynamics concentrate within invariant low-dimensional subspaces throughout training. Theoretically, we precisely characterize these invariant subspaces for two-layer networks with smooth nonlinear activations, providing insight into their emergence. Experimentally, we validate that this phenomenon extends beyond our %strict 
        theoretical assumptions. Leveraging these insights, we empirically show there exists a low-rank MLP parameterization that, when initialized within the appropriate subspaces, matches the classification performance of fully-parameterized counterparts on a variety of classification tasks.
    }
    
    
    \date{\today}
    \correspondence{\href{mailto:alecx@umich.edu}{alecx@umich.edu}}
    \resources{\href{https://github.com/alecxu00/mlp-low-dim-gd}{Code}}
    
     \headerlogo{Deepthink_landscape_do_not_delete_compressed_2.png}{https://deepthink-umich.github.io}

\newcommand{\ax}[1]{{\color{red} \textbf{AX:} #1}}
\newcommand{\widebar}[1]{\mkern 1.5mu\overline{\mkern-1.5mu#1\mkern-1.5mu}\mkern 1.5mu}

\newcommand{\tr}{\mathrm{Tr}}
\newcommand{\dist}{\operatorname{dist}}

\def\mbb{\mathbb}
\def\mc{\mathcal}

\newcommand{\elu}{\operatorname{ELU}}
\newcommand{\gelu}{\operatorname{GELU}}
\newcommand{\silu}{\operatorname{SiLU}}
\newcommand{\relu}{\operatorname{ReLU}}
\newcommand{\rrelu}{\operatorname{RReLU}}
\newcommand{\leakyrelu}{\operatorname{Leaky-ReLU}}

\fi

\begin{document}

\ifarxiv
    % arXiv template
    \makeDeepthinkHeader

    \begin{figure}[h]
      \centering
      \includegraphics[width=0.32\linewidth]{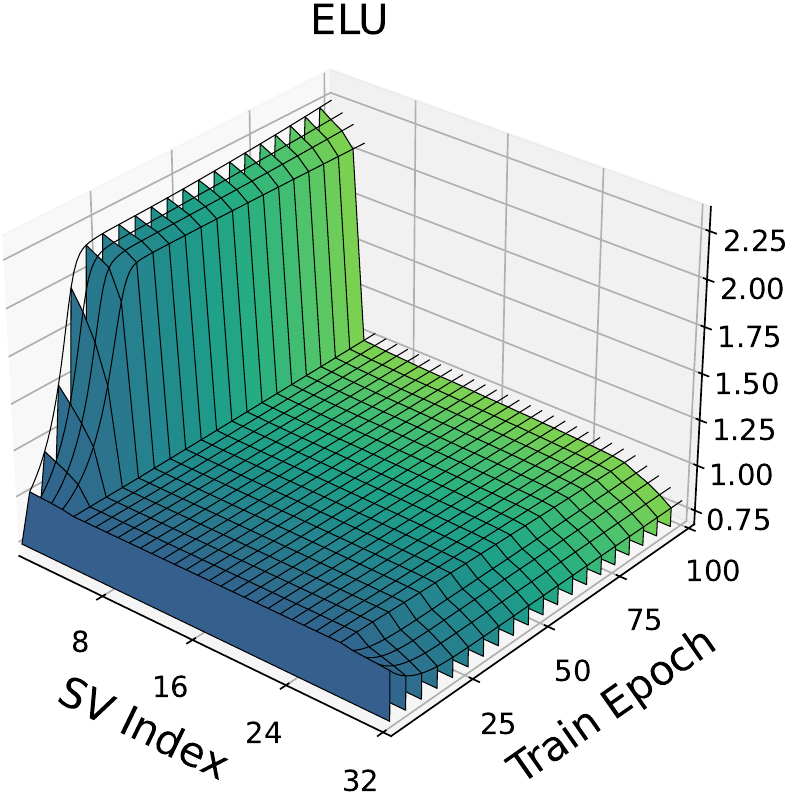}
      \includegraphics[width=0.31\linewidth]{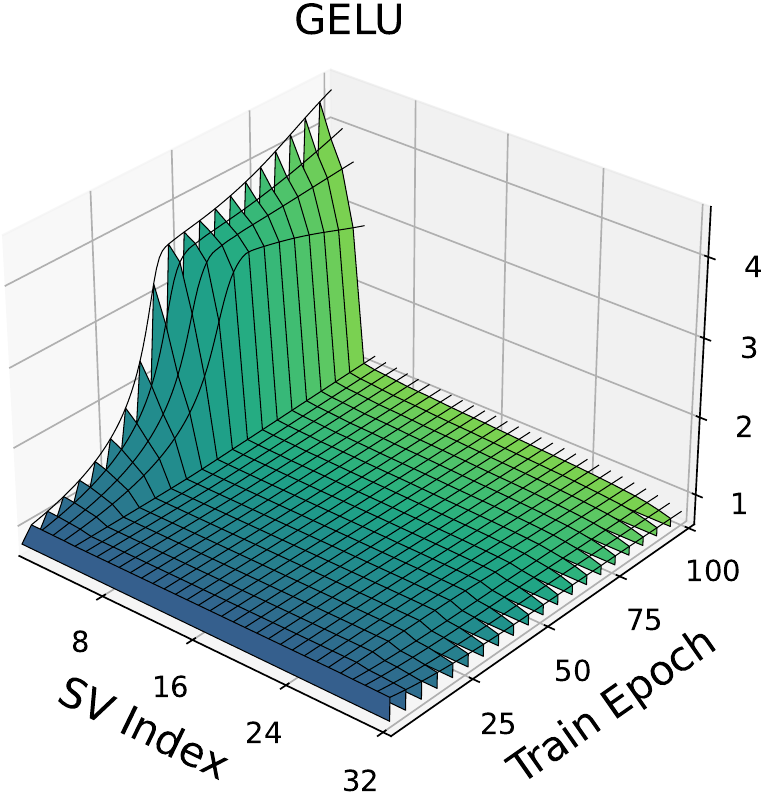}
      \includegraphics[width=0.32\linewidth]{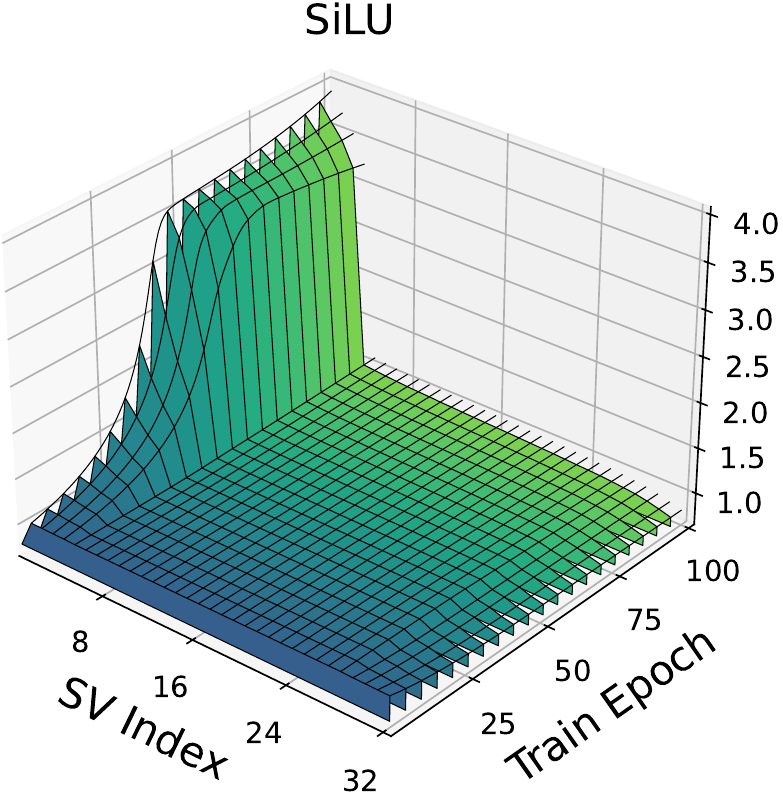}
      % \IfFileExists{Deepthink_landscape_do_not_delete_compressed.png}{%
      %   \includegraphics[width=0.9\linewidth]{Deepthink_landscape_do_not_delete_compressed.png}
      % }{%
      %   \fbox{\parbox[c][3.2cm][c]{0.9\linewidth}{\centering Teaser image placeholder}}
      % }
      \caption{\textbf{Low rank updates in MLPs with smooth activation functions.} Each plot shows how the singular values of the first layer (out of four layers total) evolve throughout training in MLPs with $\elu$, $\gelu$, and $\silu$ activation functions, which are all smooth. We trained each MLP on synthetic data and squared-error loss using gradient descent. Specific experimental details and additional plots of the deeper layer singular values are in \Cref{ssec:additional_sims_main_fig}.}
      \label{fig:teaser}
    \end{figure}

    \newpage
    \tableofcontents

    % Main content 
    \section{Introduction}
\label{sec:intro}
Recently, low-rank approaches have achieved great success in training and fine-tuning large neural networks. For example,  low-rank adaptation (LoRA) \cite{hu2022lora} is a popular fine-tuning technique for large language models (LLMs) by adding a low-rank adapter to frozen pre-trained weights. Other works have proposed projecting the parameters \cite{li2022low,zhang2023fine} or optimizer states \cite{zhao2024galore,robert2025ldadam,zhu2025apollo,rajabi2025subtrack++} onto low-dimensional subspaces, and then updating the parameters or optimizer states there; see \cite{balzano2025overview} for a survey. 

Empirical observations indicate that large model training and fine-tuning naturally occur within low-dimensional subspaces \cite{gur2018gradient,li2018measuring,larsen2022many}, which potentially explains the success of the aforementioned low-rank training approaches. However, a precise characterization of \emph{which} subspaces the optimization occurs within remains unclear. To address this question, \cite{yaras2023law,yaras2024compressible,kwon2024efficient} proved that when deep \emph{linear} networks are trained via gradient descent (GD), the weights are updated within fixed subspaces that depend on the weights and the first-layer gradient at initialization. However, practical neural networks are highly nonlinear, and no prior work has theoretically investigated how introducing nonlinearities impacts this phenomenon. To this end, we extend these works by focusing on multi-layer perceptrons (MLPs).  We find that for the MLP with small output dimension and smooth nonlinear activation functions, the training dynamics are highly concentrated within unchanging low-dimensional subspaces. Our contributions are as follows:

\begin{tcolorbox}[title=Contributions]
\begin{itemize}[leftmargin=*, labelsep=0.5em]
    \item \textbf{Theoretical analysis of two-layer networks (\Cref{sec:two_layer_theory}).} We provide the first theoretical analysis of this phenomenon in two-layer networks trained via GD. Specifically, we show that there exists a fixed %\qq{previously we used unchanged, here fixed, maybe we use the same term "fixed"?} 
    subspace such that, in each GD step, the change of the first-layer weights in this subspace is bounded. %\laura{we show there is a fixed subspace such that, in each gradient step, the change of the first-layer weights in this subspace is bounded and small.} 
    %in each GD step, we upper-bound the change in the component of the first-layer weights that lies in a high-dimensional subspace. 
    This subspace depends on the first layer's weights and gradient at initialization.
    
    %Specifically, we show the first layer's GD updates have a small component in a high-dimensional subspace that is determined at initialization. \qq{the description needs to be more precise, not very clear}
    
    \item \textbf{Low-rank training beyond theory (\Cref{sec:beyond_theory}).} Empirically, we show that the phenomenon of low-rank training dynamics extends well beyond our theoretical settings. %\textcolor{purple}
    {Specifically, in deep MLPs, we provide precise characterizations of layer-specific, low-dimensional subspaces in which each layer's weight matrix is mostly trained; these layer-specific subspaces also depend on the network at initialization. We empirically demonstrate that GD, SGD with momentum, Adam, and Muon all train deep MLP weight matrices in these identified subspaces.}
    
    %We empirically demonstrate that the GD updates of the deeper layers occur within low-dimensional subspaces that again depend on the network initialization. We also show this phenomenon approximately holds for networks trained using SGD with momentum and Adam.
    
    \item \textbf{Subspace-initialized low-rank MLPs (\Cref{sec:low_rank_param}).} Based on the previous insights, we empirically show there exists a low-rank MLP parameterization that, %\textcolor{purple}
    {\emph{if initialized using the identified subspaces from \Cref{sec:two_layer_theory,sec:beyond_theory}}}, achieves near-equivalent performance compared to their fully parameterized counterparts under the same training recipe. We conduct experiments on Fashion MNIST and CIFAR-10 using deep MLPs demonstrating this near-equivalence and the importance of the initialization. 
\end{itemize}
\end{tcolorbox}

\paragraph{Related work.} There have been several works on low-dimensional learning in neural networks \citep{li2018measuring,li2022low,larsen2022many,schotthofer2022low,yaras2024compressible,kwon2024efficient,wang2026universal}, as well as low-rank gradients in nonlinear networks \citep{ba2022high,zhao2024galore,jaiswal2025from,sonthalia2025low}. Another relevant line of work is the implicit bias towards low-rank weights in nonlinear networks \citep{frei2023implicit,kou2023implicit,timor2023implicit,min2024early}. See \Cref{sec:related} for more detailed discussions.

%\qq{I think we need to significant cite more related works on low-rank training and discuss the relationship. We also need to cite more related works on low-rank implicit bias. Currently, we cite quite a lot of our own results, which would not be viewed favorably by the reviewers.}

%\ax{Added a short paragraph on related work. Not sure if we have a lot of room to go into detailed discussions about their relationship with our work, so I put that discussion in the appendix.}
    \section{Problem Setup}
Here, we set up and motivate our problem of interest. 
\vspace{-0.3cm}

\paragraph{Notation.} We use unbolded letters $x, X$ for scalars, bold lower case letters $\bm x$ for vectors, and bold capital letters $\bm X$ for matrices. For some $N \in \mbb{N}$, $[N]$ denotes the set $\{1, 2, \dots, N\}$. For scalars $a, b$, we say $a \lesssim \mc{O}(b)$ if there exists a constant $C$ s.t. $a \leq C \cdot b$, $a \gtrsim \Omega(b)$ if $a \geq C \cdot b$, and $a = \Theta(b)$ if $a = C \cdot b$. We use $\sigma_i(\bm X)$, $\| \bm X \|_F$, $\| \bm X \|_1$, and $\| \bm X \|_{\max}$ to respectively denote the $i^{th}$ singular value, Frobenius norm, matrix-$1$ norm, and maximum magnitude element. Finally, $\mc{R}\left( \bm X \right)$ denotes the range (or column space) of $\bm X$, and $\mc{R}^\perp\left( \bm X \right)$ its orthogonal complement.

\paragraph{Data.} We consider data with inputs $\bm X \in \mathbb{R}^{d \times N}$ and labels $\bm Y \in \mathbb{R}^{K \times N}$, where $d$ is the data dimension, $N$ is the number of data points, and $K$ is the label dimension. We consider the case where $K$ is much smaller compared to $d$ and $N$ (formally defined in \Cref{assum:input_data}). Finally, we define $p := d - 2K$.

\paragraph{Network architecture.} We define an $L$-layer neural network $f_{\bm \Theta}: \mathbb{R}^{d \times N} \to \mathbb{R}^{K \times N}$ as follows:
\begin{equation} \label{eq:orig_mlp}
    f_{\bm \Theta}(\bm X) = \bm W_L \phi\left( \bm W_{L - 1} \phi\left( \dots \phi\left( \bm W_1 \bm X \right) \dots \right) \right),
\end{equation}
where $\phi(\cdot)$ is the element-wise nonlinear activation function, $\bm W_l$ denotes the $l^{th}$ layer's weight matrix, and $\bm \Theta = \{\bm W_1, \dots, \bm W_L\}$ denote the model parameters. Here, $\bm W_1 \in \mathbb{R}^{m_1 \times d}$, $\bm W_l \in \mathbb{R}^{m_l \times m_{l-1}}$ for $l \in \{2, \dots, L - 1\}$, and $\bm W_L \in \mathbb{R}^{K \times m_{L -1}}$. For simplicity, we assume $m_1 = m_2 = \dots = m_{L - 1} := m$.

% We equivalently define the network in \Cref{eq:orig_mlp} recursively as follows:
% \begin{align*}
%     \bm Z_l = \bm W_l \bm H_{l - 1} \in \mbb{R}^{m \times N} \; \; \text{and} \quad \bm H_l = \phi\left(\bm Z_l \right) \in \mbb{R}^{m \times N},
% \end{align*}
% where $\bm H_0 = \bm X$
% and $\bm Z_L = \bm W_L \bm H_{L - 1} \in \mbb{R}^{K \times N}$. Note $\bm Z_L = f_{\bm \Theta}(\bm X)$. 

\paragraph{Training.} We train $f_{\bm \Theta}$ defined in \eqref{eq:orig_mlp} via
\begin{equation} \label{eq:train-obj}
    \min \limits_{\bm \Theta} \mathcal{L} (\bm \Theta) \coloneqq \ell \left( f_{\bm \Theta} (\bm X),  \bm Y \right), %= \ell \left( \bm Z_L, \bm Y \right),
\end{equation}
where $\mathcal{L}$ is some loss function.
We solve \eqref{eq:train-obj} through gradient descent (GD) with step size $\eta$:
\begin{equation} \label{eq:gd-update}
    \bm \Theta(t + 1) = \bm \Theta(t) - \eta \nabla \mathcal{L}(\bm \Theta(t)) %, \; \; \text{or} \; \; \dot{\bm \Theta} = - \nabla \mc{L}(\bm \Theta(t))
\end{equation}
where $t$ denotes the iteration index. %Let $\bm W_l(t), \bm H_l(t), $ and $\bm Z_l(t)$ denote $\bm W_l$, $\bm H_l$, and $\bm Z_l$ at iteration $t$.

\subsection{Case Study: Smooth Activations Encourage Lower-Rank Training Dynamics}
\label{sec:motivation}

\begin{figure}[t]
    \centering
    \includegraphics[width=0.49\textwidth]{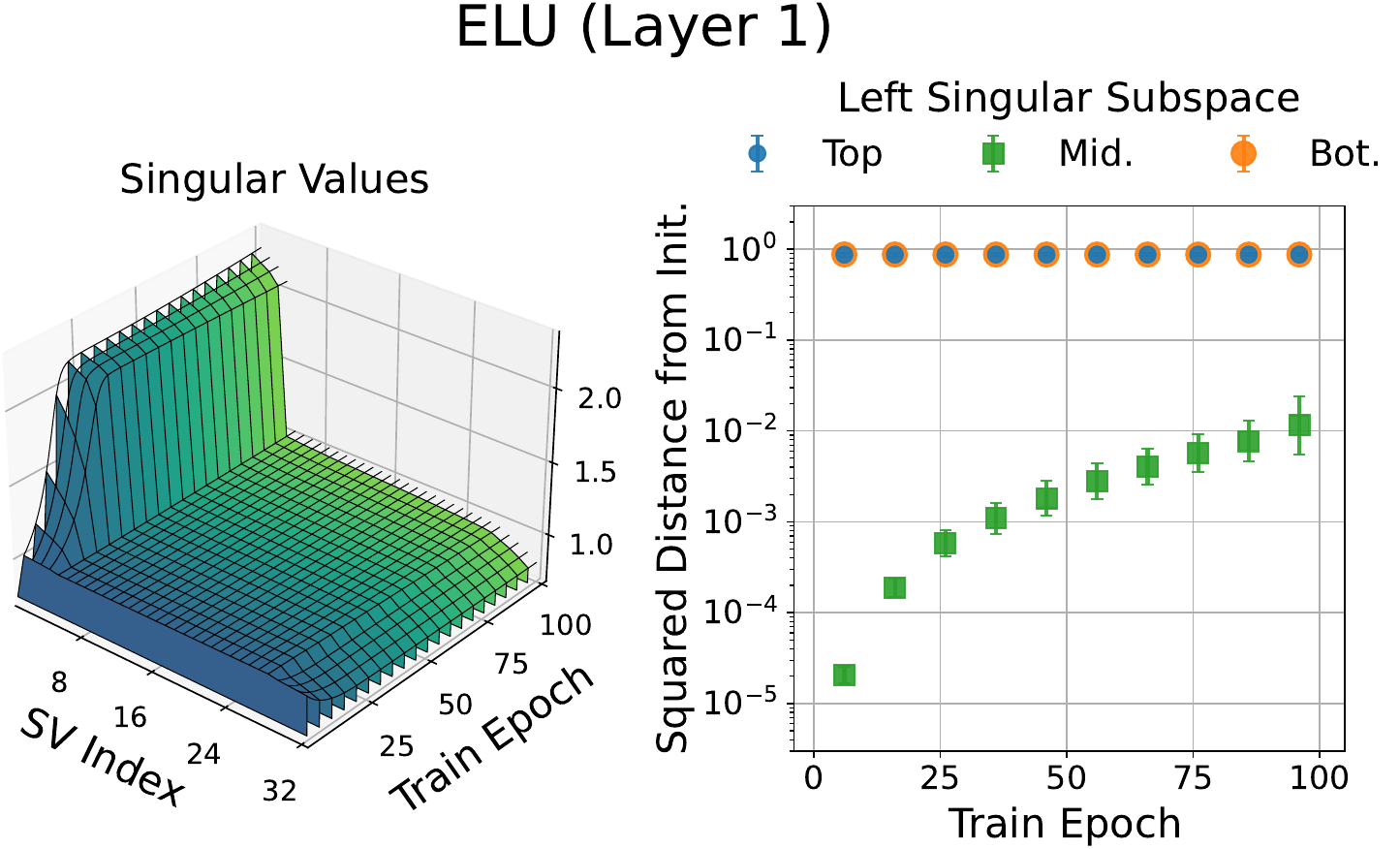}
    \includegraphics[width=0.49\textwidth]{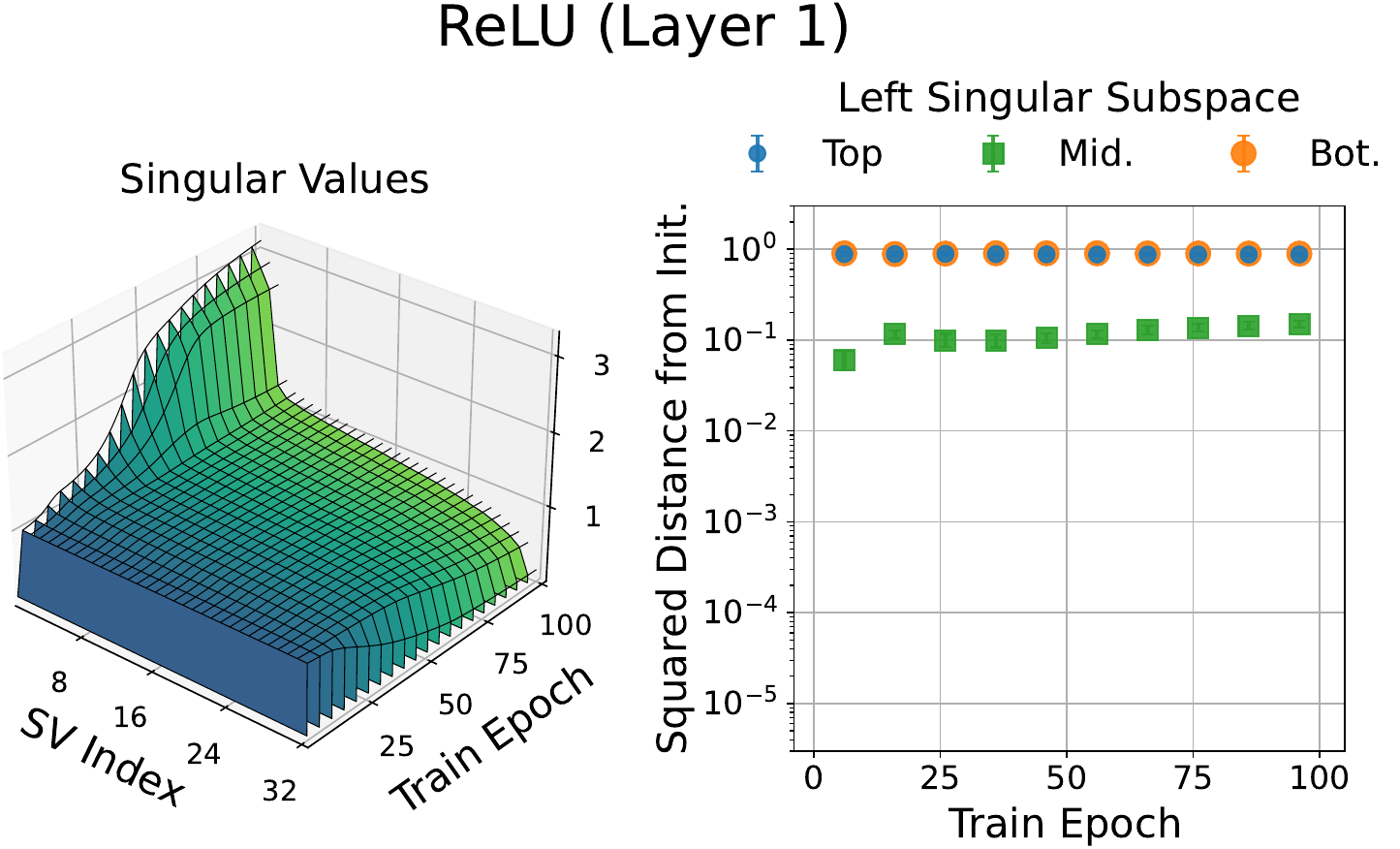}
    \caption{The middle singular subspace of the first-layer weight matrix in the $\elu$ network evolves noticeably slower than that in the $\relu$ network, and the corresponding singular values remain closer to their initialization. The right singular subspaces exhibit similar dynamics in both cases (see \Cref{ssec:additional_sims_main_fig}).}
    \label{fig:main_fig}
\end{figure}
 
To motivate our problem of interest, we present a case study showcasing when low-rank training dynamics emerge in MLPs. We trained six different MLPs using full-batch GD on synthetic classification data with $d = 32$ dimensions and $K = 4$ classes. Each MLP contained a different activation function: three had smooth activations ($\elu$, $\gelu$, and $\silu$), while the other three had nonsmooth ones ($\relu$, $\leakyrelu$, and $\mathrm{Randomized-}\relu$, or $\rrelu$). During training, we tracked each of the first three layers' singular values, as well as the changes in their top-$K$, bottom-$K$, and middle $d - 2K$ singular subspaces. We show results for the first layer of the $\elu$ and $\relu$ networks in \Cref{fig:main_fig}, while deferring other results and experimental details to \Cref{ssec:additional_sims_main_fig}. 

In the $\elu$ network's first layer, the middle $d - 2K$ singular subspace evolves very slowly throughout training, especially compared to that of the $\relu$ network. These singular subspace changes are also reflected in their corresponding singular values. This is not unique to $\elu$ and $\relu$: the changes in singular values and subspaces are noticeably smaller in MLPs with smooth activations compared to nonsmooth (see \Cref{fig:main_fig_more_results_layer1,fig:main_fig_more_results_layer2,fig:main_fig_more_results_layer3} in \Cref{ssec:additional_sims_main_fig}). Thus, for MLPs with smooth activations, the training dynamics appear much more concentrated within a lower-dimensional subspace, implying \textbf{low-rank training dynamics emerge more prominently for MLPs with smooth activation functions.} Therefore, for the remainder of our study, \textbf{we focus on characterizing these dynamics with such activations.}

    \section{Analysis on Two-Layer Nonlinear Networks} \label{sec:two_layer_theory}
In this section, we provide our main theoretical result characterizing low-rank training dynamics in two-layer networks.

\subsection{Definitions and Assumptions} 
In this section, we provide some preliminary definitions, and then state and discuss our assumptions. We first define quantities that measure the alignment between two subspaces.

\begin{definition}[Principal angles between subspaces]
\label{def:princ_angles}
    Let $\bm U_1, \bm U_2 \in \mbb{R}^{d \times r}$ be orthonormal bases of two $r$-dimensional subspaces of $\mbb{R}^d$. Then, for all $i \in [r]$, the $i^{th}$ principal angle $\psi_i$ between $\bm U_1$ and $\bm U_2$ is defined as such:
    \begin{equation*}
        \psi_i = \arccos\left( \sigma_i\left( \bm U_1^\top \bm U_2 \right) \right)  \in \left[0, \pi / 2 \right]. 
    \end{equation*}
    Also define $\Psi := \begin{bmatrix}
        \psi_1 & \dots & \psi_r
    \end{bmatrix}$. We also measure the alignment between $\bm U_1$ and $\bm U_2$ through the following metric:
    \begin{equation*}
        \| \sin \Psi(\bm U_1, \bm U_2) \|_2 := \sqrt{ \sum\limits_{i=1}^r \sin^2(\psi_i) } \in \left[0, \sqrt{r} \right].
    \end{equation*}
\end{definition}

\medskip 

\noindent Smaller principal angles $\psi_i$ indicate $\bm U_1$ and $\bm U_2$ are well aligned, and also leads to smaller $\| \sin \Psi(\bm U_1, \bm U_2) \|_2$. Thus, smaller $\| \sin \Psi(\bm U_1, \bm U_2) \|_2$ indicates greater alignment between $\bm U_1$ and $\bm U_2$.

\medskip 

\noindent We now provide our assumptions. First, we state our assumptions on the input data $\bm X$ and labels $\bm Y$.
\begin{restatable}[Input data]{assumption}{dataassumption}
\label{assum:input_data}
    The data $\bm X \in \mbb{R}^{d \times N}$ is whitened, %\footnote{Any $\bm X \in \mbb{R}^{d \times N}$ with full row rank can be whitened via preconditioning.}, i.e., $\bm X \bm X^\top = \bm I_d$,
    and the label dimension $K$ satisfies $K < d / 2$. %$ \bm Y^{K \times N}$ satisfy $\bm Y = \bm I_K \otimes \bm 1_n^\top$, where $n = N / K$ and $K < d / 2$.
    % \begin{itemize}
    %     \item $\bm X$ is whitened, i.e., $\bm X \bm X^\top = \bm I_d$, and 
    %     \item The cross-correlation matrix $\bm Y \bm X^\top$ satisfies $\frac{\sigma_1(\bm Y \bm X^\top)}{\sigma_K(\bm Y \bm X^\top)} = \kappa_{\bm Y \bm X^\top}$ for some finite constant $\kappa_{\bm Y \bm X^\top} > 1$.
    % \end{itemize}
\end{restatable}
\noindent Before proceeding, we briefly discuss \Cref{assum:input_data}. 
\begin{itemize}[leftmargin=*, labelsep=0.1em]
    \item \textbf{Whitened input data.} %\qq{the whitened input assumption has been used in other papers, we could cite more. Citing Yaras et al. would make our paper looks incremental to the initial paper} 
    Several previous works on neural network analyses assume whitened input data, e.g., \citet{arora2019convergence,braun2022exact,yaras2023law,domine2025lazy}. %As noted in \Cref{sec:intro}, our study is largely inspired by \citet{yaras2023law,yaras2024compressible}. They showed low-rank training dynamics emerge in deep linear networks for whitened input data. %Following this work, we make the same assumption. 
    In \Cref{ssec:beyond_theory_sgd}, we empirically observe this phenomenon approximately holds for unwhitened $\bm X$.

    \item \textbf{Small output dimension $K \ll d$.} %\qq{I think we can just say $K$ is small, and use classification as an example for why this valuable setting. For classificaiton, the input would not be whitened. It seems to be a bit weird if we say classification first.} 
    We consider a setting where the $K \ll d$, which is commonly studied. For instance, \citet{andriopoulos2024prevalence} studied neural collapse \citep{papyan2020prevalence} in this exact setting. %\qq{is there any other examples?, maybe we can discuss https://arxiv.org/abs/2409.04180, neural multivarate regression. Only mention classification risk the question of input data.} 
    Classification is another example of where this setting is common, as $K$ represents the number of classes. In our analysis, we show $K$ governs the rank of the low-rank training dynamics. %If $K$ is large relative to $d$, this low-rank training phenomenon is lost without further assumptions on $\bm X$. %We consider a multi-class classification setting where the number of classes $K$ much smaller than $d$, which is common in many classification datasets. In our analysis, we show $K$ governs the rank of the low-rank training dynamics. If $K$ is large relative to either $d$, this low-rank phenomena is likely lost without any further assumptions on $\bm X$. %; see \Cref{sec:large_output_dim} for more a detailed discussion.
\end{itemize}

\medskip

\noindent Next, we state our assumptions on the network architecture and training. 
\begin{restatable}[Network architecture and training]{assumption}{trainassumption}
\label{assum:network}
    The network \eqref{eq:orig_mlp} contains $L = 2$ layers, i.e., $f_{\bm \Theta}(\bm X) =: f_{\bm W_2, \bm W_1}(\bm X) = \bm W_2 \phi(\bm W_1 \bm X)$, with $\bm W_1 \in \mbb{R}^{m \times d}$ and $\bm W_2 \in \mbb{R}^{K \times m}$. Furthermore,
    \begin{itemize}[leftmargin=*, labelsep=0.5em]

        \item The width $m$ satisfies $m \geq d$, 
        %\item $\bm W_1$ is initialized as an $\epsilon$-scaled semi-orthogonal matrix, i.e., $\bm W_1^\top(0) \bm W_1(0) = \epsilon^2 \bm I_d$,
        \item The activation function $\phi$ satisfies $\phi(0) = 0$, $|\phi'(x)| \leq \beta$, and $|\phi''(x)| \leq \mu$ for all $x \in \mbb{R}$.
        \item The network is trained using GD with step size $\eta$ on the squared error loss:
        \begin{equation} \label{eq:two_layer_squared_error_loss}
            \min\limits_{\bm W_1} \mc{L}(\bm W_1) = \frac{1}{2} \left \| f_{\bm W_2, \bm W_1}(\bm X) - \bm Y \right \|_F^2,
        \end{equation}
        where $\bm W_2$ is fixed during training and is full row rank.
    \end{itemize}
\end{restatable}
\noindent From now on, we use $\bm G_1(t)$ to denote $\nabla_{\bm W_1} \mathcal{L}\left( \bm W_1(t) \right)$. We now provide some brief remarks on \Cref{assum:network}.
\begin{itemize}[leftmargin=*, labelsep=0.5em]
    \item \textbf{Squared-error loss.} %\ax{If we don't focus on classification in the theory, can we shorten this discussion?} Although cross-entropy loss is typically used in classification problems, \citet{hui2021evaluation} showed squared-error loss can achieve comparable, and sometimes better, performance than cross-entropy on classification tasks. Furthermore,  \citet{zhou2022all} studied the neural collapse phenomenon in classification problems under mean-squared error loss, and \citet{wang2023understanding} analyzed the feature compression abilities in deep linear networks trained on squared-error loss in a classification setting. 
    Several works on neural network analyses considered squared-error loss \citep{oymak2019overparameterized,oymak2020toward,bao2023global,chistikov2023learning}, even for classification tasks \citep{hui2021evaluation,zhou2022all,wang2023understanding}. In \Cref{ssec:beyond_theory_sgd}, we empirically observe in classification tasks,  training on cross-entropy loss leads to a similar low-rank training phenomenon. 
    
    \item \textbf{Fixed second layer.} In our problem, we aim to show the  weight matrices in nonlinear networks mostly get updated in 
    low-dimensional subspaces. Since $\bm W_2$ is of size $K \times m$, where $K \ll d \leq m$, we focus our analysis on the GD dynamics of $\bm W_1$, which is of size $m \times d$. To simplify the analysis, we keep $\bm W_2$ fixed during training. A fixed second layer is a standard assumption in two-layer network analyses, e.g., \citet{frei2023implicit,kou2023implicit,boursier2022gradient,oymak2019overparameterized,oymak2020toward,gopalani2025global}. 
\end{itemize}
While we recognize that training $\bm W_2$ jointly with $\bm W_1$ aligns more closely with practical scenarios, the highly coupled and non-linear optimization dynamics of joint training present significant analytical challenges. Relaxing this assumption remains an important direction for future research. Nevertheless, our current framework isolates the core mechanisms of the network, and provides a necessary foundational step toward understanding the fully joint dynamics.

\medskip 

\noindent Finally, for \Cref{thm:smooth_main_result_main_body}, we make the following technical assumption.
\Cref{sec:smooth_assum_justify}.
\begin{restatable}{assumption}{residualassumption}
\label{assum:residual}
    For all $t \geq 0$, $\left\| \bm W_2 \phi\left( \bm W_1(t) \bm X \right) - \bm Y \right\|_{\max} \leq M$ for some finite constant $M$.
    % Define $\bm \Delta_2(t) = \bm W_2 \phi(\bm W_1(t) \bm X) - \bm Y$ and 
    % $\bm G_1(t) = \nabla_{\bm W_1} \mc{L}(\bm W_1(t))$. For all $t \geq 0$, 
    % \begin{itemize}[leftmargin=*, labelsep=0.5em]

    %     \item $\| \bm \Delta_2(t) \|_{\max} \leq M$ for some finite constant $M$, and
    %     \item there exist constants $G_1$, $G_2$, and $G_3$ such that $\bm G_1(t)$ satisfies the following:
    %     \begin{align*}
    %         &\frac{\| \bm G_1(t) \|_F}{\| \bm G_1(0) \|_F} \leq G_1 \cdot \left(1 - \Theta(\eta) \right)^{\Theta(t)}, \\
    %         &\frac{\sigma_i\left( \bm G_1(t) \right)}{\sigma_i\left( \bm G_1(0) \right)} \leq G_2 \cdot \frac{\| \bm G_1(t)\|_F}{\| \bm G_1(0) \|_F}, \quad \text{for all $i \leq K$, and} \\
    %         &\sigma_K\left( \bm W_2^\top \bm Y \bm X^\top \right) - \sigma_{K + 1}\left( \bm G_1(t) \right) \geq G_3 \cdot \sigma_K\left( \bm W_2^\top \bm Y \bm X^\top \right)
    %     \end{align*}
    % \end{itemize}
\end{restatable}
\noindent This assumption states the maximum magnitude element in the residual is bounded by some finite constant, which we make for technical convenience. This assumption is quite mild: defining $\bm \Delta_2(t) := \bm W_2 \phi\left( \bm W_1(t) \bm X \right) - \bm Y$, if we did have $\left\| \bm \Delta_2(t) \right\|_{\max} \to \infty$, then we would also have $\mathcal{L}\left( \bm W_1(t) \right) = \frac{1}{2} \left\| \bm \Delta_2(t) \right\|_F^2 \to \infty$.

%We simplify the presentation for ease of understanding, and present the precise statement in \Cref{sec:smooth_proofs}.

\subsection{Main Results}
\label{ssec:main_result}
In this section, we present our main theoretical result. 
\begin{tcolorbox}[breakable]
\begin{theorem}[Hidden-Layer Low-Rank Training]
\label{thm:smooth_main_result_main_body}
     Recall $p := d - 2K$, where $d$ is the data dimension, and $K$ is the label dimension. Let $\bm L_{1, 1}(t)$ and $\bm R_{1, 1}(t)$ denote top-$K$ left and right singular subspaces of $\bm G_1(t)$, and define the singular subspace alignment to initialization at iteration $t$ as
    \begin{align*}
        A(t) := \max\bigg\{ &\left\| \sin \Theta\left( \bm L_{1, 1}(t), \bm L_{1, 1}(0) \right) \right\|_2, \left\| \sin \Theta\left( \bm R_{1, 1}(t), \bm R_{1, 1}(0) \right) \right\|_2 \bigg\}.
    \end{align*}
    Suppose $\bm W_1(0)$ satisfies $\bm W_1^\top(0) \bm W_1(0) = \epsilon^2 \bm I_d$ with $\epsilon \lesssim  \mc{O}\left( \sigma_K\left( \bm W_2^\top \bm Y \bm X^\top \right) \right)$, and let the step size satisfy $\eta \lesssim \min\left\{1, \frac{1}{\Omega\left( \| \bm W_2 \|_1 + \sigma_1^2(\bm W_2) \right)} \right\}$.
    Then, there exist orthogonal matrices $\bm U \in \mbb{R}^{m \times m}$ and $\bm V \in \mbb{R}^{d \times d}$ that only depend on $\bm W_1(0)$ and $\bm G_1(0)$ such that, for all $t \geq 0$, $\bm W_1(t)$ admits the following decomposition:
    \begin{align*}
        \bm W_1(t) = \bm U \widetilde{\bm W}_1(t) \bm V^\top = \bm U \begin{bmatrix}
            \widetilde{\bm W}_{1, 1}(t) & \widetilde{\bm W}_{1, 2}(t) \\
            \widetilde{\bm W}_{1, 3}(t) & \widetilde{\bm W}_{1, 4}(t)
        \end{bmatrix} \bm V^\top,
    \end{align*}
    where $\widetilde{\bm W}_{1, 1}(t) \in \mbb{R}^{( m - p) \times 2K}$, $\widetilde{\bm W}_{1, 2}(t) \in \mbb{R}^{(m - p) \times p}$, $\widetilde{\bm W}_{1, 3}(t) \in \mbb{R}^{p \times 2K}$, and $\widetilde{\bm W}_{1, 4}(t) \in \mbb{R}^{p \times p}$, with
    \begin{align*}
        &\widetilde{\bm W}_{1, 2}(0) = \bm 0, \; \widetilde{\bm W}_{1, 3}(0) = \bm 0, \; \frac{\| \widetilde{\bm W}_{1, 4}(0)\|_F}{\sqrt{p}} = \epsilon,
    \end{align*}
    and, for $i \in \{2, 3, 4\}$,
    \begin{align*}
        \frac{\| \widetilde{\bm W}_{1, i}(t + 1) - \widetilde{\bm W}_{1, i}(t)\|_F}{\sqrt{p}} 
        \lesssim \eta \cdot \sqrt{\sigma_1^2\left( \bm G_1(t) \right) \cdot \frac{A^2(t)}{p} + \sigma_{K + 1}^2\left( \bm G_1(t) \right)}. 
    \end{align*}
\end{theorem}
\end{tcolorbox}

\begin{figure*}[t]
    \centering
    \includegraphics[width=0.8\textwidth]{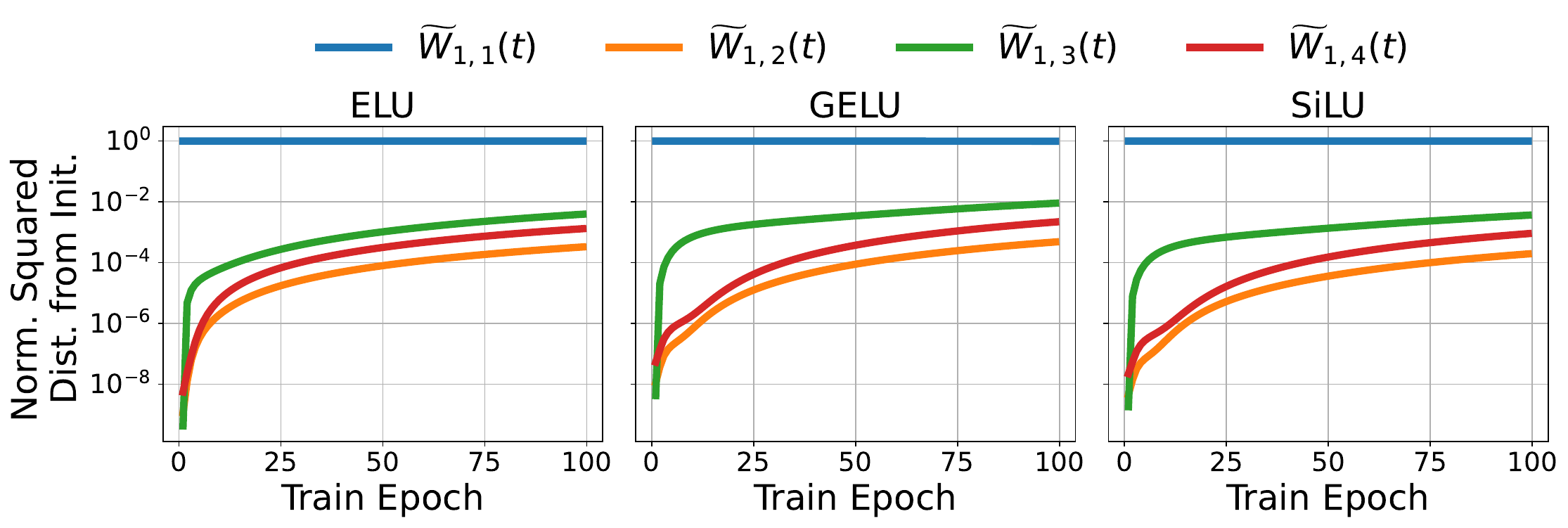}
    
    \caption{Under our exact theoretical setting, $\widetilde{\bm W}_{1, 1}(t)$ accounts for almost all of the change in $\widetilde{\bm W}_1(t)$, and thus $\bm W_1(t)$.}
    \label{fig:smooth_theory_verify}
\end{figure*}

\paragraph{Discussion on \Cref{thm:smooth_main_result_main_body}.} We briefly discuss the implications of our main theoretical result --- for ease of exposition, we focus on the case where $m = d$. Let 
$\bm U := \begin{bmatrix}
    \bm U_1 & \bm U_2
\end{bmatrix}$ and 
$\bm V = \begin{bmatrix}
    \bm V_1 & \bm V_2
\end{bmatrix}$, 
where $\bm U_1 \in \mbb{R}^{d \times 2K}$, $\bm U_2 \in \mbb{R}^{d \times p}$, $\bm V_1 \in \mbb{R}^{d \times 2K}$, and $\bm V_2 \in \mbb{R}^{d \times p}$ all have orthonormal columns, where again $p = d - 2K$. From \Cref{thm:smooth_main_result_main_body}, we have 
\begin{align*}
    \bm W_1(t) = &\underbrace{\bm U_1 \widetilde{\bm W}_{1, 1}(t) \bm V_1^\top}_{\text{rank } 2K} + \; \underbrace{\bm U_1 \widetilde{\bm W}_{1, 2}(t) \bm V_2^\top + \bm U_2 \widetilde{\bm W}_{1, 3}(t) \bm V_1^\top + \bm U_2 \widetilde{\bm W}_{1, 4}(t) \bm V_2^\top}_{\text{``perturbation terms''}}.
\end{align*}
\Cref{thm:smooth_main_result_main_body} states the perturbation terms are bounded in each iteration, implying the remainder of the GD update occurs in $\widetilde{\bm W}_{1, 1}(t).$ The bound depends on the gradient singular subspace alignment at iteration $t$ vs. initialization. 

\paragraph{Proof sketch.} We now summarize our proof strategy for \Cref{thm:smooth_main_result_main_body}. Let $\bm R_{1, 2}(t)$ and $\bm L_{1, 1}(t)$ respectively denote the bottom $d - K$ and top-$K$ right and left singular subspaces of $\bm G_1(t)$. We identify the following $p$-dimensional subspace where $\bm G_1(t)$ has a bounded component, which is analogous to the subspace identified in \cite{yaras2023law,yaras2024compressible}:
\begin{align}
\label{eq:smooth_small_update_subspace}
    \mc{S}_{small} = \mc{R}\left( \bm R_{1, 2}(0) \right) \cap \mc{R}^\perp\left( \bm W_1^\top(0) \bm L_{1, 1}(0) / \epsilon \right). 
\end{align}
Let $\bm V_2$ be an orthonormal basis for $\mc{S}_{small}$ and $\bm U_2 = \bm W_1(0) \bm V_2 / \epsilon$. We bound $\| \bm G_1(t) \bm V_2 \|_F$ and $\| \bm G_1^\top(t) \bm U_2 \|_F$ based on the alignment between the singular subspaces of $\bm G_1(t)$ and $\bm G_1(0)$.

\paragraph{Experimental verification.} We trained the first layer of two layer networks of various activations under the exact setting of \Cref{thm:smooth_main_result_main_body}; we defer specific experimental details to \Cref{ssec:additional_sims_theory_verify}. \Cref{fig:smooth_theory_verify} shows the squared distance of each $\widetilde{\bm W}_{1, i}(t)$ from $\widetilde{\bm W}_{1, i}(0)$ normalized by the squared distance between $\widetilde{\bm W}_1(t)$ and $\widetilde{\bm W}_1(0)$, i.e., $ \| \widetilde{\bm W}_{1, i}(t) - \widetilde{\bm W}_{1, i}(0) \|_F^2 / \| \widetilde{\bm W}_1(t) - \widetilde{\bm W}_1(0) \|_F^2$, averaged over $10$ trials. Clearly, for all activations, $\widetilde{\bm W}_{1, 1}(t)$ accounts for almost all of the change in $\widetilde{\bm W}_1(t)$, and thus $\bm W_1(t)$. %undergoes a significantly larger change compared to the other $\widetilde{\bm W}_{1, i}(t)$ terms. This implies $\widetilde{\bm W}_{1, 1}(t)$ is largely responsible for the changes in $\widetilde{\bm W}_1(t)$, and thus also in $\bm W_1(t)$, despite being a noticeably smaller size than the other $\widetilde{\bm W}_{1, i}(t)$ terms.
These results support the main message in \Cref{thm:smooth_main_result_main_body}: \textbf{each GD update mostly occurs in a low-dimensional subspace.} 

\paragraph{Quantifying the bound.} In our next result, we provide some quantitative intuition on the bound in \Cref{thm:smooth_main_result_main_body}. To do so, we rely on empirical observations within our theoretical setting; see \Cref{fig:smooth_assum_grad_top_sval_decay,fig:smooth_assum_tail_sval} in \Cref{ssec:smooth_assum_justify_gradients}. We formalize these observations in the following assumption, which we recognize is quite strict in broader cases. 

\begin{restatable}{assumption}{gradientassumption}
\label{assum:gradient}
    There exist constants $G_1, G_2, G_3 > 0$ such that $\bm G_1(t)$ satisfies the following:
        \begin{align*}
            &\frac{\| \bm G_1(t) \|_F}{\| \bm G_1(0) \|_F} \leq G_1 \cdot \left(1 - \Theta(\eta) \right)^{\Theta(t)}, \quad \frac{\sigma_i\left( \bm G_1(t) \right)}{\sigma_i\left( \bm G_1(0) \right)} \leq G_2 \cdot \frac{\| \bm G_1(t)\|_F}{\| \bm G_1(0) \|_F} \; \; \text{for all $i \leq K$, and} \\
            &\sigma_K\left( \bm W_2^\top \bm Y \bm X^\top \right) - \sigma_{K + 1}\left( \bm G_1(t) \right) \geq G_3 \cdot \sigma_K\left( \bm W_2^\top \bm Y \bm X^\top \right).
        \end{align*}
\end{restatable}

\noindent We now state the result below.

\begin{tcolorbox}
\begin{restatable}{corollary}{boundcorollary}
\label{coro:quantify_thm_bound}
    Suppose the setting of \Cref{thm:smooth_main_result_main_body} and \Cref{assum:gradient} hold. Then, for some constants $C_1$ and $C_2$ that depend on $G_1$, $G_2$, $G_3$, $\beta$, $\mu$, $\epsilon$, $\phi'(0)$, $\sigma_1(\bm W_2)$, and $\sigma_K(\bm W_2^\top \bm Y \bm X^\top)$, we have
    \begin{align*}
        \sqrt{ \sigma_1^2\left( \bm G_1(t) \right) \cdot \frac{A^2(t)}{p} + \sigma_{K + 1}^2\left( \bm G_1(t) \right)}  \leq \begin{cases}
            C_1 \cdot \epsilon & t = 0 \\
            C_2 \cdot \left(1 - \Theta(\eta) \right)^{\Theta(t)} \cdot \left(1 - \left( 1 - \Theta(\eta) \right)^{\Theta(t)} \right) & t > 0
        \end{cases}.
    \end{align*}
\end{restatable}
\end{tcolorbox}

\paragraph{Discussion on \Cref{coro:quantify_thm_bound}.} In the early training stages, the $1 - \left(1 - \Theta(\eta) \right)^{\Theta(t)}$ term indicates the singular subspaces are well-aligned with those at initialization. As training progresses, these singular subspaces become less aligned. However, from the $\left(1 - \Theta(\eta) \right)^{\Theta(t)}$ term, eventually the GD update itself decreases in magnitude as GD converges to a stationary point. We provide a proof of \Cref{coro:quantify_thm_bound} in \Cref{ssec:coro_proof}.

\paragraph{Why activation smoothness matters.} In \Cref{coro:quantify_thm_bound}, a key step in our proof is showing the bottom $d - K$ singular values of $\bm G_1(0)$ are at most $\Theta(\epsilon)$, which leads to the $t = 0$ result. In \Cref{sec:nonsmooth_intuition}, we provide a result showing this does not hold for a nonsmooth $\phi$. This implies when $\phi$ is nonsmooth, even a \emph{single} GD update could increase the perturbation terms by a ``large'' amount, which we precisely observe in \Cref{fig:main_fig}.

    \section{Empirical Observations Beyond Theory}
\label{sec:beyond_theory}
In this section, we investigate if low-rank training dynamics persist beyond our theoretical setting in \Cref{sec:two_layer_theory}. %particularly 1) in deep networks with different activations (\Cref{ssec:beyond_theory_deep_nets_and_activations}), 2) for cross-entropy loss instead of squared-error (\Cref{ssec:beyond_theory_loss}), and 3) for SGD with momentum instead of vanilla GD (\Cref{ssec:beyond_theory_sgd}).
We ran all experiments in \texttt{PyTorch} using an NVIDIA A100 GPU.

\subsection{Deeper Networks}
\label{ssec:beyond_theory_deep_nets_and_activations}

\begin{figure*}[t]
    \centering
    \includegraphics[width=0.49\textwidth]{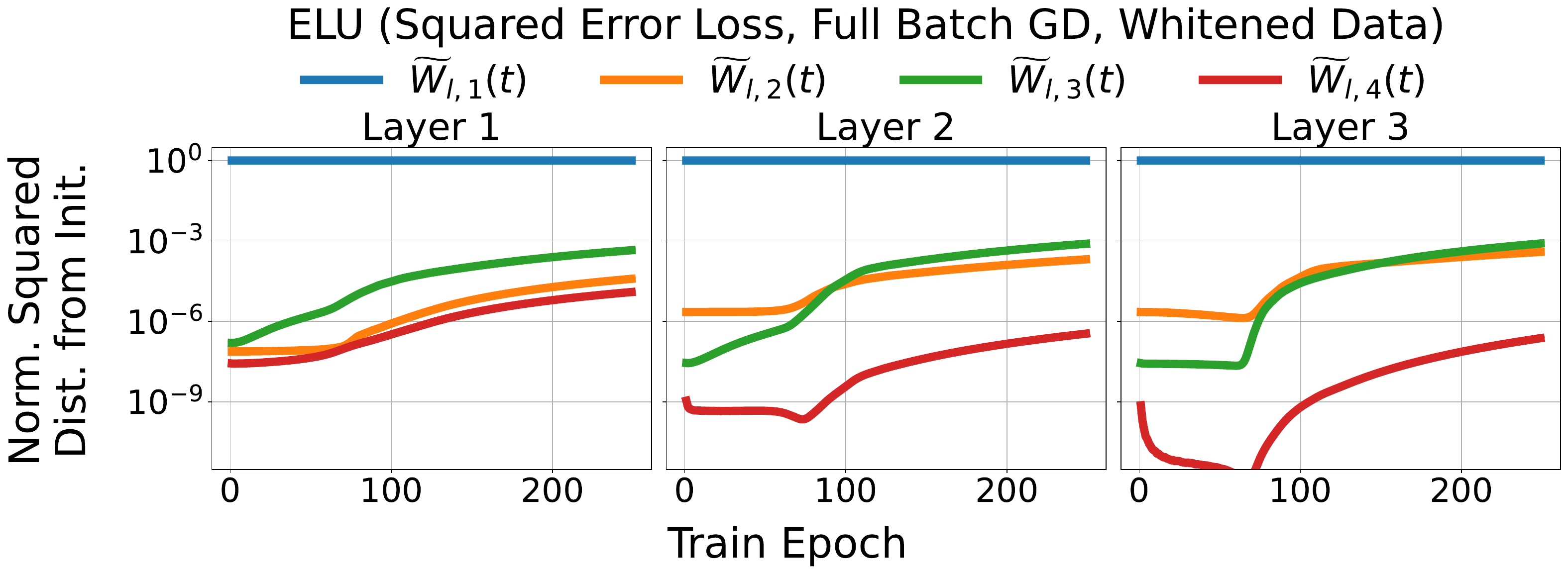}
    \includegraphics[width=0.49\textwidth]{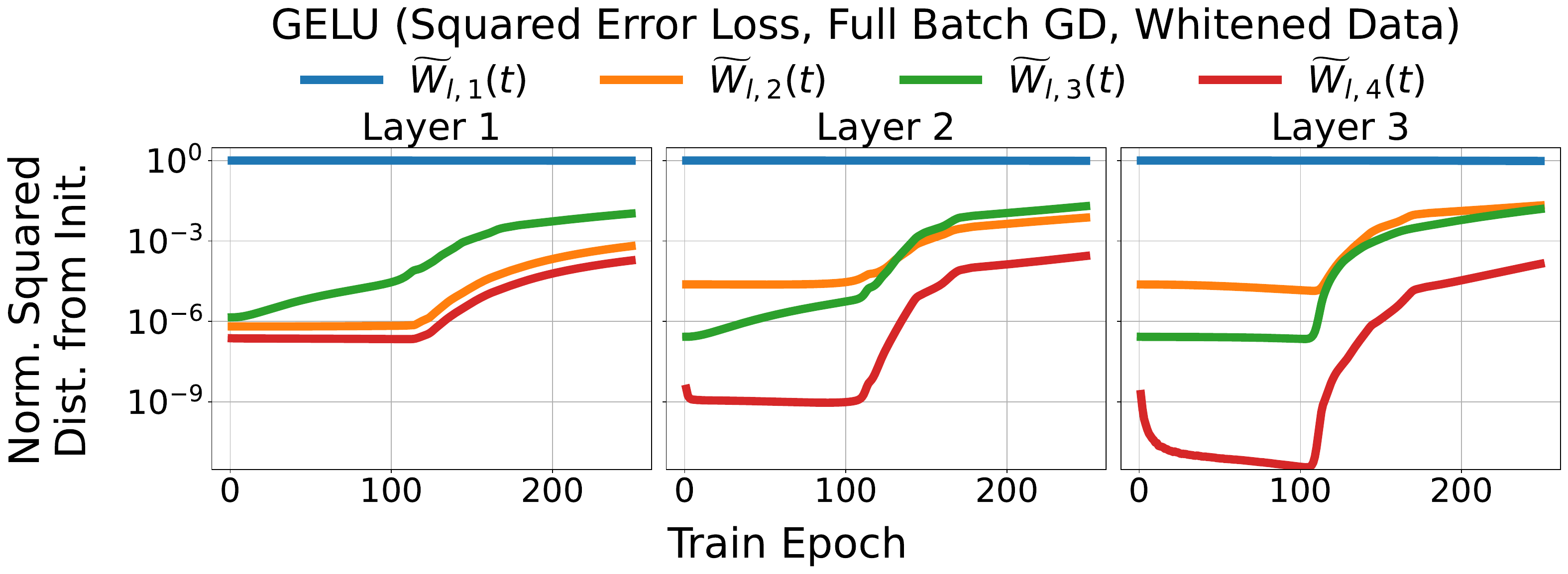}

    \caption{For every $l \in [L - 1]$ in deep MLPs with smooth activations, $\widetilde{\bm W}_{l, 1}(t)$ accounts for almost all of the change in $\widetilde{\bm W}_l(t)$.}
    \label{fig:beyond_theory_deep_nets_and_activations}
\end{figure*}

Here, we show that low-rank training GD updates emerge in deeper MLPs, where each layer $l$ has a corresponding ``small-update subspace'' $\mc{S}^{(l)}_{small}$. We are interested in identifying if the weight matrices are mostly updated in low-dimensional subspaces. Since $\bm W_L \in \mbb{R}^{K \times m}$, where $K \ll d \leq m$, we focus our attention to the first $L - 1$ layers, where $\bm W_1 \in \mbb{R}^{m \times d}$ and $\bm W_l \in \mbb{R}^{m \times m}$ for $l = 2, \dots, L - 1$. 

\paragraph{Network architecture and training.} We considered $L = 4$ layer networks with activations $\phi = \elu$ and $\gelu$. We trained the networks on squared-error loss using full-batch GD, and initialized the first $3$ layers to be $\epsilon$-scaled (semi-)orthogonal matrices with $\epsilon = 0.1$. We trained on synthetic data as described in \Cref{ssec:additional_sims_theory_verify}, and defer remaining details to \Cref{ssec:additional_sims_beyond_theory}. 

\paragraph{Small-update subspaces for deeper layers.} Recall from \eqref{eq:smooth_small_update_subspace}, the ``small update subspace'' for the first layer in a two-layer network is
\begin{equation}
\label{eq:smooth_small_update_subspace_W1_deep_net}
    \mc{S}_{small}^{(1)} = \mc{R}\left(  \bm R_{1, 2}(0) \right) \cap \mc{R}^\perp \left( \bm W_1^\top(0) \bm L_{1, 1}(0) / \epsilon \right). 
\end{equation}
Let $\bm V_{1, 2} \in \mbb{R}^{d \times p}$ be an orthonormal basis for $\mc{S}_{small}^{(1)}$, and $\bm U_{1, 2} := \bm W_1(0) \bm V_{1, 2} / \epsilon \in \mbb{R}^{m \times p}$. %From \Cref{thm:smooth_main_result_main_body}, we have $\bm W_1(t + 1) \bm V_{1, 2} \approx \bm W_1(t) \bm V_{1, 2}$ and $\bm W_1^\top(t + 1) \bm U_{1, 2} \approx \bm W_1^\top(t) \bm U_{1, 2}$. 
Inspired by \citet{yaras2023law,yaras2024compressible}, we define corresponding $\bm V_{l, 2}$ and $\bm U_{l, 2}$ in $\bm W_l$ for $l = 2, \dots, L - 1$ as follows:
\begin{align}
\label{eq:V_l_U_l_recursive}
    \bm V_{l, 2} = \bm U_{l - 1, 2} \quad \text{and} \quad \bm U_{l, 2} = \bm W_l(0) \bm V_{l, 2} / \epsilon,
\end{align}
assuming $\bm W_l(0)$ is initialized as an $\epsilon$-scaled semi-orthogonal matrix. Equivalently,
\begin{align}
\label{eq:V_l_U_l_nonrecursive}
    &\bm V_{l, 2} = \bm W_{l - 1}(0) \bm W_{l - 2}(0) \cdots \bm W_1(0) \bm V_{1, 2} / \epsilon^{l - 1} \quad \text{and} \quad \bm U_{l, 2} = \bm W_l(0) \bm W_{l - 1}(0) \bm W_{l - 2}(0) \cdots \bm W_1(0) \bm V_{1, 2} / \epsilon^l. 
\end{align}
Define $\bm V_l := \begin{bmatrix}
    \bm V_{l, 1} & \bm V_{l, 2}
\end{bmatrix} \in \mbb{R}^{m \times m}$ and $\bm U_l := \begin{bmatrix}
    \bm U_{l, 1} & \bm U_{l, 2}
\end{bmatrix} \in \mbb{R}^{m \times m}$, where $\bm V_{l, 1} \in \mbb{R}^{m \times (m - p)}$ and $\bm U_{l, 1} \in \mbb{R}^{m \times (m - p)}$ are orthogonal to $\bm V_{l, 2}$ and $\bm U_{l, 2}$ . Finally, let
\begin{align}
\label{eq:deep_layer_decomp}
    &\widetilde{\bm W}_{l}(t) = \begin{bmatrix}
        \bm U_{l, 1}^\top \bm W_l(t) \bm V_{l, 1} & \bm U_{l, 1}^\top \bm W_l(t) \bm V_{l, 2} \\
        \bm U_{l, 2}^\top \bm W_l(t) \bm V_{l, 1} & \bm U_{l, 2}^\top \bm W_l(t) \bm V_{l, 2} ,
    \end{bmatrix} := \begin{bmatrix}
        \widetilde{\bm W}_{l, 1}(t) & \widetilde{\bm W}_{l, 2}(t) \\
        \widetilde{\bm W}_{l, 3}(t) & \widetilde{\bm W}_{l, 4}(t)
    \end{bmatrix} \in \mbb{R}^{m \times m},
\end{align}
which implies $\bm W_l(t) = \bm U_l \widetilde{\bm W}_l(t) \bm V_l^\top$. Note $\widetilde{\bm W}_{l, 1}(t) \in \mbb{R}^{(m - p) \times (m - p)}$, $\widetilde{\bm W}_{l, 2}(t) \in \mbb{R}^{(m - p) \times p}$, $\widetilde{\bm W}_{l, 3}(t) \in \mbb{R}^{p \times (m - p)}$, and $\widetilde{\bm W}_{l, 4}(t) \in \mbb{R}^{p \times p}$.

\paragraph{Results.} We plot the normalized squared distance of each $\widetilde{\bm W}_{l, i}(t)$ from their initializations, i.e., $\| \widetilde{\bm W}_{l, i}(t) - \widetilde{\bm W}_{l, i}(0)\|_F^2 / \| \widetilde{\bm W}_l(t) - \widetilde{\bm W}_l(0) \|_F^2$. The results are shown in \Cref{fig:beyond_theory_deep_nets_and_activations}, averaged over $10$ trials. For all $l \in [L - 1]$, $\widetilde{\bm W}_{l, 1}(t)$ again accounts for most of the changes in $\widetilde{\bm W}_l(t)$, and thus $\bm W_l(t)$ themselves. This implies \textbf{the GD updates of the deeper layers are also concentrated within low-dimensional subspaces.}

\subsection{Loss Function, Optimizer, and Unwhitened Data}
\label{ssec:beyond_theory_sgd}

\begin{figure*}[t!]
    \centering
    \includegraphics[width=0.49\textwidth]{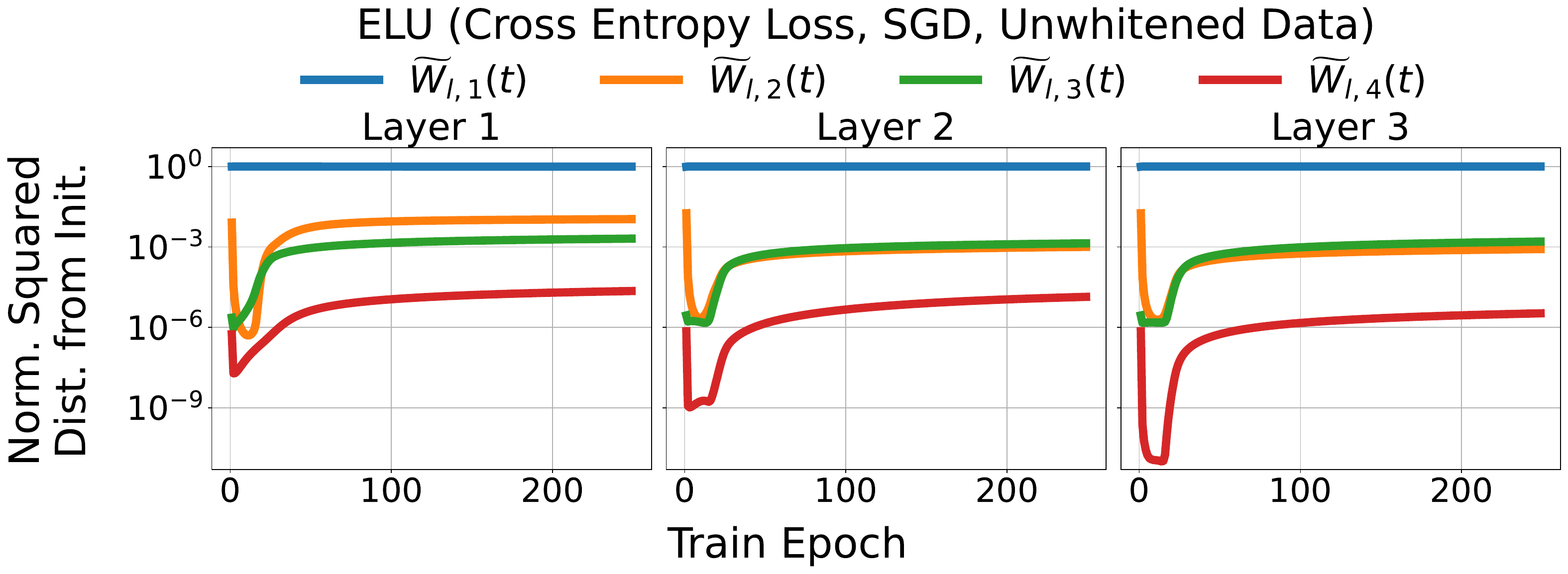}
    \includegraphics[width=0.49\textwidth]{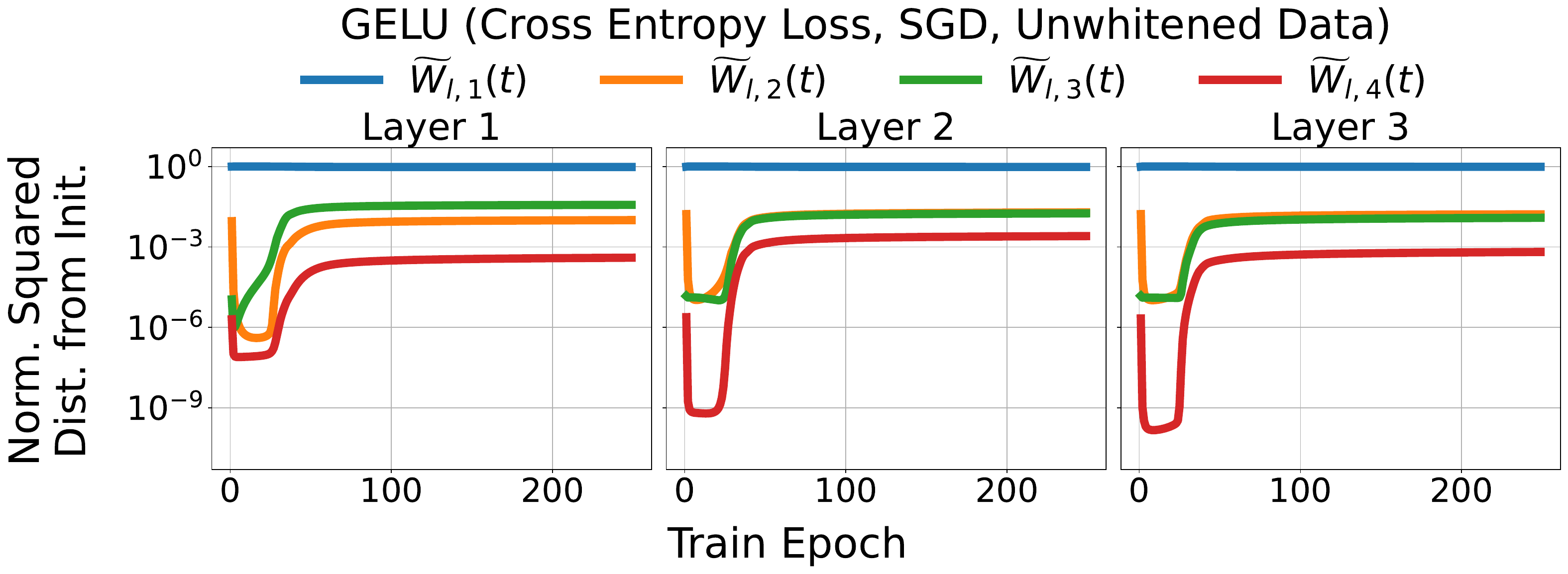}
    \includegraphics[width=0.49\textwidth]{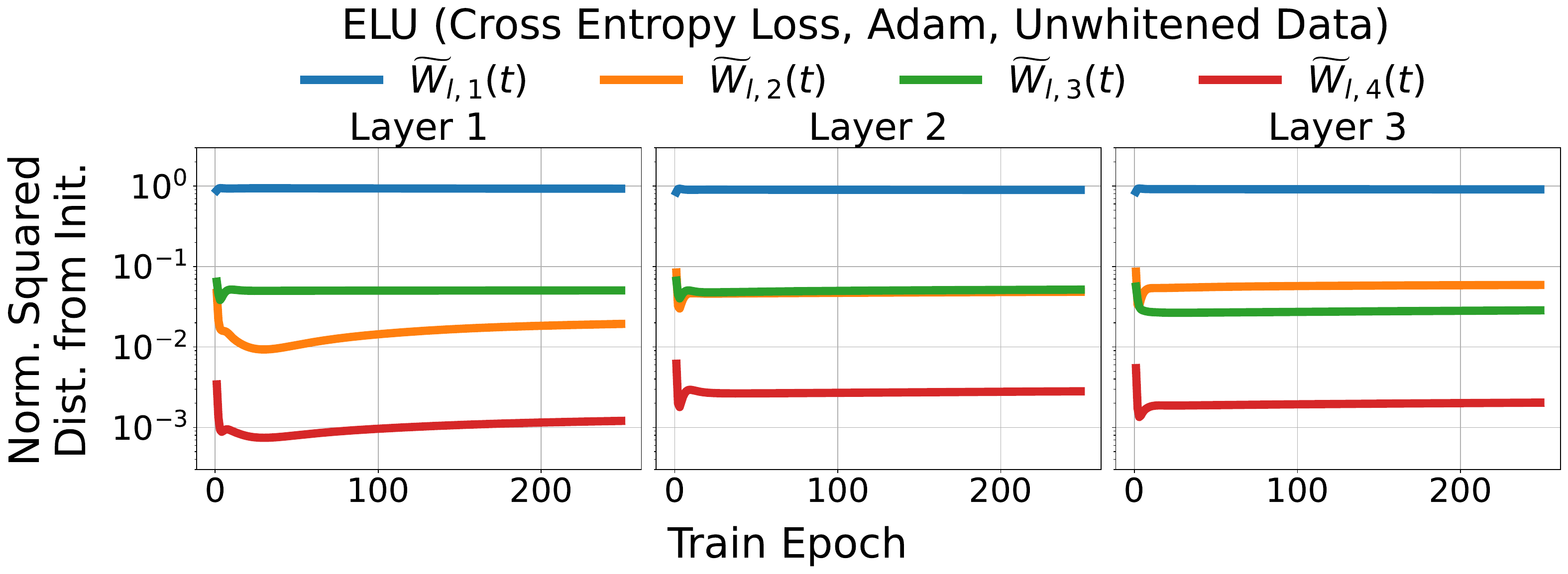}
    \includegraphics[width=0.49\textwidth]{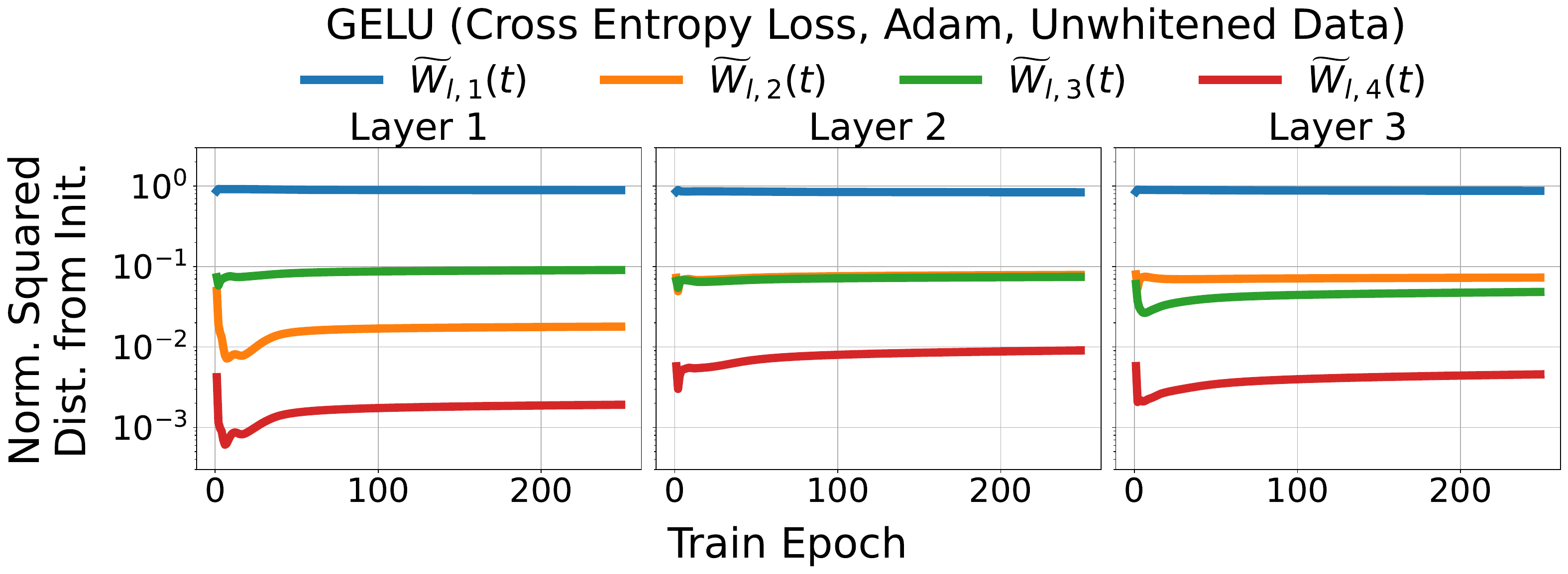}
    \includegraphics[width=0.49\textwidth]{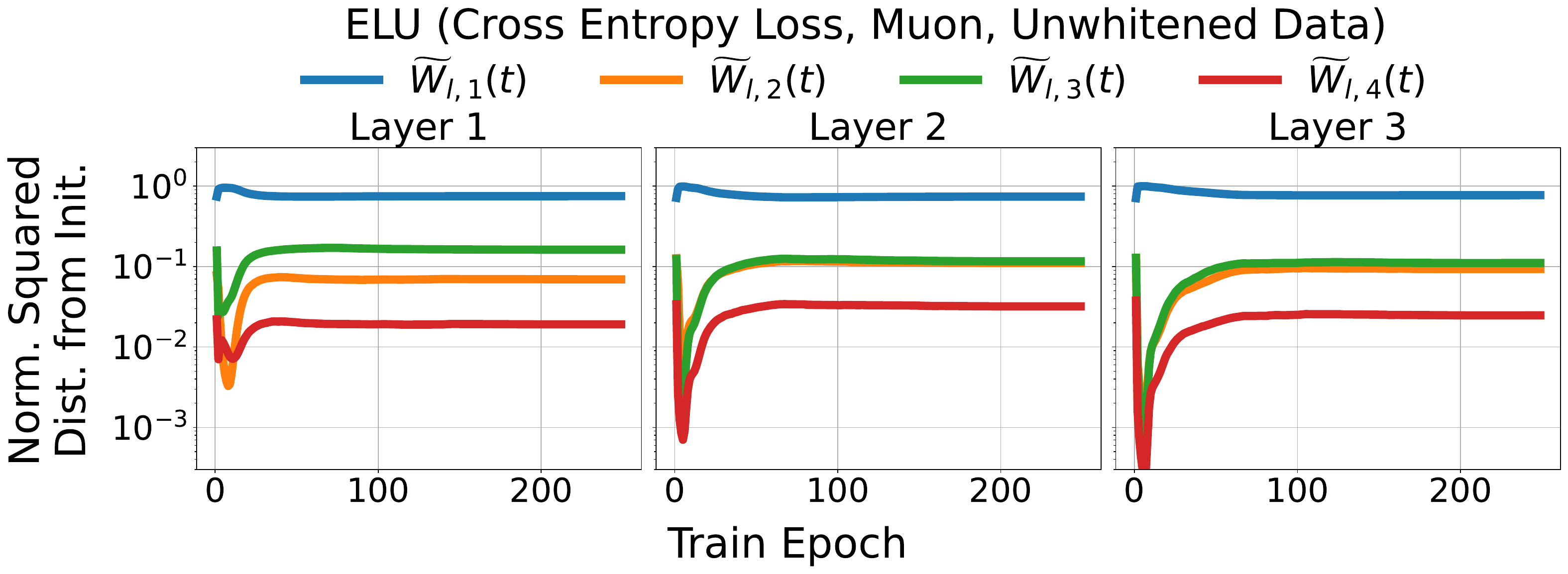}
    \includegraphics[width=0.49\textwidth]{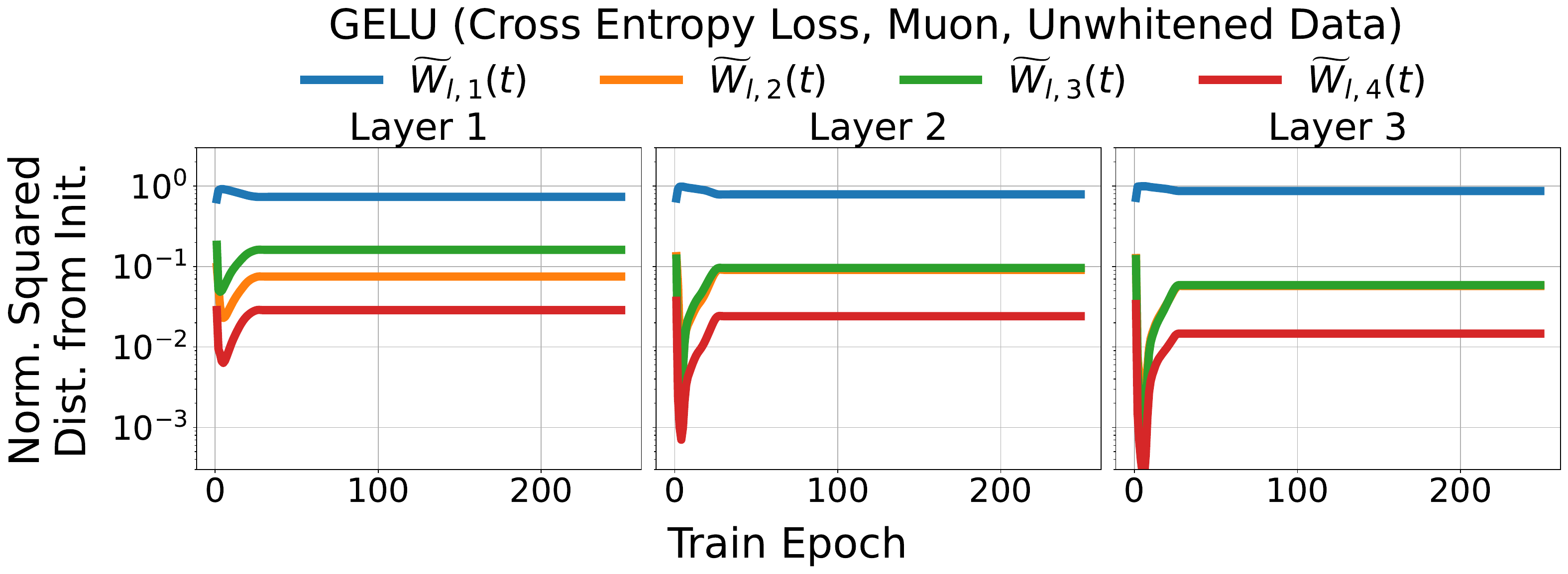}

    \caption{Training deep MLPs with SGD plus momentum (top row), or with Adam (bottom row), on unwhitened input data using cross-entropy loss approximately maintains the previously observed low-rank training dynamics.}
    \label{fig:beyond_theory_optimizer_loss_unwhitened}
\end{figure*}

Next, we investigate the impact of the loss function, optimizer, and unwhitened data. Thus far, we have been training MLPs using full batch GD on squared-error loss with whitened input data. Here, we generated $\bm X$ as described in \Cref{ssec:additional_sims_theory_verify}, but did not whiten $\bm X$. We also trained the networks using 1) minibatch SGD with momentum, 2) Adam, and 3) Muon, all on cross-entropy loss; we again defer specific experimental details to \Cref{ssec:additional_sims_beyond_theory}. 

\Cref{fig:beyond_theory_optimizer_loss_unwhitened} again shows the normalized squared distance of each $\widetilde{\bm W}_{l, i}(t)$ from their initializations. Once again, $\widetilde{\bm W}_{l, 1}(t)$ accounts for a majority of the change in $\widetilde{\bm W}_l(t)$, and thus $\bm W_l(t)$, albeit less so when trained with Adam and Muon. Regardless, this implies the low-rank training phenomenon approximately holds in more realistic training settings. %, as low-rank training dynamics also emerge in smooth-activation networks trained using SGD with momentum.

% \subsection{Unwhitened Data}
% \label{ssec:beyond_theory_unwhitened_data}
% Finally, we consider \emph{unwhitened} input data. We repeat the experiments from \Cref{ssec:beyond_theory_sgd}, but skip the whitening pre-processing step on $\bm X$. Here, we set $\eta = 10^{-4}$ for the $\elu$ network, and $\eta = 5 \times 10^{-4}$ for $\gelu$. \Cref{fig:beyond_theory_unwhitened} once again shows $\| \widetilde{\bm W}_{l, i}(t) - \widetilde{\bm W}_{l, i}(0) \|_F$, where again $\widetilde{\bm W}_{l, 1}(t)$ accounts for a large majority of the change in $\widetilde{\bm W}_l(t)$. This implies our results are not unique to whitened input data.

% \begin{figure*}
%     \centering
%     \includegraphics[width=0.49\textwidth]{arxiv/figs/beyond_theory/ce_loss_SGD_batch_size_32/ELU_dist_from_init_unwhitened_data.pdf}
%     \includegraphics[width=0.49\textwidth]{arxiv/figs/beyond_theory/ce_loss_SGD_batch_size_32/GELU_dist_from_init_unwhitened_data.pdf}
%     %\includegraphics[width=0.49\textwidth]{arxiv/figs/beyond_theory/ce_loss_SGD_batch_size_32/SiLU_dist_from_init_unwhitened_data.pdf}
%     \caption{Even when the input data is not whitened, $\widetilde{\bm W}_{1, 1}(t)$ still accounts for a large majority of the change in $\widetilde{\bm W}_1(t)$.}
%     \label{fig:beyond_theory_unwhitened}
% \end{figure*}

    \section{Low-Rank Parameterization in MLPs} \label{sec:low_rank_param}
From \Cref{sec:two_layer_theory,sec:beyond_theory}, in an MLP with smooth activation functions, we find each weight matrix is mostly updated within a layer-dependent $2K$-dimensional subspace. Based on these insights, we find there exists a \textbf{low-rank parameterization} that, when initialized properly, achieves \emph{near-equivalent} performance compared to their fully parameterized counterpart. For some width parameter $r$, we refer to the following parameterization as a \textbf{low-rank MLP:} %\qq{fix the equation (11)}
\begin{equation} \label{eq:low_rank_mlp}
    \bm W_L \phi\Big( \widetilde{\bm U} \widetilde{\bm W}_{L - 1} \phi \big( \widetilde{\bm W}_{L - 2} \dots \widetilde{\bm W}_2 \phi\big( \widetilde{\bm W}_1 \widetilde{\bm V}^\top \bm X \big) \dots \big) \Big),
\end{equation}
where $\bm W_L \in \mathbb{R}^{K \times m}$, $\widetilde{\bm U} \in \mbb{R}^{m \times r}, \widetilde{\bm V} \in \mathbb{R}^{d \times r}$ and $\widetilde{\bm W}_l \in \mathbb{R}^{r \times r}$ for $l \in [L - 1]$. 

\paragraph{Initializing and training $\widetilde{\bm U}$ and $\widetilde{\bm V}$.} For ease of exposition, suppose $m = d$, so $m - p = d - (d - 2K) = 2K$. In this case, from \Cref{sec:beyond_theory}, we observe each layer is updated in a $2K$-dimensional subspace, where each layer's subspace is defined from \eqref{eq:V_l_U_l_nonrecursive}. Thus, setting $r = 2K$, and initializing $\bm W_l(0)$ as $\epsilon$-scaled orthogonal matrices for all $l \in [L - 1]$, we find $\widetilde{\bm U}$ and $\widetilde{\bm V}$ should be initialized as follows. Recall $\mc{S}^{(1)}_{small}$ in \eqref{eq:smooth_small_update_subspace_W1_deep_net}, and denote its orthogonal complement as $\mc{S}^{(1)}_{big}$. We initialize $\widetilde{\bm V} \in \mbb{R}^{d \times 2K}$ as an orthonormal basis of $\mc{S}^{(1)}_{big}$, and then initialize $\widetilde{\bm U}$ according to \eqref{eq:V_l_U_l_nonrecursive}, i.e., $\widetilde{\bm U} = \bm W_{L - 1}(0) \cdots \bm W_1(0) \widetilde{\bm V} / \epsilon^{L - 1} \in \mbb{R}^{d \times 2K}$. We call this the \textbf{``$\mathbf{\mc{S}_{big}}$ initialization''}, which we find to be critical to its performance. After initialization, we train $\widetilde{\bm U}$ and $\widetilde{\bm V}$ using the same learning rate as $\widetilde{\bm W}_l$.

\subsection{Experiments}
In this section, we provide empirical evidence showing that, when $\widetilde{\bm U}$ and $\widetilde{\bm V}$ are properly initialized, \textbf{the low-rank MLP %\eqref{eq:low_rank_mlp}
achieves near-equivalent performance} (test loss and/or accuracy) \textbf{compared to the fully parameterized MLP}, %in \eqref{eq:orig_mlp}}, 
assuming the hyperparameters remain the same when training the two types of networks. We again ran all experiments in \texttt{PyTorch} using an NVIDIA A40 GPU.

We emphasize that the goal of these experiments is not necessarily to find optimal hyperparameters to maximize performance. Rather, we aim to show that, for a given MLP that is trained using some (possibly sub-optimal) training recipe, there exists a low-rank MLP that achieves similar performance if initialized properly.

\begin{figure*}[t]
    \centering
    \includegraphics[width=0.49\textwidth]{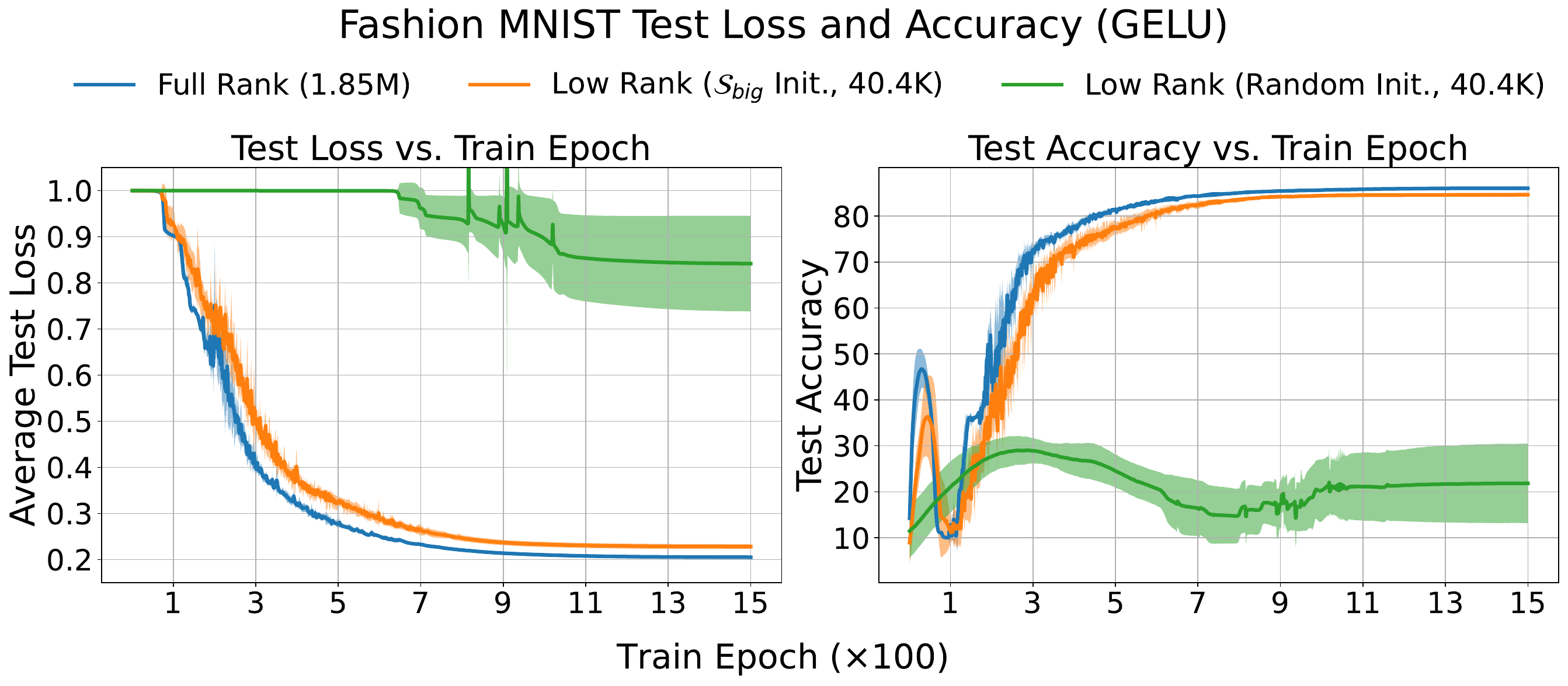}
    \includegraphics[width=0.49\textwidth]{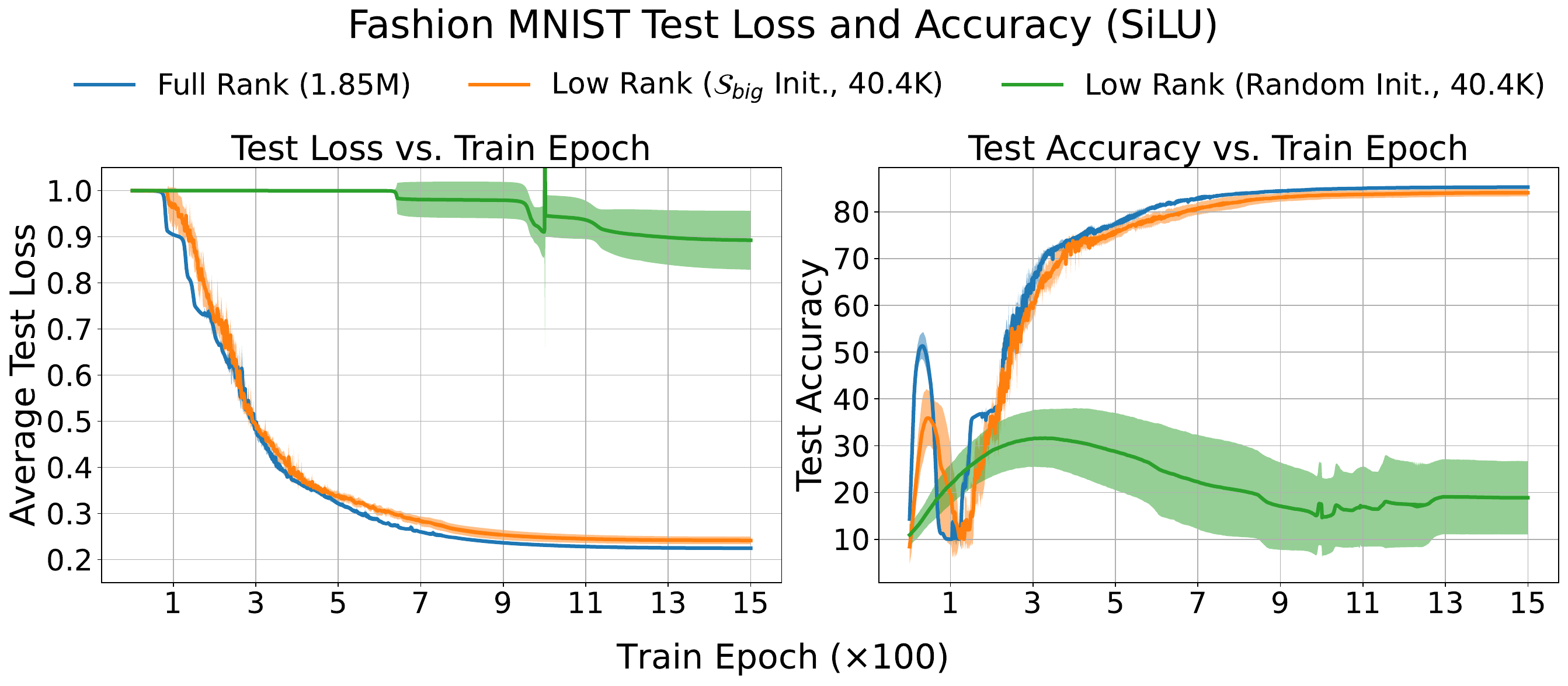}
    \caption{The properly-initialized low-rank MLP (orange) achieves nearly-identical test loss and accuracy trajectories as the fully parameterized MLP (blue). Meanwhile, the low-rank MLP with random subspace initialization (green) gets stuck at a noticeably worse local minimum.}
    \label{fig:fashion_mnist}
\end{figure*}

\begin{figure*}[t]
    \centering
    \includegraphics[width=0.49\textwidth]{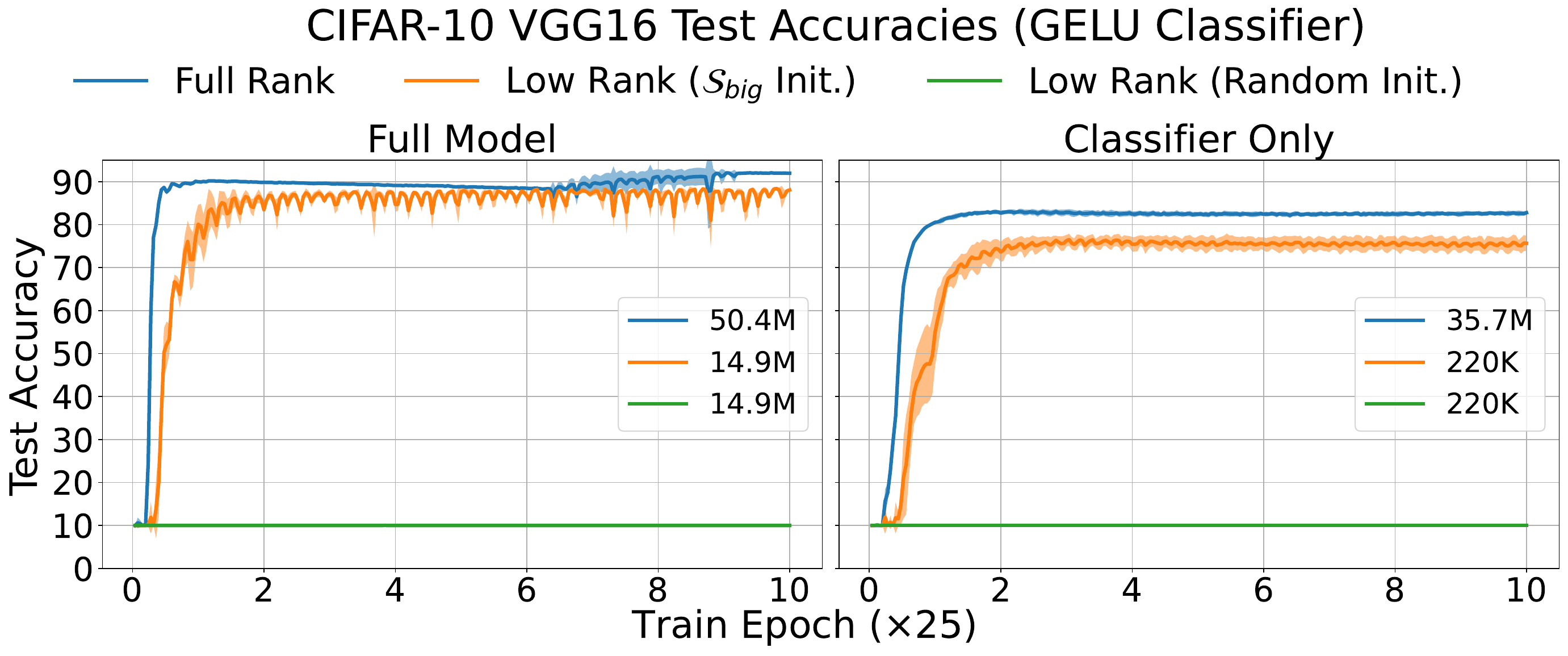}
    \includegraphics[width=0.49\textwidth]{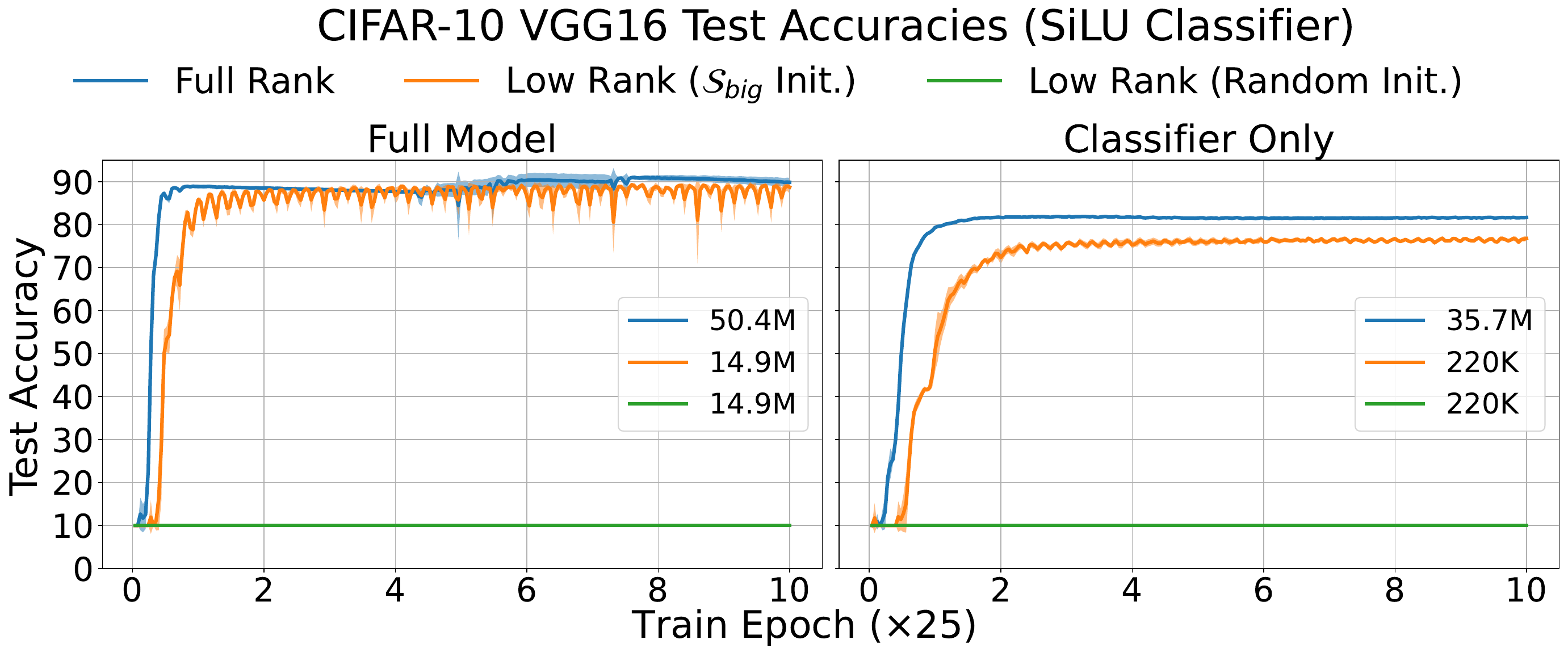}
    \caption{The properly-initialized low-rank MLP classifier head (orange) achieves similar test accuracies to the fully parameterized version (blue) in VGG-16 on CIFAR-10, while the random-subspace initialized low-rank MLP (green) is stuck in the random guessing stage.}
    \label{fig:cifar10}
\end{figure*}

\subsubsection{Fashion MNIST using MLPs}
\label{sssec:fashion_mnist}
We trained fully parameterized and low-rank MLPs on Fashion MNIST \citep{xiao2017fashion}. We pre-processed all images by re-scaling the pixel values to lie in $[0, 1]$, and then normalized them to have zero mean and unit variance. 

%\paragraph{Data.} %Each data sample is a $28 \times 28$ grayscale image, so $d = 28^2 = 784$. The training set contains $N_k = 6 \times 10^3$ images from all $K = 10$ classes, so $N = 6 \times 10^4$. The test set contains $1000$ images per class. 

\paragraph{Network architecture and training.} We used $L = 4$ layer networks with $\gelu$ and $\silu$ activations, setting $m = d = 784$ and $r = 2K$. We initialized each $\bm W_l$ and $\widetilde{\bm W}_l$ as $\epsilon$-scaled (semi-)orthogonal matrices, with $\epsilon = 0.1$. We initialized  $\widetilde{\bm U}$ and $\widetilde{\bm V}$ in two different ways: 1) the $\mc{S}_{big}$ initialization scheme, and 2) as random semi-orthogonal matrices. We trained all networks on all $N = 5 \times 10^4$ training images for $T = 1500$ epochs using full-batch GD on (total) squared-error loss, which aligns with the training algorithm in our theoretical setting. Since this is a classification setting, we used one-hot encoded labels. Finally, we set $\eta = 10^{-5}$ and used a cosine annealing scheduler. %We defer the remaining experimental details to \Cref{ssec:additional_fmnist_details}. 

%We trained all networks on all $N = 5 \times 10^4$ training images for $T = 1500$ epochs using full-batch GD on (total) squared-error loss, which aligns with the training algorithm in our theoretical setting. We set $\eta = 5 \times 10^{-6}$ for the $\elu$ network, and $\eta = 10^{-5}$ otherwise, all with a cosine annealing scheduler. 

%The uncompressed models each contained $\mathbf{1.85}$\textbf{M} trainable parameters, while the compressed counterparts each contained $\mathbf{40.4}$\textbf{K} parameters. We note all of the compressed models in this setting are \emph{underparameterized}

%1) as $\bm U_{L-1}$ and $\bm V_1$ defined in \Cref{ssec:compression}, and 2) as random semi-orthogonal matrices. We refer to the first method as ``$\mc{S}^\perp$ initialization'' with $\mc{S}^\perp$ defined above, and the second method as ``random subspace initialization'' (or ``random initialization'').  We trained all networks for $T = 1500$ epochs on sum of squared-error loss. We set $\eta = 5 \times 10^{-6}$ for the $\elu(\cdot)$ networks, and $\eta = 10^{-5}$ for the $\gelu(\cdot)$ and $\silu(\cdot)$ networks. We also used a cosine annealing scheduler during training. The uncompressed models each contained $\mathbf{1.85}$\textbf{M} trainable parameters, while the compressed counterparts each contained $\mathbf{40.4}$\textbf{K} parameters. We note all of the compressed models in this setting are \emph{underparameterized}.

\paragraph{Results.} \Cref{fig:fashion_mnist} shows the test losses and accuracies for the fully parameterized and low-rank MLPs, averaged over $5$ trials. Note that we trained the networks using vanilla full-batch GD, so all networks (likely) converged to local minima. Nevertheless, keeping the training algorithm and hyperparameters fixed between the two parameterizations, the low-rank MLP with $\mc{S}_{big}$ initialization achieves nearly identical performance to the fully parameterized MLP. Meanwhile, the low-rank MLP with random subspace initialization gets stuck in a much worse local minimum. We provide an ablation study on the initialization of $\widetilde{\bm U}$ and $\widetilde{\bm V}$ in \Cref{ssec:angle_ablation}.

\subsubsection{CIFAR-10 using VGG-16}
\label{sssec:cifar10}
Here, we train a modified version of VGG-16 \citep{simonyan2014very} on CIFAR-10 \citep{krizhevsky2009learning}. We pre-processed each image the same way as in \Cref{sssec:fashion_mnist}, but after re-sizing the images to be of size $224 \times 224 \times 3$ using bilinear interpolation.

\paragraph{Network architecture.}  We used modified VGG-16 models \citep{simonyan2014very}. The original model contains a convolutional network as a feature extractor, and then $L = 3$ layer $\relu$ MLP classifier head. Due to limited computational resources, we reduced the final average pooling layer size from $7 \times 7$ to $3 \times 3$, and thus changed the input dimension of the classifier head from $512 \times 7 \times 7$  to $512 \times 3 \times 3$. This  reduced the full model size from about $138$M parameters to about $50$M. We also  replaced the $\relu$ in the original classifier head with $\gelu$ and $\silu$. 

\paragraph{Training.} In all experiments, we initialized the convolutional layers using ImageNet pre-trained weights, and initialized $\bm W_l$ and $\widetilde{\bm W}_l$ in the classifier head in the same manner as in \Cref{sssec:fashion_mnist}, again with $r = 2K$. To initialize $\widetilde{\bm U}$ and $\widetilde{\bm V}$, and used 1) an (unbiased) estimator of the first MLP-layer gradient via a random batch of $128$ images to estimate $\mc{S}_{big}^{(1)}$, and 2) random semi-orthogonal matrices. For all model types, we conducted two training paradigms:
\begin{enumerate}[leftmargin=*, labelsep=0.5em]

    \item \textbf{Full model training.} We fined-tuned the convolutional layers from the ImageNet pre-trained weights, in addition to training the classifier head from initialization.
    \item \textbf{Classifier-only training.} We kept the convolutional layers frozen at the ImageNet pre-trained weights, and only trained the classifier head from initialization.  
\end{enumerate}

\medskip 

\noindent We trained all models on cross-entropy loss using SGD with momentum for $250$ epochs. We set $\eta = 5 \times 10^{-3}$ with a cosine annealing scheduler, the batch size to $128$, the momentum to $0.9$, and weight decay to $5 \times 10^{-4}$. %We provide remaining details in \Cref{ssec:additional_cifar10_details}. %for $T = 250$ epochs. We set $\eta = 5 \times 10^{-3}$ with a cosine annealing scheduler, the batch size to $128$, the momentum to $0.9$, and weight decay to $5 \times 10^{-4}$.

\paragraph{Results.} \Cref{fig:cifar10} shows the test accuracies of VGG-16 with the fully-parameterized and low-rank classifier heads, again averaged over $5$ trials. Under full model training, VGG-16 with the low-rank MLP and $\mc{S}_{big}$ initialization consistently achieved very similar test accuracy to the corresponding model with the fully parameterized MLP. Meanwhile, using the same training hyperparameters, the VGG-16 with random-subspace initialized MLP did not escape the random guessing stage. 

When we only trained the classifier head, there was a noticeable gap in test accuracy (about $5\% - 10\%$) between the low-rank MLP with $\mc{S}_{big}$ initialization and the fully parameterized version. In \Cref{ssec:rank_ablation}, we observe setting the width $r = 4K$ noticeably closes this gap. We also note the low-rank MLP with random $\widetilde{\bm U}, \widetilde{\bm V}$ initialization again did not escape the random guessing stage.

Finally, the model with the low-rank classifier head took noticeably more epochs to (nearly) match the performance of the fully parameterized version. We believe this is because we initialized $\widetilde{\bm V}$ and $\widetilde{\bm U}$ using an \emph{estimate} of $\mc{S}_{big}^{(1)}$ via a minibatch of data, rather than using all of the data samples to determine the true $\mc{S}_{big}^{(1)}$.

    \section{Related Work}
\label{sec:related}
In this section, we provide detailed discussions on related works. 

\paragraph{Low-dimensional learning in neural networks.} A recent line of empirical work has shown the GD and SGD dynamics of deep neural networks (DNNs) occur within a small subset of the full parameter space. Specifically, \cite{li2018measuring,li2022low,larsen2022many} observed DNNs can be trained in low-dimensional subspaces of the parameter space. Our work is similar to these works, since we observe and prove a similar phenomenon in MLP architectures. The main difference is \cite{li2018measuring,larsen2022many} showed low-dimensional training is \emph{possible} in \emph{random} subspaces, and \cite{li2022low} determined the subspace to train in via the full-parameter training trajectory of an initial training phase. Meanwhile, in our work, we show the weights are \emph{naturally} trained within low-dimensional subspaces that depend on the weights and first-layer gradient at initialization of the fully parameterized MLP. Identifying these subspaces only requires a single forward and backward pass, and so no training is needed to find the optimization subspaces in our work.

Similarly, \cite{frankle2019lottery} proposed the lottery ticket hypothesis: dense, randomly initialized neural networks contain  sparse sub-networks that, if trained separately, achieve similar performance to the full network in similar training time. Our work is similar in that we identify low-dimensional spaces where the training largely occurs within. The main difference is that in \cite{frankle2019lottery}, although the sparse winning lottery tickets emerge at random initialization, identifying these lottery tickets requires an initial training and pruning phase in the full parameter space. In our work, the low-dimensional subspaces in MLPs can be completely identified at initialization of the original network using only a single forward and backward pass.

%\textcolor{purple}
{Another similar work is \cite{schotthofer2022low}, which proposes a low-rank training algorithm that explicitly sets low-rank constraints on the weight matrices. In particular, they truncate the randomly-initialized weights using SVDs, and then update the resulting factors directly by integrating differential equations. To derive error bounds on their algorithm, \cite{schotthofer2022low} assumes the initialized weights have already been truncated to a low-rank approximation, and analyze the training dynamics of the already-truncated weights. In our work, we show that even when the weights in MLPs are trained in their full parameter space, the updates are largely concentrated within low-dimensional subspaces. In other words, even without explicit low-rank constraints on the weights, gradient descent undergoes implicit subspace dynamics.}

Furthermore, \cite{wang2026universal} developed theoretical results showing  two-layer networks with $m$ neurons can instead be trained with a $\mc{O}\left( \operatorname{poly}\log(m)\right)$-sized subset of the original neurons to achieve nearly identical training dynamics. This subset can be identified at initialization of the original, fully parameterized model via a moment-matching algorithm. Our work is similar in that we also identify compressible dynamics in nonlinear networks, but instead within low-dimensional subspaces of the original parameter space. Similar to \cite{wang2026universal}, these subspaces can be identified at initialization of the original network.

%\textcolor{purple}
{Additionally, \cite{huh2023simplicitybias} empirically studied low-rank simplicity biases in neural networks. They primarily studied the rank of the \emph{data representations} before and after network training, observing that deeper networks typically result in lower-rank penultimate layer embeddings. While \cite{huh2023simplicitybias} briefly studied low-rank structure in the weight matrices, they focused on how the weights' \emph{singular values} evolve throughout training. In our work, we analyzed the training dynamics of the weight matrices in MLP models, showing that these weights mostly get trained in invariant subspaces.}

Finally, there is a rich line of literature on the emergence of low-rank structure in matrix factorization \citep{gunasekar2017implicit,li2021towards} and deep linear networks  \citep{arora2019convergence,gidel2019implicit,yaras2023law,yaras2024compressible,kwon2024efficient}; see \cite{vardi2023implicit} for a survey. As mentioned in \Cref{sec:intro}, among these works, \cite{yaras2023law,yaras2024compressible,kwon2024efficient} are most closely related to ours. They theoretically proved that when deep linear networks are trained via GD, each weight matrix is updated in an unchanging low-dimensional subspace that is determined at initialization. This subspace is also dependent on the weights and the first-layer gradient at initialization. \cite{yaras2023law} proved this in a setting with whitened input data, while \cite{yaras2024compressible,kwon2024efficient} focused on deep matrix factorization (i.e., no input data). This work is largely inspired by the findings of \cite{yaras2023law}, as we extend the investigation to nonlinear MLP architectures.

\paragraph{Low rank gradients in neural networks.} Another line of work studied the emergence of low-rank gradients in deep neural networks. In particular, \cite{gur2018gradient} empirically observed the gradients become well-aligned with the corresponding Hessian's top eigenspace, which remains approximately constant throughout long training periods. Furthermore, \cite{ba2022high,zhao2024galore,jaiswal2025from,sonthalia2025low} theoretically analyzed the emergence of low-rank gradients in nonlinear networks. For instance, \cite{ba2022high} showed under a student-teacher model with Gaussian input data and scalar outputs, the gradient of the first layer in a two-layer nonlinear network is approximately rank-$1$. Likewise, under the same network assumptions but with more relaxed data assumptions, \cite{sonthalia2025low} showed the gradient is approximately rank two. In a related direction, \cite{zhao2024galore,jaiswal2025from} studied the low-rank property of the gradients of \emph{reversible} neural networks. Specifically, \cite{zhao2024galore} provided upper bounds on the gradient stable ranks in reversible networks, while \cite{jaiswal2025from} showed the gradients asymptotically align with the top eigenspace of their corresponding Hessians, in line with the observations of \cite{gur2018gradient}. Our work is related to the emergence of low-rank gradients in the following way. In \Cref{thm:smooth_main_result_main_body}, we show that the upper bound on the change in the perturbation terms of $\widetilde{\bm W}_1(t)$ depends on the gradient's $(K + 1)^{th}$ singular value. The low-rank GD dynamics in $\bm W_1(t)$ relies on the gradient being approximately rank-$K$.

\paragraph{Implicit bias in two-layer networks.} Finally, several other works studied the implicit bias of gradient flow (GF) and GD towards low-rank weights in two-layer nonlinear networks. For example, 
\cite{frei2023implicit,kou2023implicit} showed when two-layer ReLU and Leaky-ReLU networks are trained using GF or GD, the first-layer weights have a bounded stable rank at convergence. \cite{min2024early} showed a similar result in ReLU networks for orthogonally separable data. Our analysis on two-layer networks differs from these works, since we show \emph{each GD update} of the first-layer weights mostly occurs in a low-dimensional subspace. Meanwhile, these works focus on the low-rankness of the first-layer weights \emph{at convergence.}

%\textcolor{purple}{
\subsection{FLOP Comparison Between Low-Dimensional Training Approaches}
In this section, we provide a FLOP comparison between the low-rank MLP in \Cref{sec:low_rank_param}, Lottery Ticket Hypothesis (LTH) \citep{frankle2019lottery}, and Dynamic Low-Rank Training (DLRT) \citep{schotthofer2022low}. Specifically, we compute the number of FLOPs to initialize each method, as well as each optimization step. For the original dense model, we use a $L = 3$ layer MLP with $m = d = 784$, which contains $1.2M$ parameters. We also set the batch size to $256$. We set the low-rank MLP width to $r = 20$, the LTH sparsity level to $0.968$, and keep the DLRT rank constant at $10$. The low-rank MLP and LTH models contain $40K$ parameters, while DLRT contains $39.4K$ parameters. To measure the per-iteration FLOPs for low-rank MLP and LTH, we used the number of FLOPs in a single SGD step.
%} 

%\textcolor{purple}{
\Cref{tab:flops} shows the resulting initialization and per-step FLOPs. The low-rank MLP in \Cref{sec:low_rank_param} requires fewer FLOPs to initialize than DLRT, and requires fewer than LTH when LTH requires at least $13$ full-model iterations. The low-rank MLP also requires fewer per-iteration FLOPs than DLRT; this is because DLRT requires one-step integrators and QR decompositions in each training step, while the low-rank MLP can be trained with standard SGD steps. The theoretical best per-iteration FLOPs in LTH is lower than that of the low-rank MLP. However, it is difficult to achieve this theoretical best LTH FLOPs in practice: LTH introduces unstructured sparsity in the weights, which modern GPUs cannot fully exploit to maximize practical FLOP reduction. Meanwhile, low-rank matrix products in the low-rank MLP remain GPU friendly, so achieving practical FLOP reduction is more feasible. 
%}

\begin{table}[t]
  \centering
  \caption{Initialization and per-iteration FLOP comparison. $T$ is the number of dense training iterations before LTH pruning.}
  \label{tab:flops}
  \begin{tabular}{lrrr}
    \toprule
    Model & Trainable params & Init. FLOPs & Per-iter. FLOPs \\
    \midrule
    Dense                                  & $1.2M$ & N/A                       & $1.59\times10^{9}$ \\
    Low-rank MLP (this paper, width $ = 20$)                              & $40K$   & $1.99\times10^{10}$       & $5.35\times10^{7}$ \\
    LTH (sparsity $ = 0.968$)             & $40K$    & $T \times 1.59\times10^{9}$ & $5.13\times10^{7}$ \\
    DLRT (rank $ = 10$)  & $39.4K$    & $2.12\times10^{10}$       & $9.92\times10^{7}$ \\
    \bottomrule
  \end{tabular}
\end{table}

    \section{Conclusion}
In this work, we investigated when low-rank training dynamics emerge in nonlinear networks, using MLP architectures as a case study. We showed that in MLPs with smooth activation functions, the training dynamics are highly concentrated within invariant low-dimensional subspaces. We provided theoretical insight into this phenomenon on two-layer networks trained with GD, and our experiments show the phenomenon holds well beyond our theoretical setting. From these insights, we showed there exists a low-rank MLP parameterization that, if initialized in the proper subspaces, nearly matches the performance of their high-rank counterparts in classification tasks. 

\vspace{-0.2cm}

%\textcolor{purple}{
\subsection{Limitations and Future Work}
We discuss the limitations of this work, and possible directions for future work. 
\paragraph{Limitations.} We specifically focused on MLP architectures, so it is unclear if or how our results extend to other architectures. Furthermore, our theoretical results rely on several assumptions, such as a fixed second layer in a two-layer network, full-batch gradient descent, whitened input data, orthogonal initialization, and a small label dimension $K$ relative to the input dimension $d$. \Cref{coro:quantify_thm_bound} also relies on specific technical assumptions regarding gradient descent that are not proved explicitly in our setting. Finally, we evaluated the low-rank parameterization in \Cref{sec:low_rank_param} using small-scale models; it is unclear how the parameterization and initialization scheme would scale across different and larger architectures, such as transformers \citep{vaswani2017attention}.
%}
%\textcolor{purple}{
\paragraph{Future work.} We believe there are several promising avenues for future work. First, extending the study to consider other architectures, such as self-attention, is certainly an important direction. To sketch an idea, following the setting in \cite{ataee2023max,vasudeva2025implicit}, consider training a simplified softmax attention model using squared-error loss: 
\begin{equation*}
    \mathcal{L}(\bm \Theta) = \frac{1}{2} \sum\limits_{i=1}^N \left\| \bm y_i - f_{\bm \Theta}(\bm X_i) \right\|_2^2 \quad \text{where} \quad f_{\bm \Theta}(\bm X_i) = \bm W_V \bm X_i^\top \operatorname{softmax}\left( \bm X_i \bm W_{QK} \bm x_{i, 1} \right),
\end{equation*}
where $\bm X_i = \begin{bmatrix}
    \bm x_{i, 1} & \bm x_{i, 2} & \dots & \bm x_{i, T}
\end{bmatrix}^\top \in \mathbb{R}^{T \times d}$ is a sequence of token embeddings, $\bm y_i \in \mathbb{R}^K$ is the corresponding label, $\bm W_{QK} \in \mathbb{R}^{d \times d}$, and $\bm W_V \in \mathbb{R}^{K \times d}$. Here, $f_{\bm \Theta}\left( \bm X_i \right)$ resembles an MLP with a softmax activation function. However, softmax is not applied element-wise, there is a quadratic term $\bm X_i \bm W_{QK} \bm x_{i, 1}$, and there is an extra $\bm X_i^\top$ term after the activation. We believe generalizing the findings of this paper to this setting would be interesting. Furthermore, empirically evaluating the low-rank MLP parameterization and initialization scheme in \Cref{sec:low_rank_param} at larger scales, such as in transformer \citep{vaswani2017attention} and/or MLP-Mixer \citep{tolstikhin2021mlp} models, is another important direction for future work.

\section*{Acknowledgements}
QQ and AX acknowledge support from NSF CAREER CCF-2143904, NSF IIS-2312842, NSF IIS-2402950,
NSF CCSS-2532643, and a Google Research Scholar Award. LB and CY acknowledge support from NSF
CCF-2331590, NSF CAREER CCF-1845076, and an Intel early career award.

\clearpage

    \bibliographystyle{plainnat}
    \bibliography{reference}
    
    \clearpage
    
    \appendix

\section{Additional Simulation Details and Results} 
\label[appendix]{sec:additional_sims}
In this section, we provide additional experimental details and/or results for \Cref{fig:teaser}, \Cref{sec:motivation} (\Cref{fig:main_fig}), \Cref{ssec:main_result} (\Cref{fig:smooth_theory_verify}), and \Cref{sec:beyond_theory} (\Cref{fig:beyond_theory_deep_nets_and_activations,fig:beyond_theory_optimizer_loss_unwhitened}).

\subsection{Additional Details and Results for \Cref{fig:main_fig}}
\label[appendix]{ssec:additional_sims_main_fig}
In this section, we provide experimental details and additional results for the experiments introduced in \Cref{fig:teaser} and \Cref{sec:motivation} (\Cref{fig:main_fig}). 

\begin{figure}[h!]
    \centering
    \includegraphics[width=0.49\linewidth]{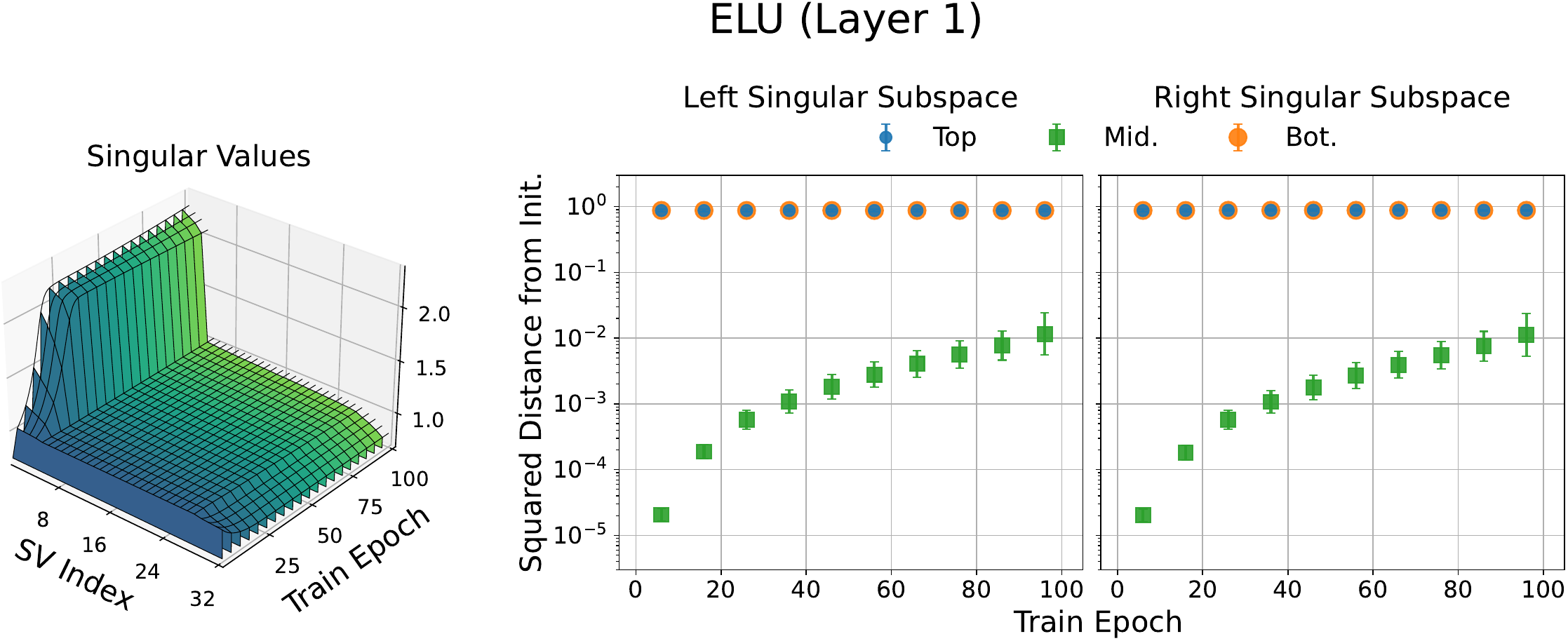}
    \includegraphics[width=0.49\linewidth]{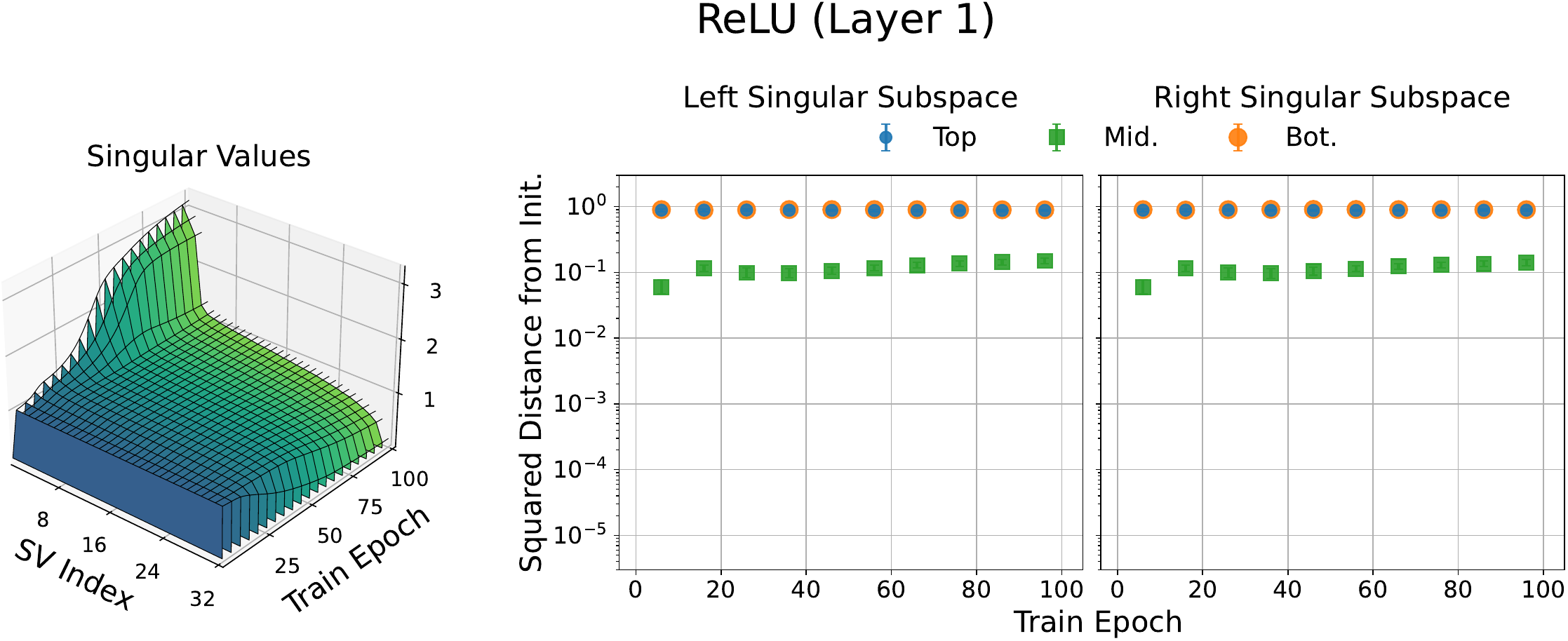}
    \includegraphics[width=0.49\linewidth]{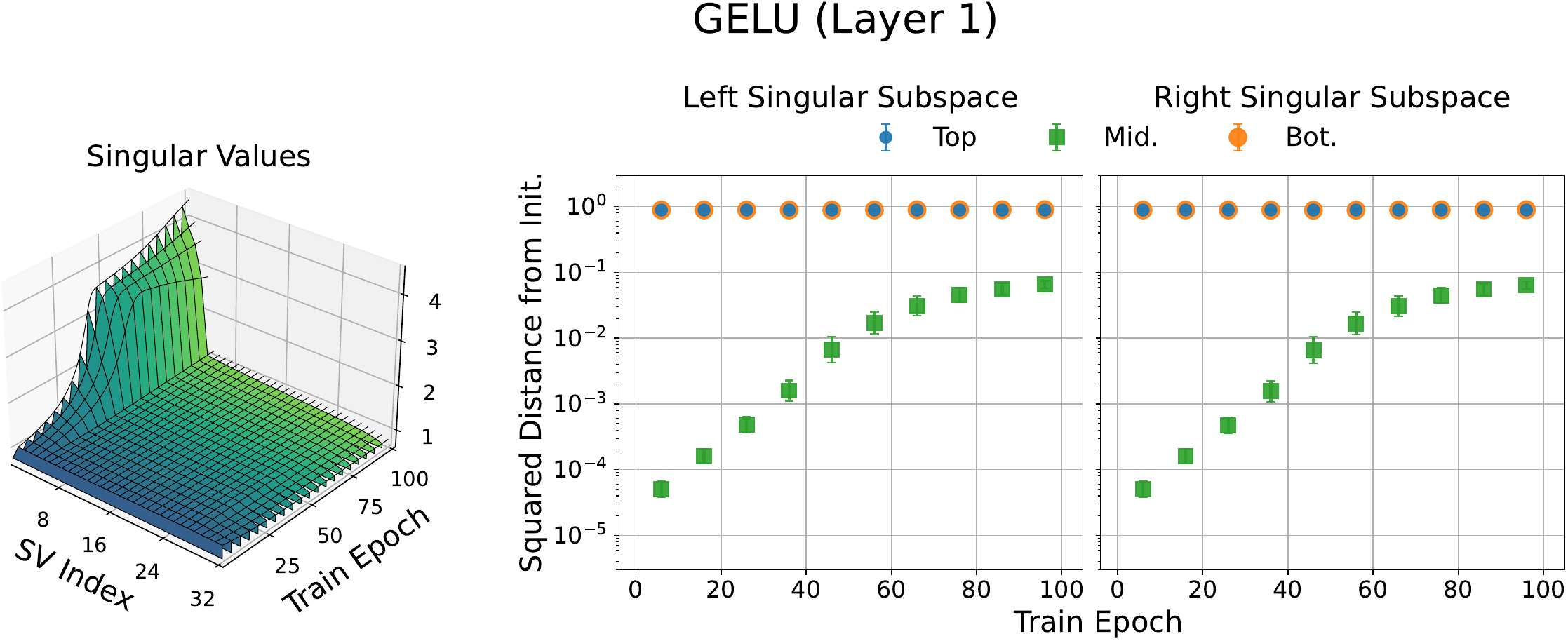}
    \includegraphics[width=0.49\linewidth]{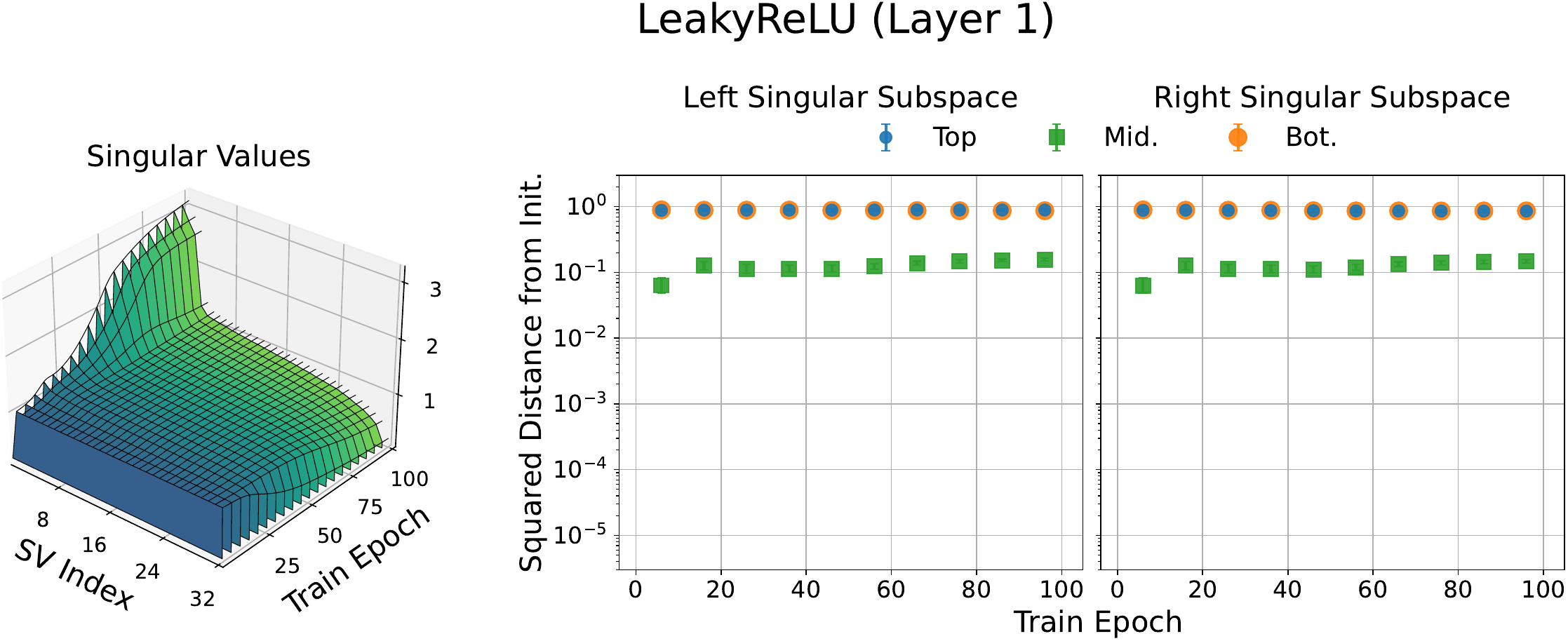}
    \includegraphics[width=0.49\linewidth]{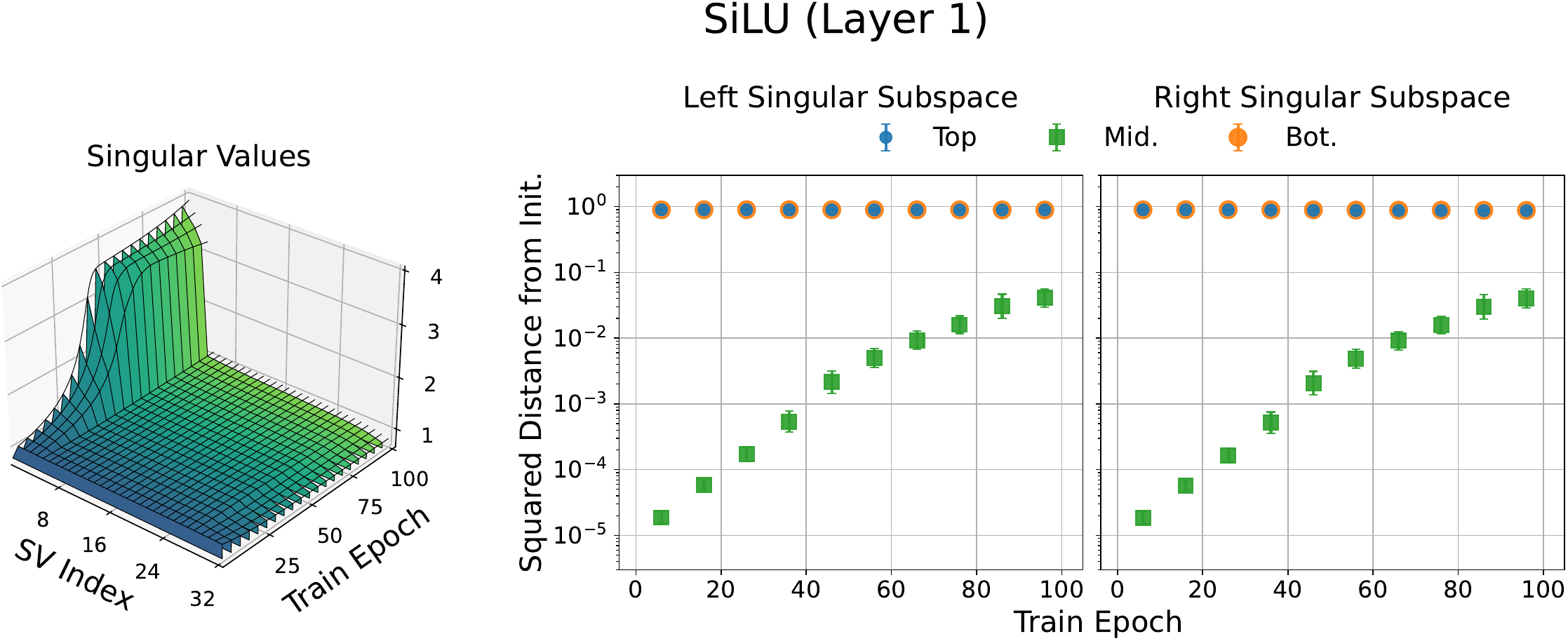}
    \includegraphics[width=0.49\linewidth]{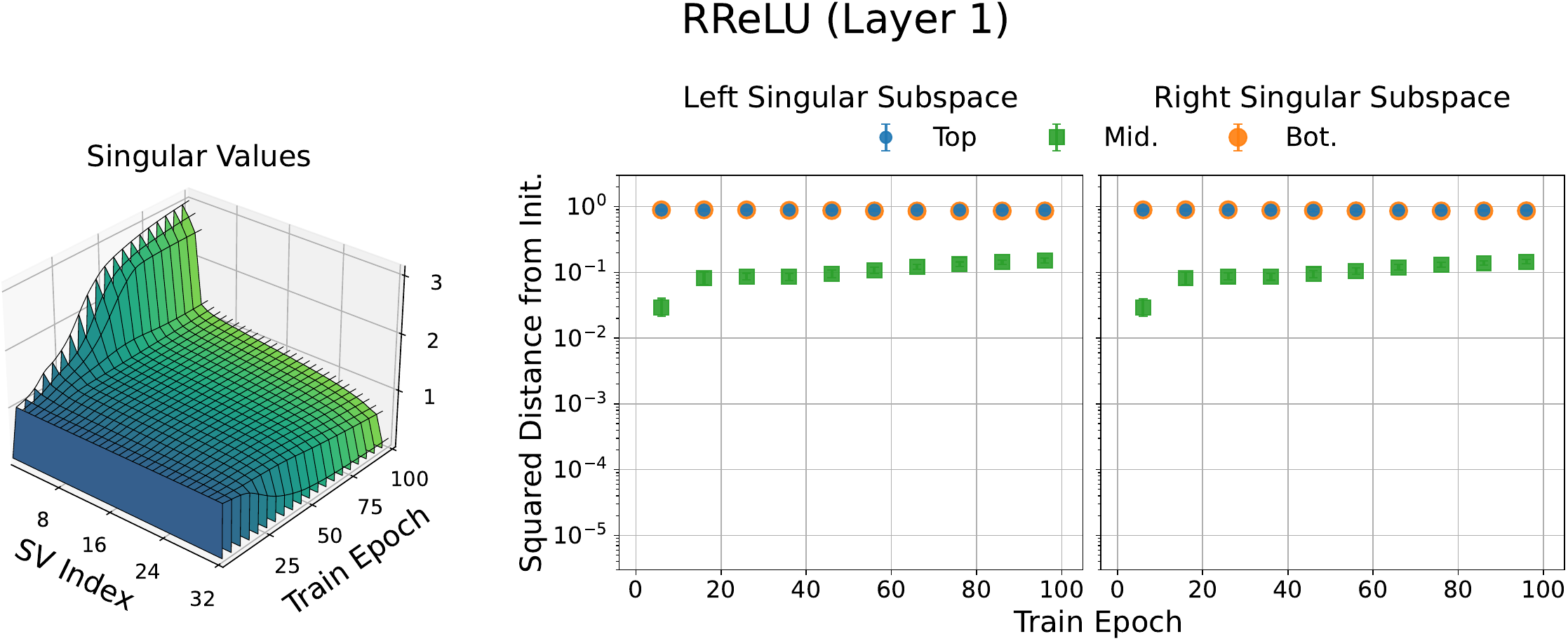}
    \caption{For networks with smooth activation functions, the middle singular subspace of the first layer weights evolves noticeably slower than in networks with non-smooth activation functions.}
    \label{fig:main_fig_more_results_layer1}
\end{figure}

\begin{figure}[h!]
    \centering
    \includegraphics[width=0.49\linewidth]{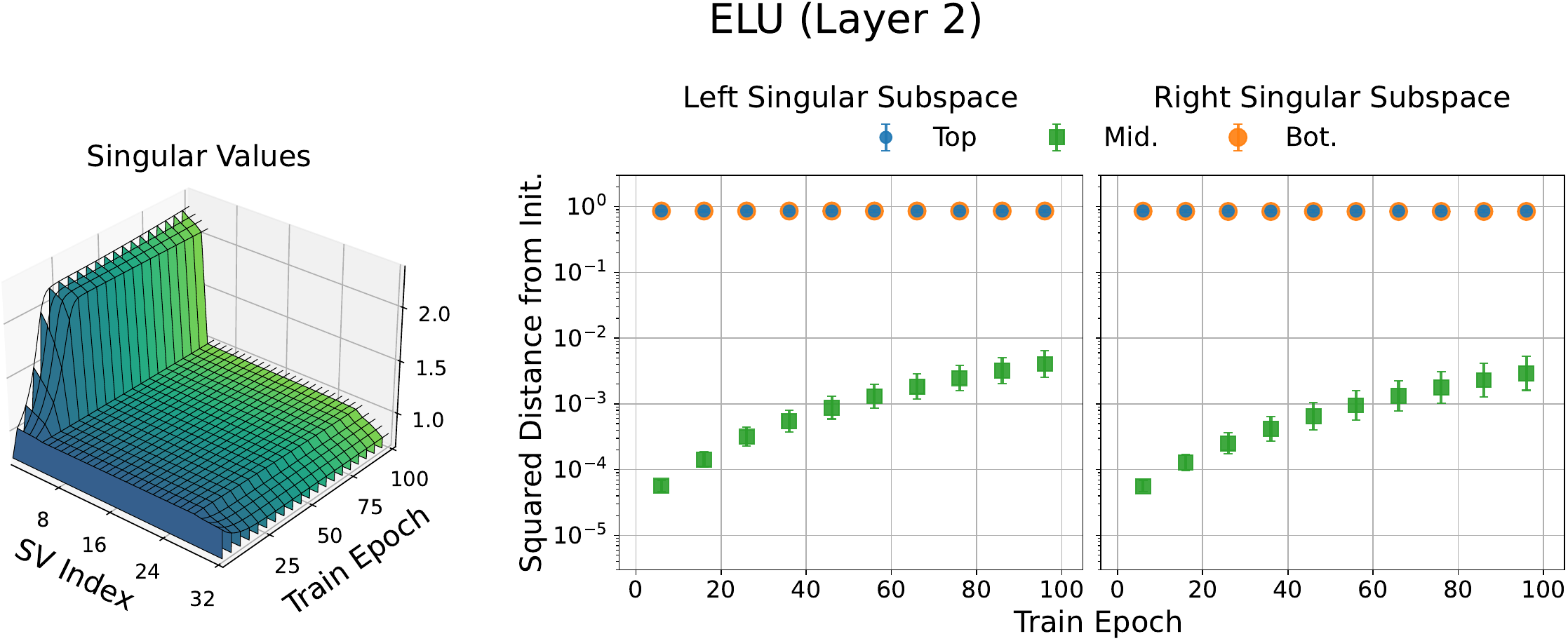}
    \includegraphics[width=0.49\linewidth]{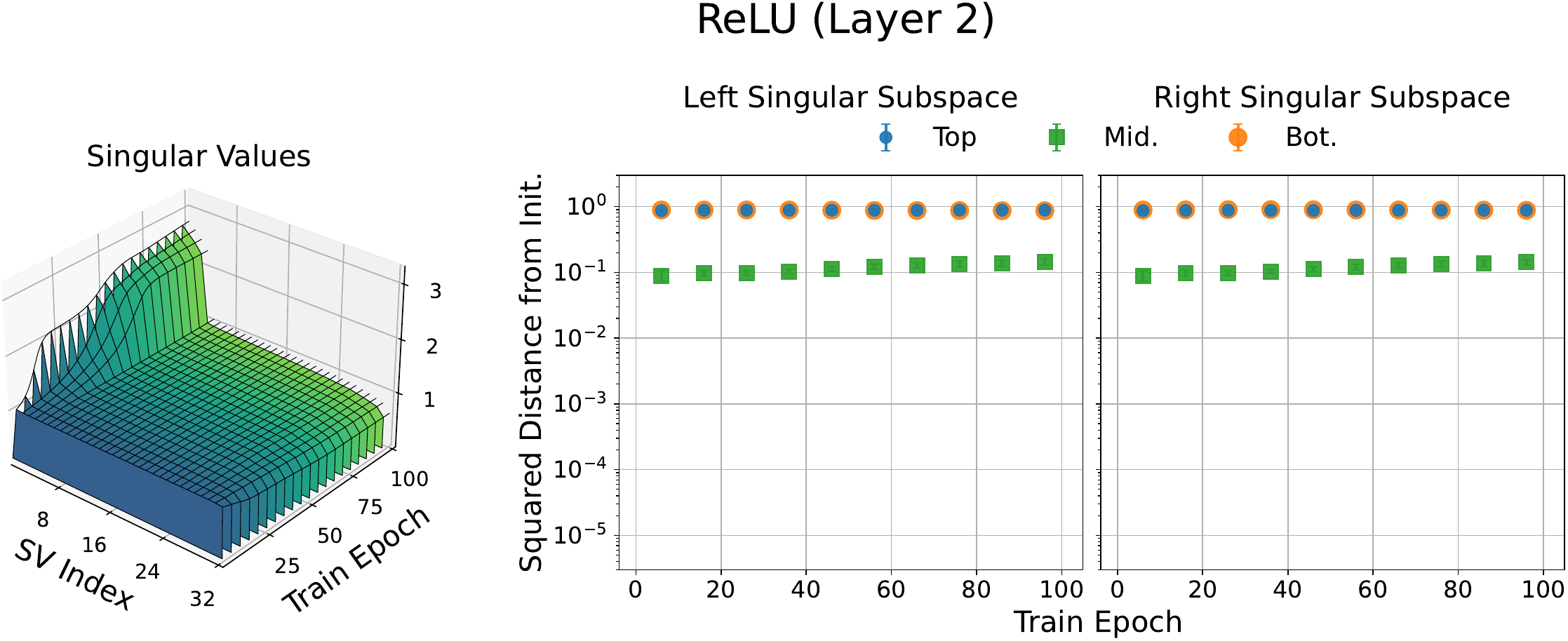}
    \includegraphics[width=0.49\linewidth]{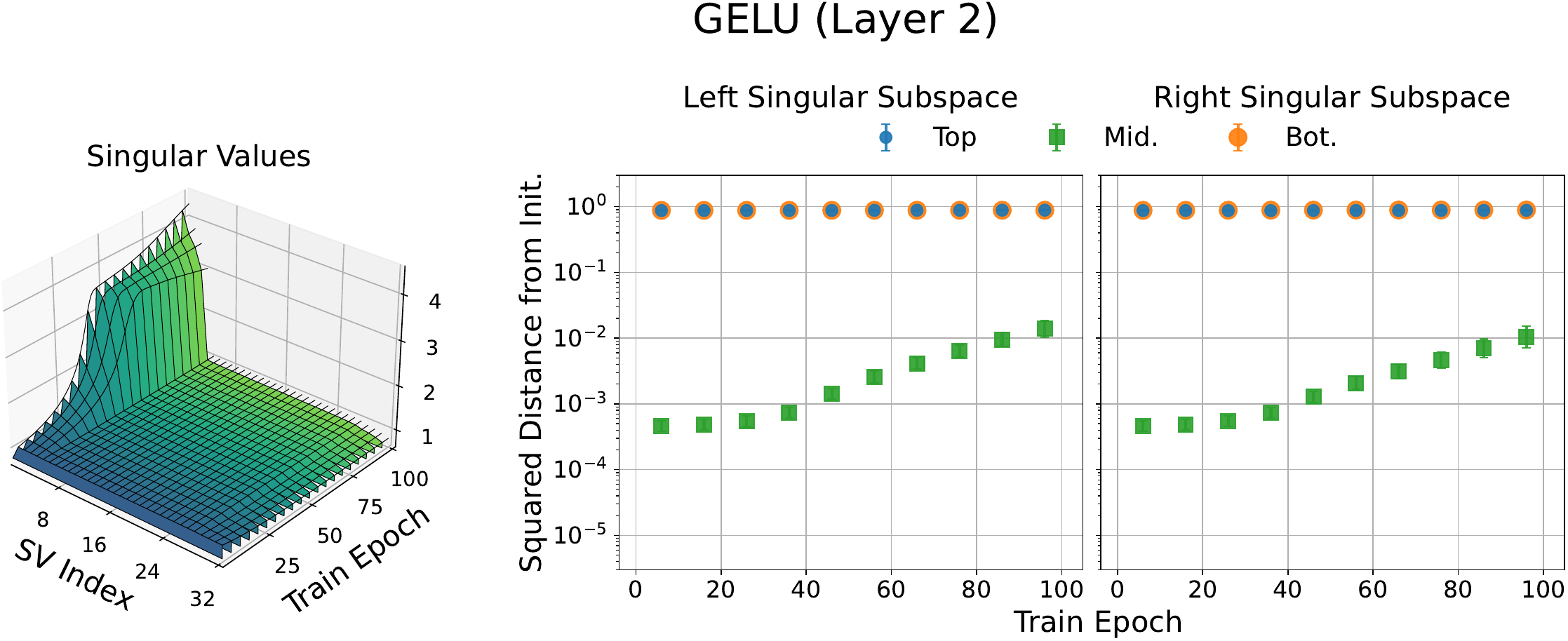}
    \includegraphics[width=0.49\linewidth]{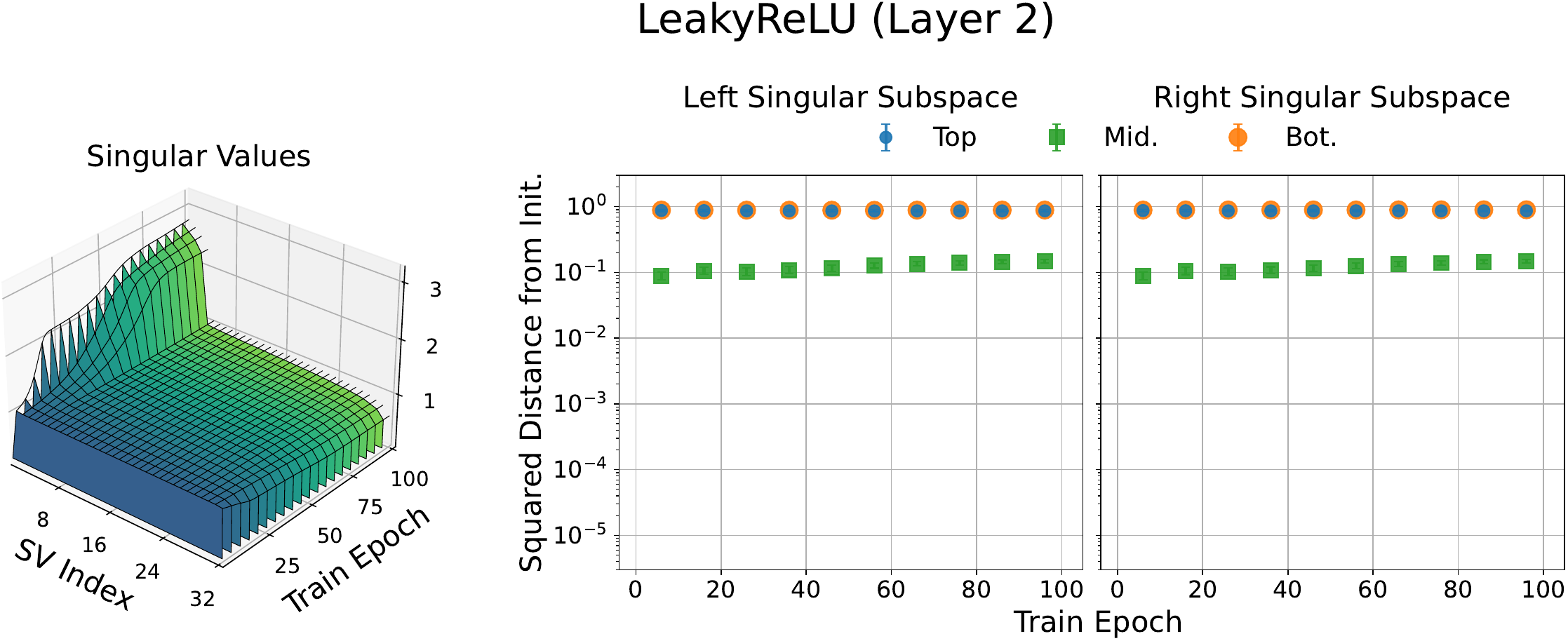}
    \includegraphics[width=0.49\linewidth]{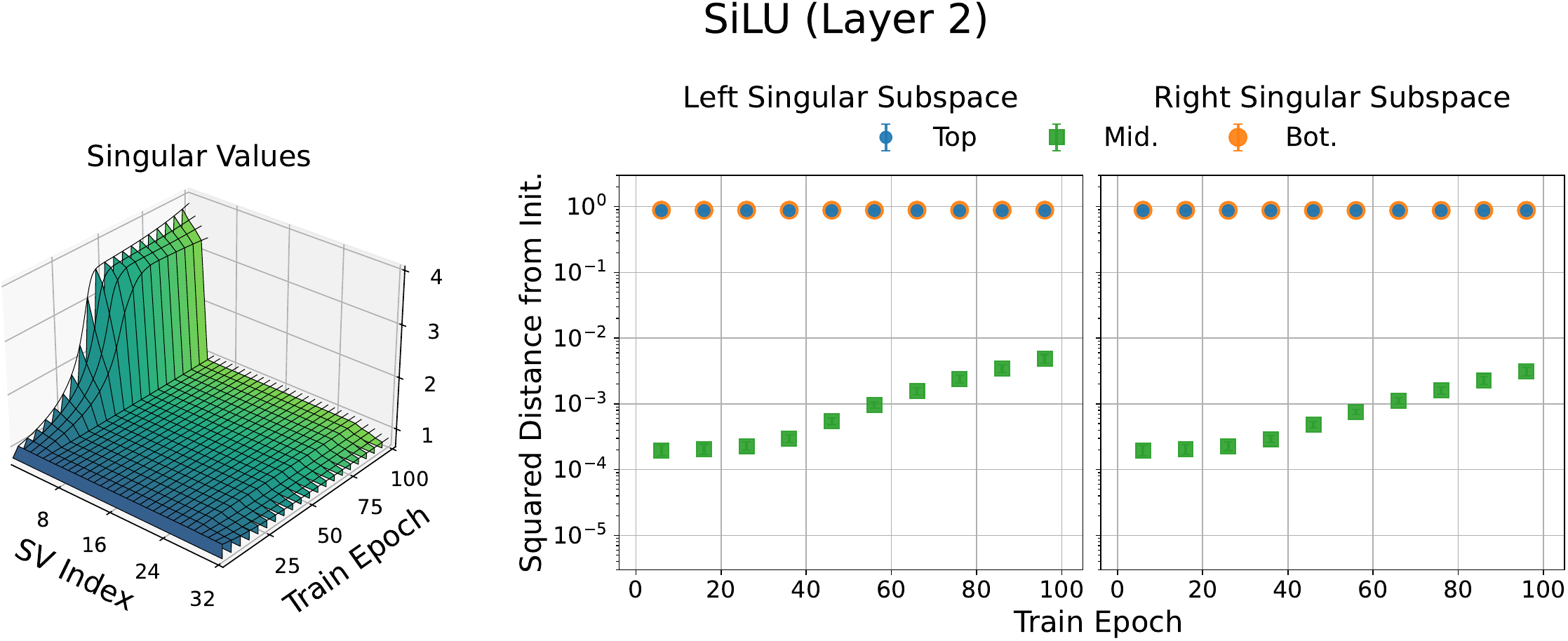}
    \includegraphics[width=0.49\linewidth]{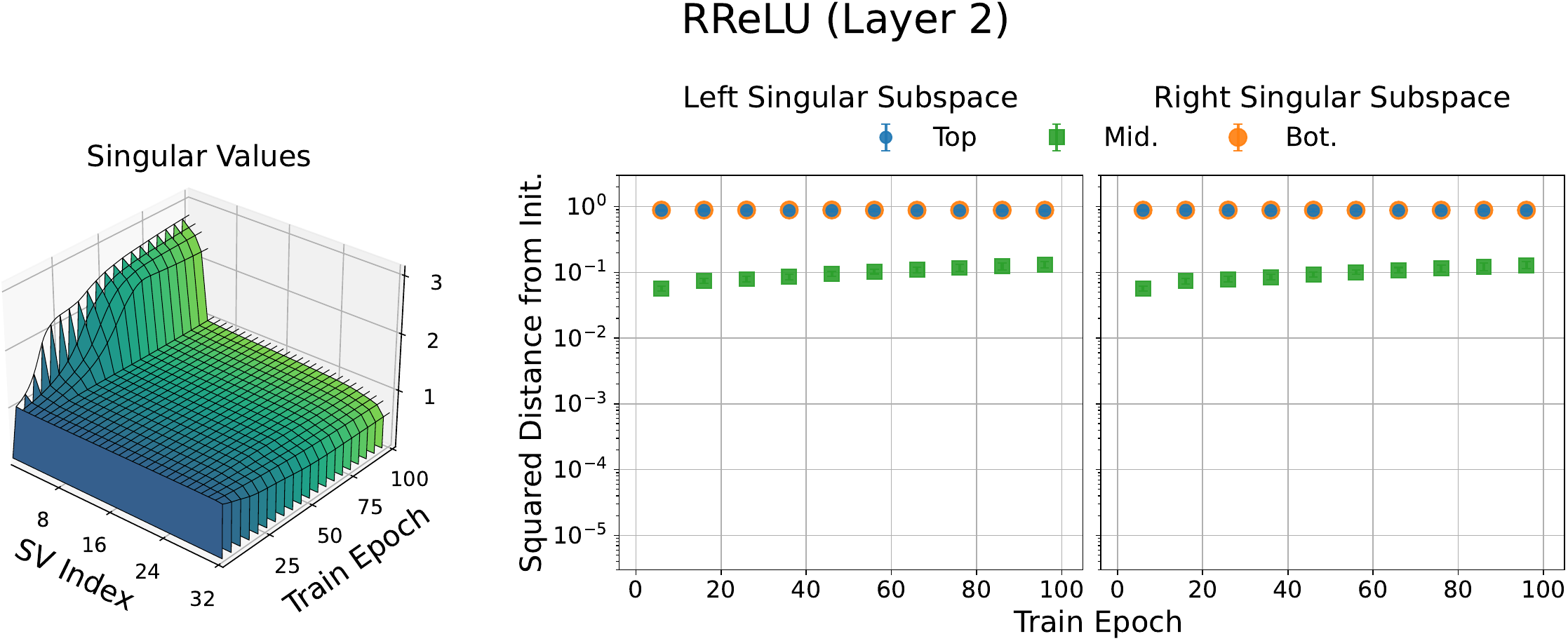}
    \caption{The change in the middle singular subspaces of the second layer weights in networks with smooth vs. nonsmooth activation functions mimic those in the first layer weights (\Cref{fig:main_fig_more_results_layer1}).}
    \label{fig:main_fig_more_results_layer2}
\end{figure}

\begin{figure}[h!]
    \centering
    \includegraphics[width=0.49\linewidth]{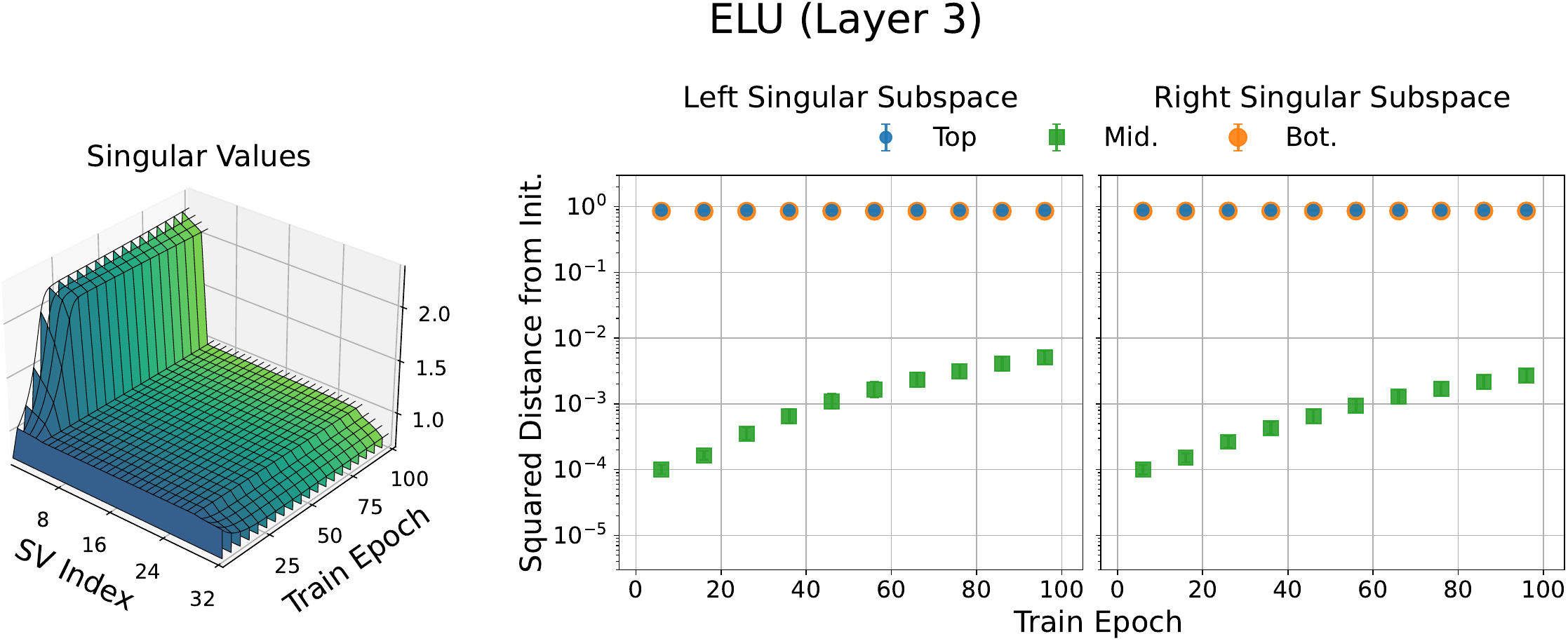}
    \includegraphics[width=0.49\linewidth]{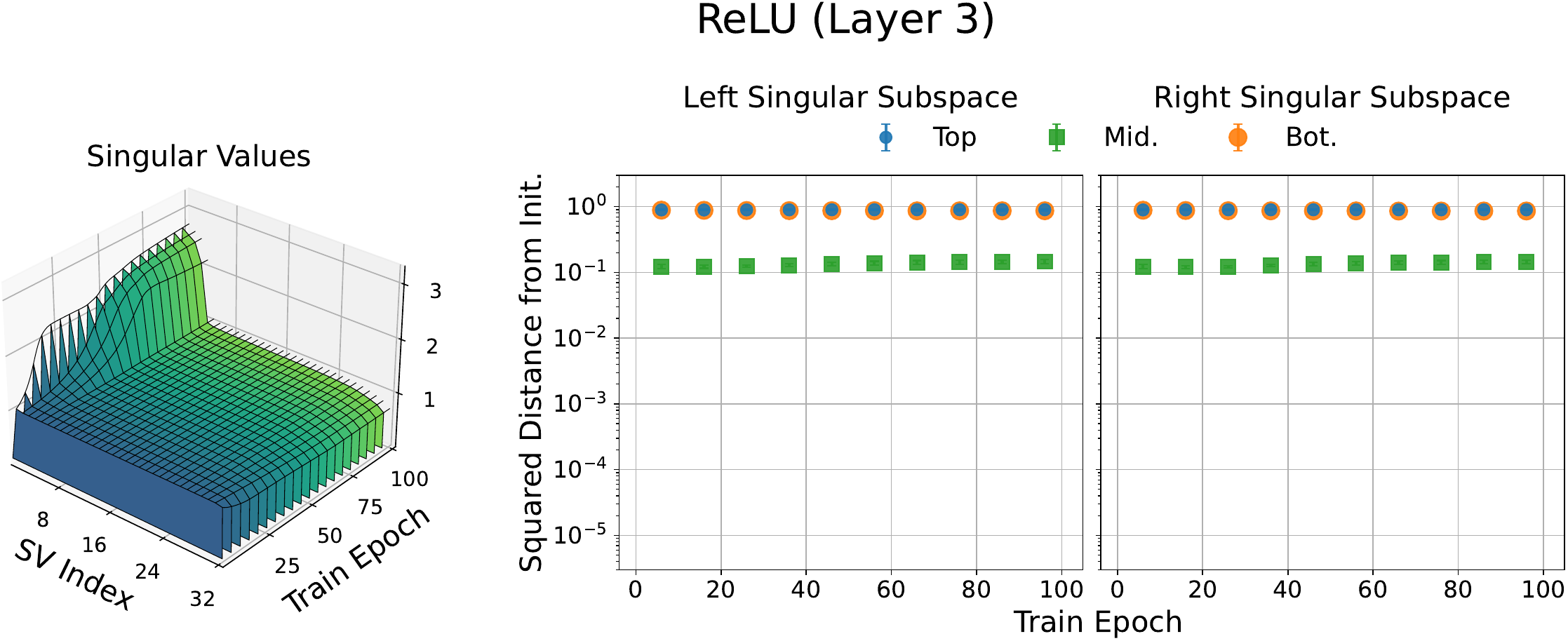}
    \includegraphics[width=0.49\linewidth]{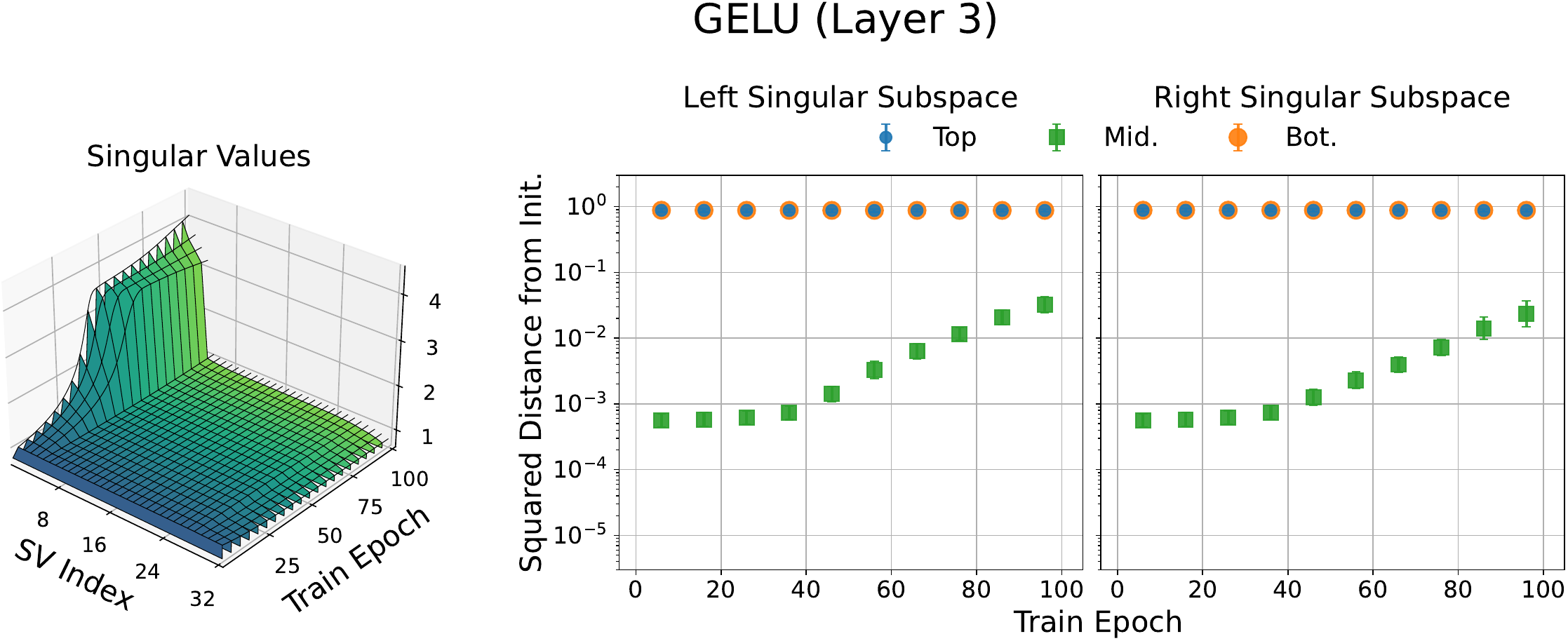}
    \includegraphics[width=0.49\linewidth]{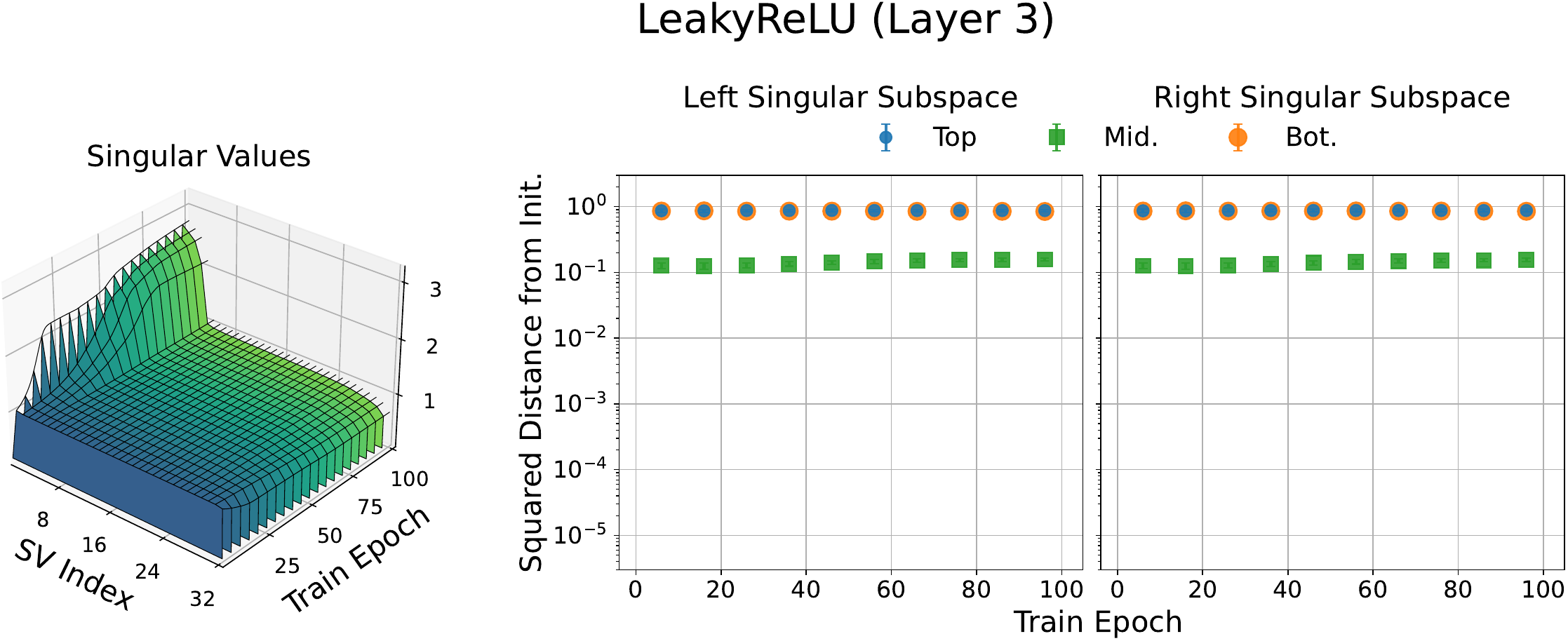}
    \includegraphics[width=0.49\linewidth]{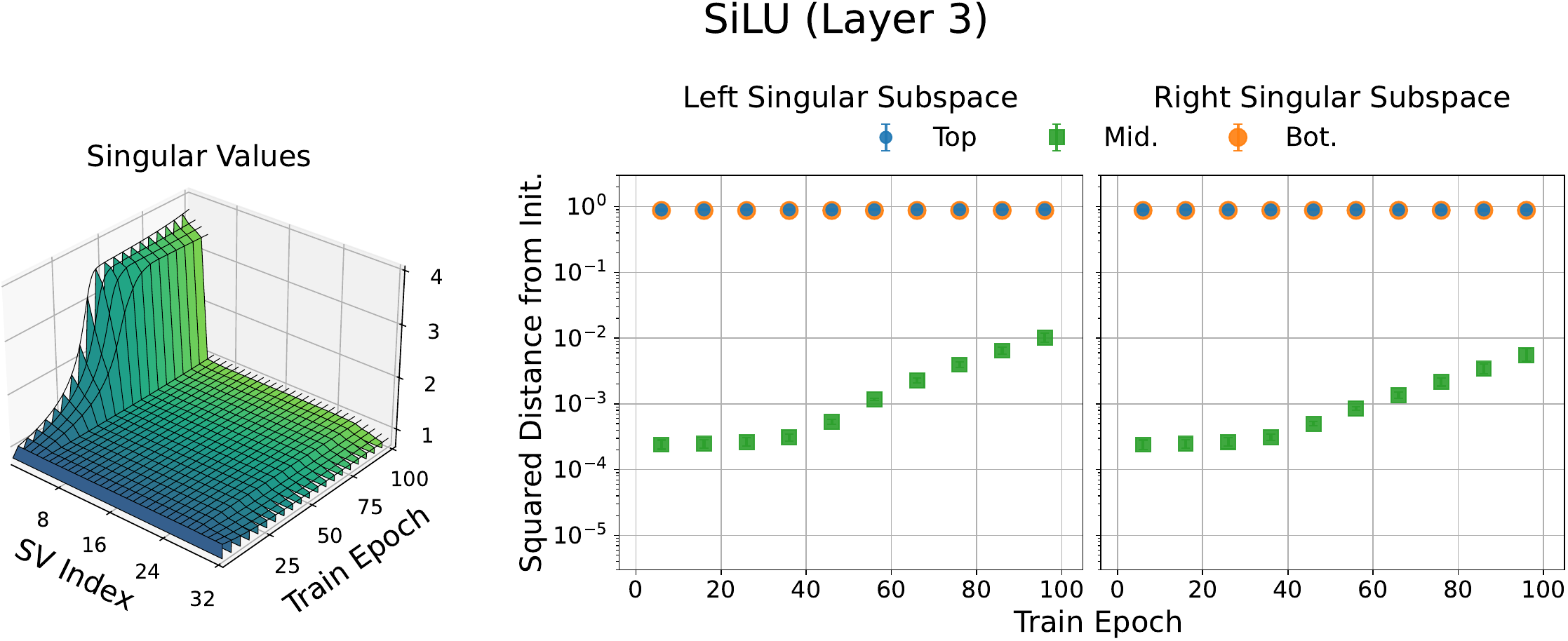}
    \includegraphics[width=0.49\linewidth]{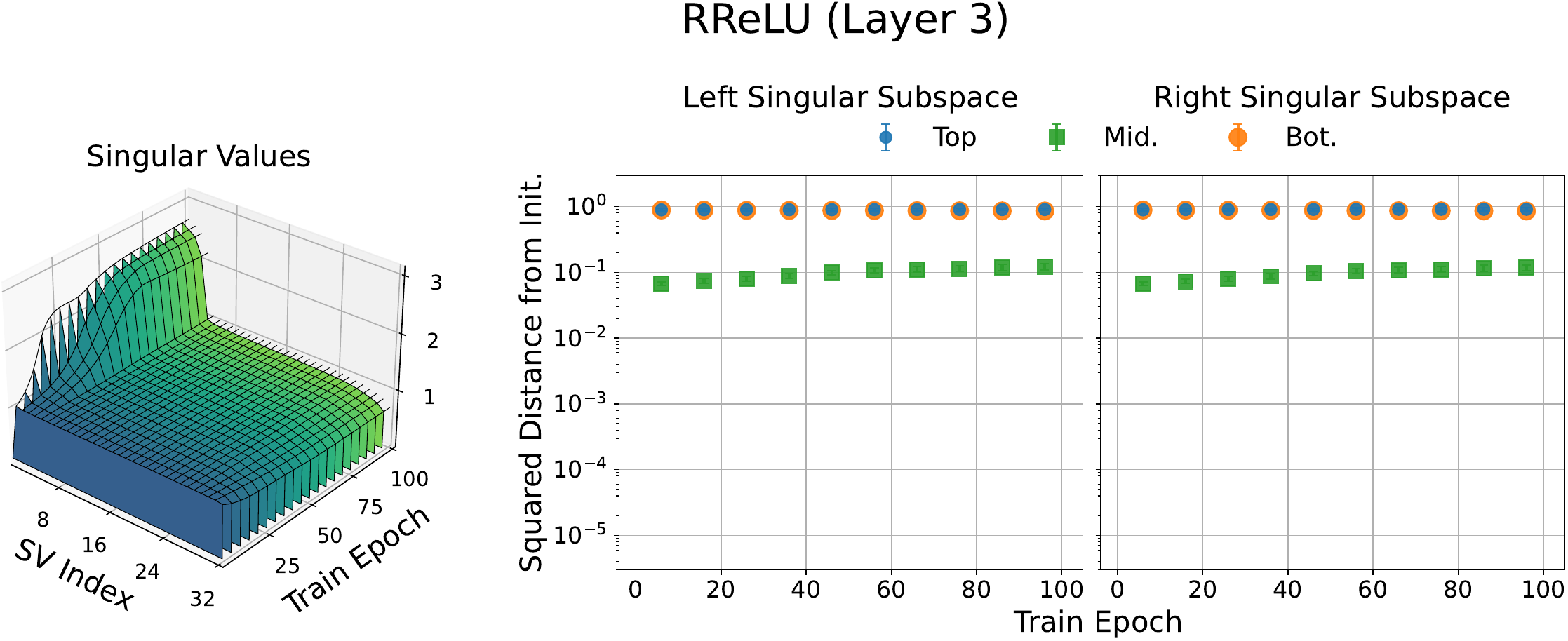}
    \caption{The change in the middle singular subspaces of the third layer weights in networks with smooth vs. nonsmooth activation functions mimic those in the first and second layer weights (\Cref{fig:main_fig_more_results_layer1,fig:main_fig_more_results_layer2}).}
    \label{fig:main_fig_more_results_layer3}
\end{figure}

\paragraph{Data generation.} To generate the data, we set $d = 32$, $K = 4$, and $N = 2000$, so $n = N / K = 500$. We generated $\bm X$ from a Gaussian mixture distribution as follows: first, we sampled $K$ means $\bm \mu_k \sim \mc{N}\left( \bm 0_d, \bm I_d \right)$. Next, for each $k \in [K]$, we generated $n = 500$ samples in the $k^{th}$ class via $\bm x_{k, i} \sim \mc{N}\left( \bm \mu_k, \sigma^2 \bm I_d \right)$ with $\sigma^2 = 3$. Then, we set $\bm X = \begin{bmatrix}
    \bm x_{1, 1} & \dots & \bm x_{1, n} & \dots & \bm x_{K, 1} & \dots & \bm x_{K, n}
\end{bmatrix}.$ Finally, we pre-processed $\bm X$ by whitening, e.g., $\bm X \bm X^\top = \bm I_d$. We generated the label matrix $\bm Y$ via $\bm Y = \bm I_K \otimes \bm 1_n$. 

\paragraph{Network architectures and training.} We considered six different $L = 4$ layer MLPs: three with smooth activation functions $\elu$, $\gelu$, $\silu$, and three with nonsmooth activations $\relu$, $\leakyrelu$, and a randomized $\leakyrelu$, called $\rrelu$. For $\elu$, we set the \texttt{PyTorch} parameter $\alpha = 1$, while for $\leakyrelu$, we set the \texttt{PyTorch} slope parameter $\alpha = 0.01$. In all networks, we set $m = d = 32$, and initialized all the weight matrices as $\epsilon$-scaled orthogonal matrices, with $\epsilon = 1$. 

Let $\bm W_l(t) = \begin{bmatrix}
    \bm A_{l, 1}(t) & \bm A_{l, 2}(t) & \bm A_{l, 3}(t)
\end{bmatrix} \begin{bmatrix}
    \bm S_{l, 1}(t) &  & \\
    & \bm S_{l, 2}(t) & \\
    & & \bm S_{l, 3}(t)
\end{bmatrix}  \begin{bmatrix}
    \bm B_{l, 1}(t) & \bm B_{l, 2}(t) & \bm B_{l, 3}(t)
\end{bmatrix}^\top$ be an SVD of $\bm W_l(t)$, where $\bm A_{l, 1}(t) \bm S_{l, 1}(t) \bm B_{l, 1}^\top(t)$, $\bm A_{l, 2}(t) \bm S_{l, 2}(t) \bm B_{l, 2}^\top(t)$, and $\bm A_{l, 3}(t) \bm S_{l, 3}(t) \bm B_{l, 3}^\top(t)$ respectively denote the top-$K$, middle $d - 2K$, and bottom-$K$ SVD components. We tracked the change in top and bottom-$K$ left singular subspaces via
\begin{equation*}
    \| \sin \Psi\left( \bm A_{l, 1}(t), \bm A_{l, 1}(0) \right) \|_2^2 \quad \text{and} \quad \| \sin \Psi\left( \bm A_{l, 3}(t), \bm A_{l, 3}(0) \right) \|_2^2, 
\end{equation*}
and similarly for the top and bottom-$K$ right singular subspaces, where $\| \sin \Psi\left( \bm U_1, \bm U_2 \right) \|_2$ is defined in \Cref{def:princ_angles} for $\bm U_1, \bm U_2$ with orthonormal columns. For the middle $d - 2K$ singular subspaces, we computed
\begin{equation*}
    \left\| \sin \Psi \left( \bm A_{l, 2}(t), \bm U_{l, 2} \right) \right\|_F^2 \quad \text{and} \quad \left\| \sin \Psi \left( \bm B_{l, 2}(t), \bm V_{l, 2} \right) \right\|_F^2,
\end{equation*}
where $\bm U_{l, 2}$ and $\bm V_{l, 2}$ are defined in \Cref{eq:V_l_U_l_nonrecursive}.

\paragraph{Additional results.} \Cref{fig:main_fig_more_results_layer1,fig:main_fig_more_results_layer2,fig:main_fig_more_results_layer3} shows the same results as in \Cref{fig:main_fig}, but with all nonlinear activations and layers. The behavior in the MLPs with smooth activations is similar to that of the $\elu$ network in \Cref{fig:main_fig}, while the MLPs with nonsmooth activations behave similarly to that of the $\relu$ network in \Cref{fig:main_fig}. This supports our conjecture that smooth activations encourage lower-rank training dynamics in MLPs.

\subsection{Experimental Details for \Cref{fig:smooth_theory_verify}}
\label[appendix]{ssec:additional_sims_theory_verify}
Here, we provide additional details on the experimental setup to generate \Cref{fig:smooth_theory_verify}. All experiments were run in \texttt{PyTorch} using an NVIDIA A40 GPU.

\paragraph{Data generation.}  We set $d = 64$, $N = 1000$, and $K = 4$. Under these parameters, we generated the data as described in \Cref{ssec:additional_sims_main_fig}. 

\paragraph{Network architecture and training.} We trained $\bm W_1$ only in a two-layer neural network $f_{\bm W_1}(\bm X) = \bm W_2 \phi(\bm W_1 \bm X)$ with $m = 72$ on squared error loss \eqref{eq:two_layer_squared_error_loss} using full-batch GD \eqref{eq:gd-update} with $\eta = 10^{-2}$. We considered $\phi = \elu, \gelu,$ and $\silu$, which are all smooth. We initialized $\bm W_1 \in \mbb{R}^{m \times d}$ to be an $\epsilon$-scaled semi-orthogonal matrix sampled uniformly at random, with $\epsilon = 10^{-2}$, and then set $\bm W_2$ to have frozen iid uniform entries between $-1$ and $1$. %As reference, here $\widetilde{\bm W}_{1, 1}(t) \in \mbb{R}^{16 \times 8}$, $\widetilde{\bm W}_{1, 2}(t) \in \mbb{R}^{16 \times 56}$. $\widetilde{\bm W}_{1, 3}(t) \in \mbb{R}^{56 \times 8}$, and $\widetilde{\bm W}_{1, 4}(t) \in \mbb{R}^{56 \times 56}$. 

\subsection{Computing Left and Right Weight Factors}
\label[appendix]{ssec:computing_S_small_details}
In this section, we provide additional details on how to compute the $\bm U \in \mathbb{R}^{m \times m}$ and $\bm V \in \mathbb{R}^{d \times d}$ factors from \Cref{thm:smooth_main_result_main_body}. Recall $p := d - 2K$, as well as $\mc{S}_{small}$ from \Cref{eq:smooth_small_update_subspace}:
\begin{equation*}
    \mc{S}_{small} := \mc{R}\left( \bm R_{1, 2}(0) \right) \cap \mc{R}^\perp\left( \bm W_1^\top(0) \bm L_{1, 1}(0) / \epsilon \right).
\end{equation*}
Assuming $\bm W_1(0)$ is initialized as an $\epsilon$-scaled semi-orthogonal matrix, $\bm U$ and $\bm V$ factors in \Cref{thm:smooth_main_result_main_body} can be computed as follows: 
\begin{enumerate}
    \item Compute an orthonormal basis of $\mc{S}_{small}$, denoted as $\bm V_2 \in \mathbb{R}^{d \times p}$, via \Cref{algo:S_small}.
    \item Take $\bm Q$ factor of $\operatorname{QR}\left( \bm V_2 \right)$ to obtain $\bm V_1$, and define $\bm V := \begin{bmatrix}
        \bm V_1 & \bm V_2
    \end{bmatrix}$
    \item Set $\bm U = \bm W_1(0) \bm V / \epsilon$, or equivalently $\begin{bmatrix}
        \bm U_1 & \bm U_2
    \end{bmatrix} = \begin{bmatrix}
        \bm W_1(0) \bm V_1 / \epsilon & \bm W_1(0) \bm V_2 / \epsilon
    \end{bmatrix}$.
\end{enumerate}
Overall, this procedure requires $1$ SVD, and at most $3$ QR decompositions. We emphasize these computations \textbf{only need to be done once:} after a single forward-backward pass of the original model. 

\begin{algorithm}
\caption{Computing orthonormal basis of $\mc{S}_{small}$}
\label{algo:S_small}
\begin{enumerate}
    \item $\bm L_1(0), \bm R_1(0) \leftarrow \operatorname{SVD}\left(\bm G_1(0) \right)$ 
    \item By definition of $\operatorname{SVD}$, take $\bm R_{1, 2}(0)$ as orthonormal basis for $\mc{R}\left( \bm R_{1, 2}(0) \right)$.
    \item Do $\bm W_1^\top(0) \bm L_{1, 2}(0) / \epsilon$ if $\bm W_1(0)$ is square, otherwise take last $d - K$ columns of $\bm Q$ factor of $\operatorname{QR}\left( \bm W_1^\top(0) \bm L_{1, 1}(0) \right)$, to obtain orthonormal basis for $\mc{R}^\perp\left( \bm W_1^\top(0) \bm L_{1, 1}(0) / \epsilon \right)$. Denote this basis as $\widetilde{\bm L}_{1, 1}(0)$.
    \item Take $\bm Q$ factor of $\operatorname{QR}\left( \begin{bmatrix}
        \bm R_{1, 2}(0) &
        \widetilde{\bm L}_{1, 1}(0)
    \end{bmatrix} \right)$ to obtain orthonormal basis for $\mc{S}_{small}$.
\end{enumerate}
\end{algorithm}

\subsection{Experimental Details for \Cref{sec:beyond_theory}}
\label[appendix]{ssec:additional_sims_beyond_theory}
In this section, we provide additional experimental details for \Cref{sec:beyond_theory}. 

\paragraph{Data generation.} In \Cref{ssec:beyond_theory_deep_nets_and_activations,ssec:beyond_theory_sgd} (\Cref{fig:beyond_theory_deep_nets_and_activations,fig:beyond_theory_optimizer_loss_unwhitened} respectively), we adhered to the exact same data generation process as in \Cref{ssec:additional_sims_main_fig}, but we skipped the whitening pre-processing step on $\bm X$ in \Cref{ssec:beyond_theory_sgd} (\Cref{fig:beyond_theory_optimizer_loss_unwhitened}).

\paragraph{Network architecture and training.} We considered $L = 4$ layer networks of width $m = 72$ with activations $\phi = \elu$ and $\gelu$. We initialized the first $3$ layers to be $\epsilon$-scaled (semi-)orthogonal matrices with $\epsilon = 0.1$, and the last layer $\bm W_L$ with iid uniform entries between $-1$ and $1$. 

In \Cref{ssec:beyond_theory_deep_nets_and_activations} (\Cref{fig:beyond_theory_deep_nets_and_activations}), we trained the network on squared-error loss using full-batch GD for $250$ epochs. For the $\elu$ network, we set $\eta = 10^{-3}$, while for the $\gelu$ network, we set $\eta = 5 \times 10^{-3}$. Meanwhile, in \Cref{ssec:beyond_theory_sgd} (\Cref{fig:beyond_theory_optimizer_loss_unwhitened}), we trained the networks using 1) SGD with momentum, 2) Adam, and 3) Muon. For SGD and Muon, we set the momentum to $0.9$, while for Adam, we used the default \texttt{PyTorch} parameters. For all optimizers, we set the batch size to $32$, $\eta = 10^{-4}$ for the $\elu$ network, and $\eta = 5 \times 10^{-4}$ for $\gelu$.

    \section{Empirical Justifications for Assumptions}
\label[appendix]{sec:smooth_assum_justify}

In this section, we provide some empirical justification for fixed $\bm W_2$ in \Cref{assum:network}, as well as  the gradient norm and singular values in \Cref{assum:gradient}. In all experiments, we adopt the exact same data generation process, network architecture, and training method as in \Cref{fig:smooth_theory_verify}.

\subsection{Fixed vs. Trained Second Layer}
\label[appendix]{ssec:smooth_assum_justify_fixed_W2}
Here, we show training $\bm W_2$ within our theoretical setting still leads to low-rank training dynamics in $\bm W_1$, implying a frozen $\bm W_2$ is not strictly necessary for low-rank training to arise. Specifically, \Cref{fig:fixed_vs_trained_W2} (top) shows training $\bm W_2$ leads to nearly identical results as \Cref{fig:smooth_theory_verify} (also included in \Cref{fig:fixed_vs_trained_W2} (bottom) as reference), where $\bm W_2$ is frozen. In other words, even when $\bm W_2$ is trained, a large majority of the change in $\bm W_1(t)$ is due to $\widetilde{\bm W}_{1, 1}(t)$. This implies the overall low-rank training phenomenon is not strictly because $\bm W_2$ is fixed in our theory. Rather, fixing $\bm W_2$ leads to an analytically tractable setting to provide some theoretical characterization of low-rank training dynamics in nonlinear networks. 

\begin{figure}[t]
    \centering
    \includegraphics[width=0.8\linewidth]{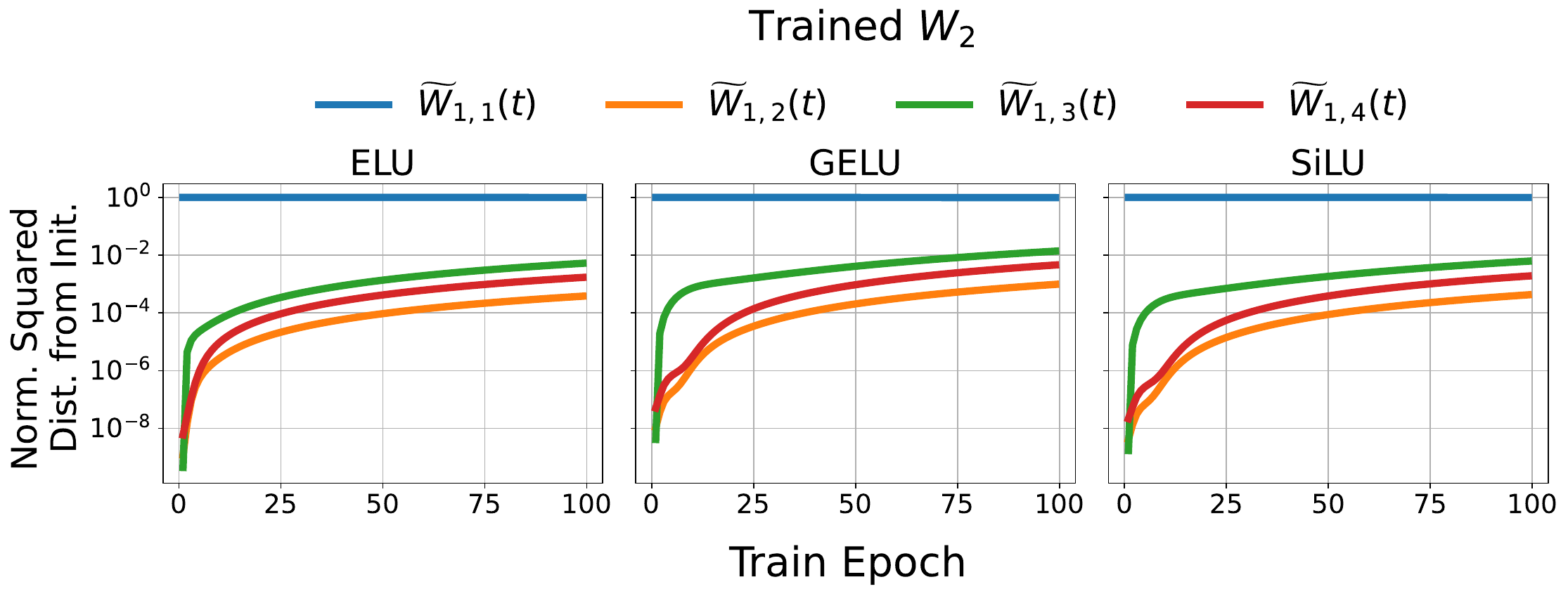}
    \includegraphics[width=0.8\linewidth]{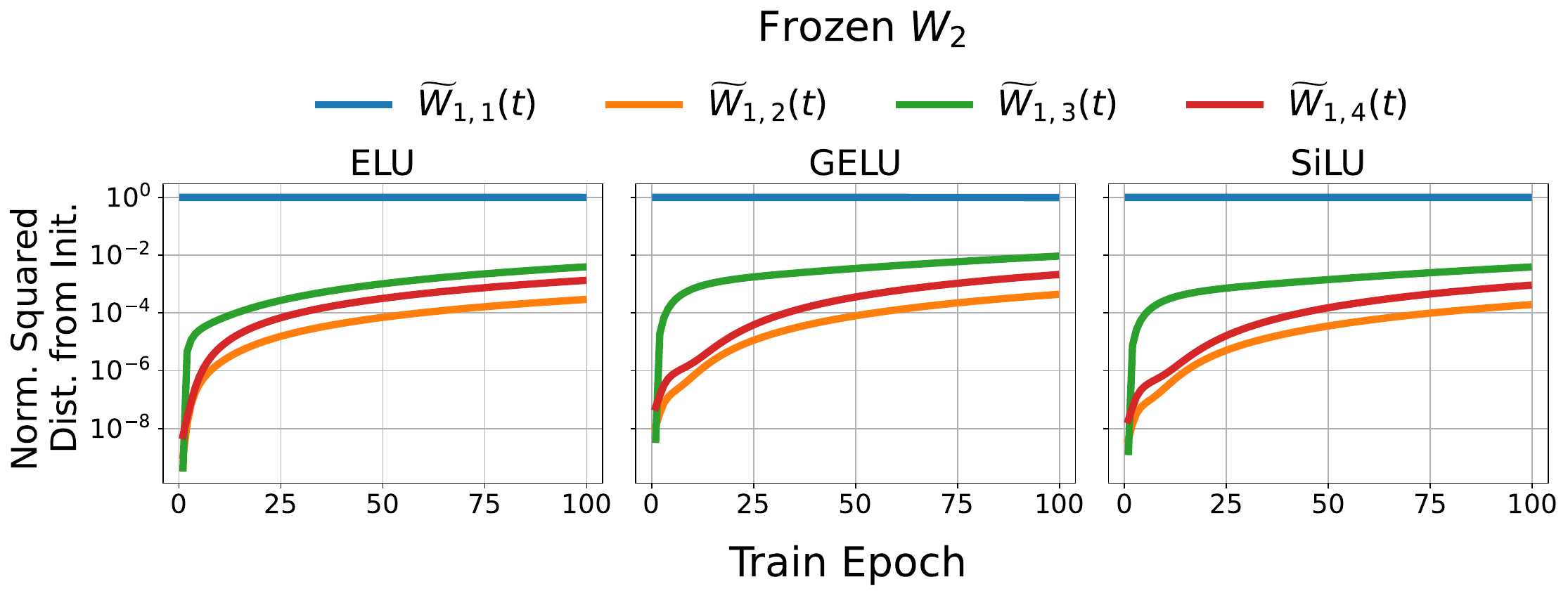}
    \caption{\textbf{Top:} even when $\bm W_2$ is trained, nearly all of the change in $\widetilde{\bm W}_1(t)$ is from $\widetilde{\bm W}_{1, 1}(t)$. \textbf{Bottom:} repeat of \Cref{fig:smooth_theory_verify} as reference.}
    \label{fig:fixed_vs_trained_W2}
\end{figure}

\subsection{Gradient Norm and Singular Value Decay}
\label[appendix]{ssec:smooth_assum_justify_gradients}
In this section, we provide empirical justifications for \Cref{assum:gradient}. We trained all models for $250$ epochs instead of $100$.

\paragraph{Gradient norm decay.} In \Cref{fig:smooth_assum_grad_top_sval_decay}, we show $\| \bm G_1(t) \|_F$ decays monotonically with respect to $\| \bm G_1(0) \|_F$, and that the top-$K$ singular values of $\bm G_1(t)$ decay at a similar rate. Specifically, we plot $\frac{\| \bm G_1(t) \|_F}{\| \bm G_1(0) \|_F}$, $\max\limits_{i \in [K]} \frac{\sigma_i\left( \bm G_1(t) \right)}{\sigma_i\left( \bm G_1(0) \right)}$, and $G_2 \cdot \frac{\| \bm G_1(t) \|_F}{\| \bm G_1(0) \|_F}$, where $G_2 = 2$. The gradient norm clearly monotonically decays throughout training, despite the objective being non-convex. This supports the assumption that $\| \bm G_1(t) \|_F \leq G_1 \cdot \left(1 - \Theta(\eta) \right)^{\Theta(t)} \cdot \| \bm G_1(0) \|_F$. Additionally, all of the top-$K$ singular values of $\bm G_1(t)$ appear to decay at a similar rate, as \Cref{fig:smooth_assum_grad_top_sval_decay} shows $\frac{\sigma_i\left( \bm G_1(t) \right)}{\sigma_i\left( \bm G_1(0) \right)} \leq G_2 \cdot \frac{\| \bm G_1(t) \|_F}{\| \bm G_1(0)  \|_F}$ for $i \in [K]$, which supports that assumption as well. 

\begin{figure}[t]
    \centering
    \includegraphics[width=1.0\textwidth]{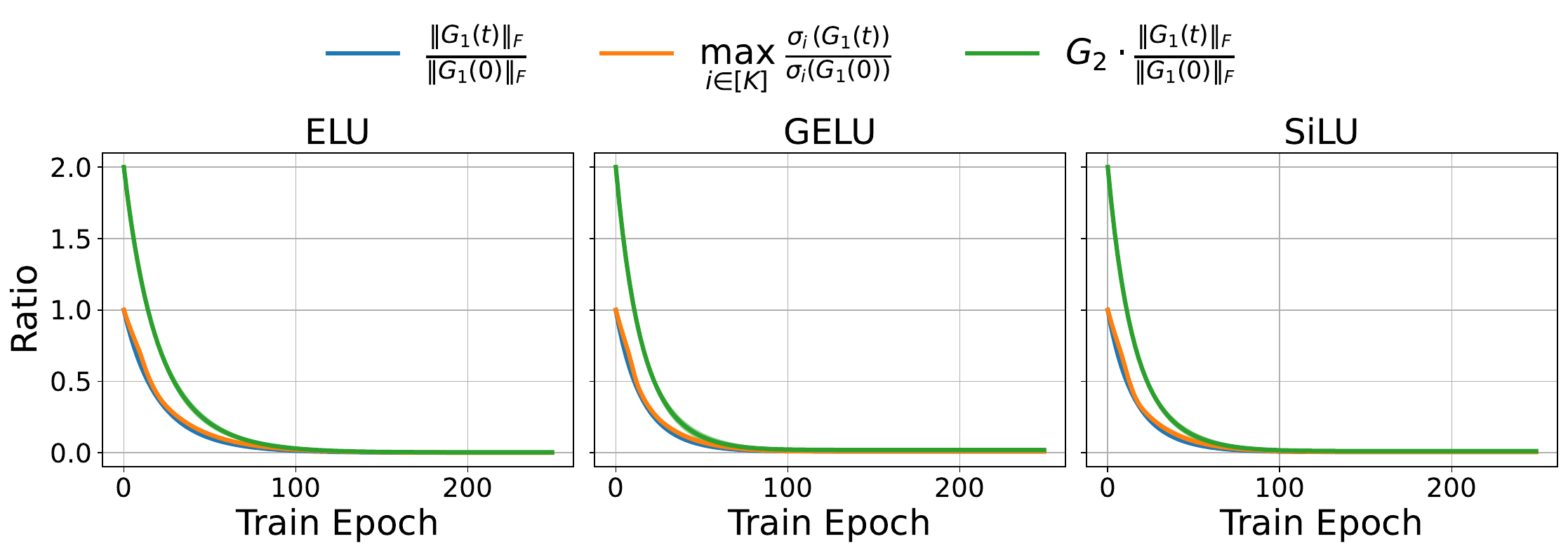}
    \caption{The gradient norm and top-$K$ singular values appear to monotonically decay from their initial values, despite the objective being non-convex.}
    \label{fig:smooth_assum_grad_top_sval_decay}
\end{figure}

\paragraph{Gradient singular values.} Finally, we show the tail singular values of $\bm G_1(t)$ remain much smaller than $\sigma_K\left( \bm W_2^\top \bm Y \bm X^\top \right)$. \Cref{fig:smooth_assum_tail_sval} plots $\sigma_K\left( \bm W_2^\top \bm Y \bm X^\top \right)$ vs. $\sigma_{K + 1}\left( \bm G_1(t) \right) $ for different $\epsilon$, all averaged over $10$ trials. Clearly, $\sigma_K\left( \bm W_2^\top \bm Y \bm X^\top \right)$ is much larger than $\sigma_{K + 1}\left( \bm G_1(t) \right)$, justifying our assumption that $\sigma_K\left( \bm W_2^\top \bm Y \bm X^\top \right) - \sigma_{K + 1}\left( \bm G_1(t) \right) \geq G_3 \cdot \sigma_K\left( \bm W_2^\top \bm Y \bm X^\top \right)$. For all $3$ activations, $\sigma_{K + 1}\left( \bm G_1(t) \right)$ remains very small for all GD iterations. %We conjecture this is because under our setting, the network is approximately linear, i.e., $\bm W_2 \phi\left( \bm W_1(t) \bm X \right) \approx \phi'(0) \cdot \bm W_2 \bm W_1(t) \bm X$, and so $\bm G_1(t) \approx \left( \nabla_{\bm W_1}  \frac{1}{2} \cdot \left\| \phi'(0) \cdot \bm W_2 \bm W_1(t) \bm X - \bm Y \right \|_F^2 \right) + \bm E$ for some ``small'' perturbation term $\bm E$.

\begin{figure}[t]
    \centering
    \includegraphics[width=\textwidth]{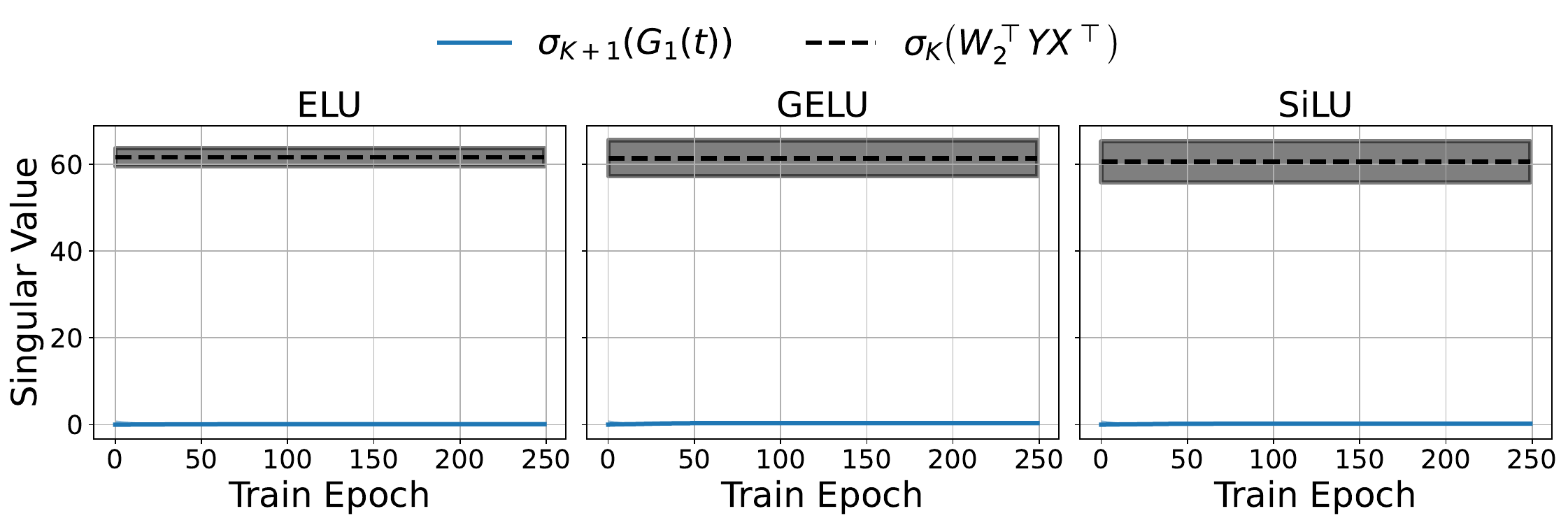}
    \caption{Throughout training, $\sigma_{K + 1}\left( \bm G_1(t) \right)$ remains significantly smaller than $\sigma_K\left( \bm W_2^\top \bm Y \bm X^\top \right)$.}
    \label{fig:smooth_assum_tail_sval}
\end{figure}
    \section{Why Activation Smoothness Matters: Some Intuition}
\label[appendix]{sec:nonsmooth_intuition}
From \Cref{fig:main_fig}, we observed that compared to MLPs with smooth activation functions, the training dynamics in MLPs with non-smooth activations are noticeably less prominent. In this section, we provide some brief intuition for why this is the case. In particular, recall a step in our proof sketch of \Cref{thm:smooth_main_result_main_body}: when we assume $\phi$ is smooth, then $\sigma_i\left( \bm G_1(0) \right) = \Theta(\epsilon)$ for all $i \geq K + 1$. Our next result shows that in a simplified setting, this step does not hold for a nonsmooth $\phi$. Specifically, we show (most of) these tail singular values of $\bm G_1(0)$ are \emph{not} on the same order as $\epsilon$. Although the theoretical setting is highly simplified, we empirically observe similar conclusions hold in broader settings; see \Cref{ssec:nonsmooth_theory_verification}. The proof of \Cref{prop:nonsmooth_result} is provided in \Cref{sec:nonsmooth_proofs}.

\begin{proposition}
\label{prop:nonsmooth_result}
    Suppose $d = N$, $\bm W_1(0)_{ij} \overset{iid}{\sim} \mc{N}(0, \frac{\epsilon^2}{m})$, and $\phi = \relu$. Then, for any $\delta \in (0, 1)$, with probability at least $1 - \delta$ w.r.t. the randomness in $\bm W_1(0)$, 
    \begin{equation*}
        \sigma_{d - K}\left( \bm G_1(0) \right) \geq  \sqrt{\frac{\lambda_{\min} \left( \bm V^\top \bm D \bm V \right)}{4}  - \left( \frac{R'}{6} \cdot \log\left( \frac{2(d - K)}{\delta} \right) + \sqrt{2 \cdot \log\left( \frac{2(d - K)}{\delta} \right) \cdot \frac{R' \cdot \lambda_{\max}\left( \bm V^\top \bm D \bm V \right) }{16}} \right)},
    \end{equation*}
    where $\bm V \in \mbb{R}^{d \times (d - K)}$ is an orthonormal basis for $\mc{N}\left( \bm W_2^\top \bm \Delta_2(0) \right)$, $\bm D := \mathrm{diag}\left( \left\| \left( \bm W_2^\top \bm \Delta_2(0) \right)_{:, 1} \right\|_2^2, \dots, \left\| \left( \bm W_2^\top \bm \Delta_2(0) \right)_{:, N} \right\|_2^2  \right) $, and $R' := \max\limits_{i \in [m]} \| \left( \bm W_2^\top \bm \Delta_2(0) \right)_{i, :} \|_2^2$. 
\end{proposition}
\paragraph{Intuition behind \Cref{prop:nonsmooth_result}.} Below, we provide some intuition behind \Cref{prop:nonsmooth_result}. First, note for small $\epsilon$, we have $\bm \Delta_2(0) \approx -\bm Y$, and so $\bm W_2^\top \bm \Delta_2(0) \approx -\bm W_2^\top \bm Y \in \mbb{R}^{m \times N}$. If we consider a classification task with balanced classes and one-hot encoded labels, we have $\bm Y = \bm I_K \otimes \bm 1_n^\top$, where $n = N / K$. Thus, 
\begin{align*}
    \bm W_2^\top \bm Y = \begin{bmatrix}
        \left(\bm W_2\right)_{1, 1} \bm 1_n^\top & \dots & \left( \bm W_2 \right)_{K, 1} \bm 1_n^\top \\
        \ & \ddots & \ \\
        \left( \bm W_2 \right)_{1, m} \bm 1_n^\top & \dots & \left( \bm W_2 \right)_{K, m} \bm 1_n^\top
    \end{bmatrix},
\end{align*}
i.e., the columns of $\bm W_2^\top \bm Y$ are (repeated) rows of $\bm W_2$ which are of size $m$, and each row of $\bm W_2^\top \bm Y$ contains $n = N / K = d / K$ copies of the corresponding column entries of $\bm W_2$, which are each of size $K$. Next, if $\bm W_2$ contains, e.g., iid sub-Gaussian entries with zero mean and unit variance, then $\left\| - \left(\bm W_2^\top \bm Y\right)_{i, :} \right\|_2^2 \approx d$ and $\left\| - \left(\bm W_2^\top \bm Y\right)_{:, j} \right\|_2^2 \approx m$, and so $\bm D \approx m \bm I_N \implies \bm V^\top \bm D \bm V \approx m \bm I_{d - K} \implies \lambda_{\min}\left(\bm V^\top \bm D \bm V\right) \approx \lambda_{\max}\left(\bm V^\top \bm D \bm V \right) \approx m$, and $R' \approx d$. For a fixed failure probability $\delta \in (0, 1)$, we have
\begin{align*}
    \sigma_{d - K}\left( \bm G_1(0) \right) \gtrsim \sqrt{\Theta(m) - \Theta(d \log d) -\Theta\left(\sqrt{md\log d}\right) }.
\end{align*}
Then, if $m \gtrsim \Omega\left(d \log d \right)$, this lower bound is non-vacuous and not of the same order as the initialization scale $\epsilon$. 

\paragraph{New assumptions on $\bm X$ and $\bm W_1(0)$.} In \Cref{prop:nonsmooth_result}, we assume $d = N$, i.e., $\bm X$ is an exactly orthogonal matrix, and $\left( \bm W_1(0) \right)_{ij} \sim \mc{N}\left( 0, \frac{\epsilon^2}{m} \right)$. This is mostly for ease of analysis, specifically to introduce iid random variables in the mask $\phi'\left( \bm W_1(0) \bm X \right)$. Although we recognize $d = N$ is quite restrictive, we empirically show the broader conclusions hold in more general settings, i.e., $d < N$. %we note previous neural network analyses have made similar assumptions, e.g., exactly orthonormal data in \cite{boursier2022gradient}, and nearly-orthogonal data in \cite{frei2023implicit,kou2023implicit,wang2023understanding}.
Next, since independent Gaussian vectors are approximately orthogonal in high dimensions, $\bm W_1(0)$ in \Cref{prop:nonsmooth_result} satisfies $\bm W_1^\top(0) \bm W_1(0) \approx \epsilon^2 \bm I_d$, which is the assumption we make on $\bm W_1(0)$ in \Cref{thm:smooth_main_result_main_body}.

\subsection{Verifying \Cref{prop:nonsmooth_result}}
\label{ssec:nonsmooth_theory_verification}

\begin{figure}[t]
    \centering
    \includegraphics[width=0.8\linewidth]{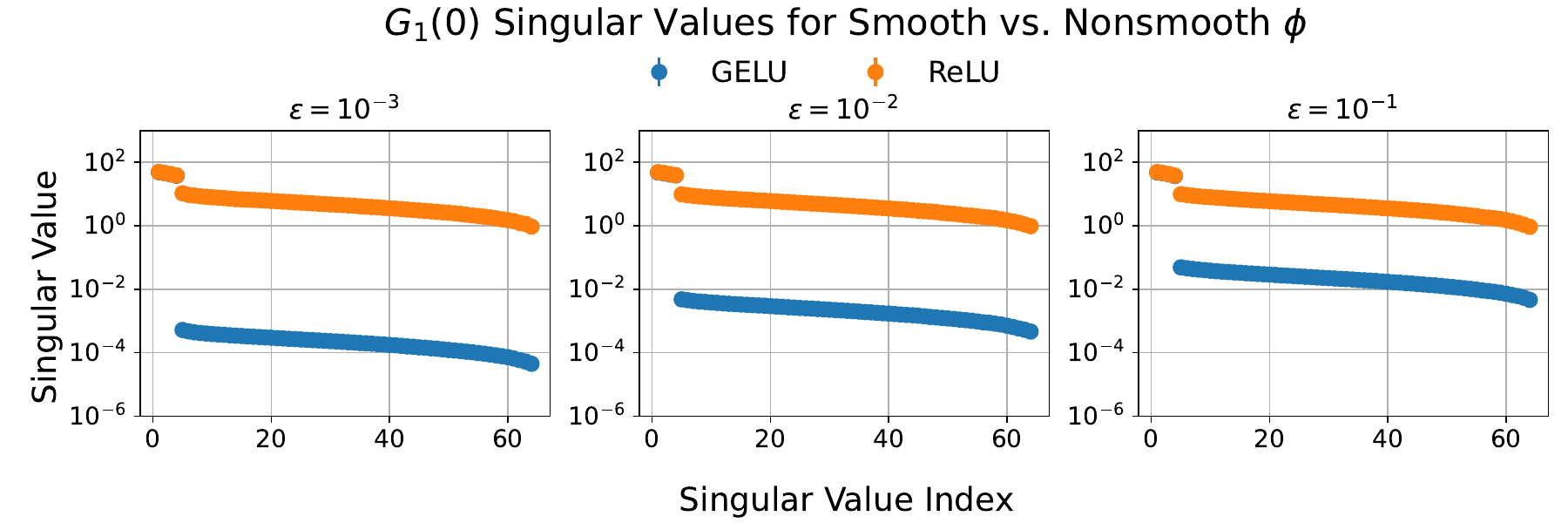}
    \caption{The bottom $d - K$ singular values of $\bm G_1(0)$ scale with $\epsilon$ when $\phi$ is smooth ($\gelu$), but do not change with $\epsilon$ when $\phi$ is nonsmooth ($\relu$). The top-$K$ singular values of $\bm G_1(0)$ overlap for both smooth and nonsmooth $\phi$.}
    \label{fig:smooth_vs_nonsmooth_svals}
\end{figure}

Here, we provide empirical evidence supporting \Cref{prop:nonsmooth_result}. In particular, we show that when $\phi = \relu$, the tail singular values of $\bm G_1(0)$ are independent of the (small) initialization scale $\epsilon$. Here, all experiments were done in \texttt{numpy} on a Macbook Air with an M3 chip. We adhered to the same data generation process and network initialization schemes as in \Cref{fig:smooth_theory_verify}, but with $\epsilon \in \{10^{-3}, 10^{-2}, 10^{-1}\}$. We note this setting violates the assumptions on $\bm X$ and $\bm W_1(0)$ in \Cref{prop:nonsmooth_result}. We computed the singular values of $\bm G_1(0)$ for both $\phi = \gelu$ and $\phi = \relu$ over $10$ trials. 

\paragraph{Results.} All $d$ singular values of $\bm G_1(0)$ are shown in \Cref{fig:smooth_vs_nonsmooth_svals}. Clearly, when $\phi = \gelu$, the tail singular values of $\bm G_1(0)$ scale linearly with $\epsilon$. Meanwhile, when $\phi = \relu$, the bottom $d - K$ singular values appear to be on a similar order as the top-$K$ and \emph{independent} of $\epsilon$. This supports our broader conclusions from \Cref{prop:nonsmooth_result}, despite the specific theoretical assumptions in \Cref{prop:nonsmooth_result} being violated here. 
    
\section{Proofs for \Cref{sec:two_layer_theory}}
\label[appendix]{sec:smooth_proofs}
In this section, we first re-state our assumptions in \Cref{ssec:smooth_assumptions}, and derive the analytical form of the loss gradient with respect to (w.r.t.) $\bm W_1$ in \Cref{ssec:smooth_gradient}. We then provide supporting lemmas for \Cref{thm:smooth_main_result_main_body} with their proofs in \Cref{ssec:smooth_aux_proofs}, and the proof of \Cref{thm:smooth_main_result_main_body} in \Cref{ssec:smooth_main_proof}. Finally, we provide the proof of \Cref{coro:quantify_thm_bound} in \Cref{ssec:coro_proof}.

\subsection{Notation and Assumptions}
\label[appendix]{ssec:smooth_assumptions}
Here, we re-state our notation and assumptions for \Cref{thm:smooth_main_result_main_body} for convenience, and introduce some new notation.

\paragraph{Notation.} We use unbolded letters $x, X$ for scalars, bold lower case letters $\bm x$ for vectors, and bold capital letters $\bm X$ for matrices. For some $N \in \mbb{N}$, $[N]$ denotes the set $\{1, 2, \dots, N\}$. For scalars $a, b$, we say $a \lesssim \mc{O}(b)$ if there exists a constant $C$ s.t. $a \leq C \cdot b$, $a \gtrsim \Omega(b)$ if $a \geq C \cdot b$, and $a = \Theta(b)$ if $a = C \cdot b$. We use $\sigma_i(\bm X)$, $\| \bm X \|_F$, $\| \bm X \|_1$, $\| \bm X \|_{\infty}$, and $\| \bm X \|_{\max}$ to respectively denote the $i^{th}$ singular value, Frobenius norm, matrix-$1$ norm, matrix-$\infty$ norm, and maximum magnitude element, and use $\bm X_{i, :}$ and $\bm X_{:, j}$ to respectively denote the $i^{th}$ row and $j^{th}$ column of $\bm X$, where $\bm X_{i, :}$ is written as a column vector. Finally, $\mc{R}\left( \bm X \right)$ denotes the range (or column space) of $\bm X$, and $\mc{R}^\perp\left( \bm X \right)$ its orthogonal complement.

We now re-state our assumptions here for convenience.

\dataassumption*

\trainassumption*

\residualassumption*

\gradientassumption*

We emphasize \Cref{thm:smooth_main_result_main_body} only relies on Assumptions~\ref{assum:input_data} through \ref{assum:residual}, while \Cref{coro:quantify_thm_bound} additionally relies on \Cref{assum:gradient}.

\subsection{Analytical Gradient Form}
\label[appendix]{ssec:smooth_gradient}
Here, we derive the analytical form of the loss gradient w.r.t. $\bm W_1$, which we denote as $\bm G_1$. Recall 
\begin{align*}
    \bm H_0 = \bm X \in \mbb{R}^{d \times N}, \; \; \bm Z_1 = \bm W_1 \bm H_0 \in \mbb{R}^{m \times N}, \; \; \bm H_1 = \phi\left(\bm Z_1 \right) \in \mbb{R}^{m \times N}, \; \; \text{and} \; \; \bm Z_2 = \bm W_2 \bm H_1 \in \mbb{R}^{K \times N}.
\end{align*}
Notice $\bm Z_2 = f_{\bm W_1}(\bm X)$. Define
\begin{align*}
    &\bm \Delta_2 = \nabla_{\bm Z_2} \mc{L}\left( \bm W_1 \right) = \bm Z_2 - \bm Y \in \mbb{R}^{K \times N} \; \; \text{and} \\
    &\bm \Delta_1 = \nabla_{\bm Z_1} \mc{L}\left( \bm W_1 \right) %\frac{\partial \ell}{\partial \bm Z_1} = \frac{\partial \ell}{\partial \bm Z_2} \frac{\partial \bm Z_2}{\partial \bm H_1} \frac{\partial \bm H_1}{\partial \bm Z_1} 
    = \left( \nabla_{\bm Z_1} \bm H_1 \right)^\top \left( \nabla_{\bm H_1} \bm Z_2 \right) ^\top \nabla_{\bm Z_2} \mc{L}\left( \bm W_1 \right) =  \bm W_2^\top \bm \Delta_2 \odot \phi'(\bm Z_1) \in \mbb{R}^{m \times N},
\end{align*}
where $\phi'(\cdot)$ is the derivative of activation function $\phi(\cdot)$. The gradient of the loss w.r.t. $\bm W_1$ via backpropagation is
\begin{equation}
    \bm G_1 \coloneqq \nabla_{\bm W_1} \mc{L}\left( \bm W_1 \right) = \bm \Delta_1 \bm H_0^\top = \bm \Delta_1 \bm X^\top.
\end{equation}
For all $t > 0$ and $l \in \{1, 2\}$, let $\bm Z_l(t)$, $\bm H_1(t)$, $\bm \Delta_l(t)$, and $\bm G_1(t)$ denote the values of their corresponding matrices at iteration $t$. Then, 
\begin{equation*}
    \bm \Delta_2(t) = \bm Z_2(t) - \bm Y , \; \; \bm \Delta_1(t) = \bm W_2^\top(t) \bm \Delta_2(t) \odot \phi'(\bm Z_1(t)), \; \; \text{and} \; \; \bm G_1(t) = \bm \Delta_1(t) \bm X^\top.
\end{equation*}
Substituting the expression for $\bm \Delta_1(t)$ into $\bm G_1(t)$ yields
\begin{align}
    \label{eq:grad_iter_t}
    \bm G_1(t) &= \bm \Delta_1(t) \bm H_0^\top = \bigg( \bm W_2^\top \bm \Delta_2(t) \odot \phi'(\bm Z_1(t)) \bigg) \bm X^\top = \bigg( \bm W_2^\top \Big( \bm Z_2(t) - \bm Y \Big) \odot \phi'\big( \bm W_1(t) \bm X \big) \bigg) \bm X^\top.
\end{align}

\subsection{Supporting Results}
\label[appendix]{ssec:smooth_aux_proofs}
In this section, we provide supporting results that are useful in proving \Cref{thm:smooth_main_result_main_body}. For notational brevity, we use $\ell(t) := \mc{L}\left( \bm W_1(t) \right) = \ell(\bm Z_2(t), \bm Y)$ to denote the value of the loss at GD iteration $t$. 

\subsubsection{Auxiliary Results}
In this section, we provide auxiliary results in linear algebra. First, for a matrix $\bm A$, we define $| \bm A |$ as applying $| \cdot |$ element-wise to $\bm A$, i.e., $| \bm A|_{ij} = | \bm A_{ij} |$. Our first result relates the spectral norm of $\bm A$ with $| \bm A |$.
\begin{lemma}\
\label{lem:spectral_norm_matrix_absolute_value}
    For any $\bm A \in \mbb{R}^{m \times n}$, we have $\sigma_1\left( \bm A \right) \leq \sigma_1 \left( \left| \bm A \right| \right)$.
\end{lemma}
\begin{proof}
    For some square matrix $\bm B \in \mbb{R}^{n \times n}$, let $\rho\left( \bm B \right) $  denote its spectral radius. Without loss  of generality, suppose $m \geq n$. By definition, we have $\sigma_1^2\left( \bm A \right) = \rho \left( \bm A^\top \bm A \right)$, so we focus on upper bounding $\rho \left( \bm A^\top \bm A \right)$ in terms of $\rho \left( \left| \bm A^\top \bm A  \right| \right)$. First, we trivially have $\bm A^\top \bm A \leq | \bm A^\top \bm A|$, where $\leq$ is done element-wise. From \cite[Exercise 8.1.9]{horn2012matrix}, 
    \begin{align*}
        | \bm A^\top \bm A| \leq | \bm A |^\top | \bm A|.
    \end{align*}
    Then, applying \cite[Theorem 8.1.18]{horn2012matrix} twice yields
    \begin{align*}
        \rho\left( \bm A^\top \bm A \right) \leq \rho\left( \left| \bm A^\top \bm A \right| \right) \leq \rho \left( | \bm A |^\top | \bm A | \right),
    \end{align*}
    which implies
    \begin{align*}
        \sigma_1^2\left( \bm A \right) = \rho\left( \bm A^\top \bm A \right) \leq \rho \left( | \bm A |^\top | \bm A | \right) = \sigma_1^2\left( | \bm A | \right),
    \end{align*}
    which completes the proof. 
\end{proof}

\medskip 
Our next result applies upper bounds on the matrix $\infty$ and $1$ norms on products of particular matrix types. 
\begin{lemma}
\label{lem:matrix_1_infty_norm_bounds}
    Suppose $\bm W \in \mbb{R}^{m \times d}$ satisfies $\bm W^\top \bm W = \epsilon^2 \bm I_d$, and $\bm X \in \mbb{R}^{d \times N}$ satisfies $\bm X \bm X^\top = \bm I_d$. Then, we have
    \begin{align*}
        \| \bm W \bm X \odot \bm W \bm X \|_{\infty} \leq \epsilon^2 \quad \text{and} \quad \| \bm W \bm X \odot \bm W \bm X \|_1 \leq \epsilon^2. 
    \end{align*}
\end{lemma}
\begin{proof}
    Define $\bm Z := \bm W \bm X$. Notice $\left( \bm Z \odot \bm Z \right)_{ij} = \bm Z_{ij}^2$, and so $\| \bm Z \odot \bm Z \|_{\infty}$ and $\| \bm Z \odot \bm Z \|_1$ are the maximum row and column squared Euclidean norms of $\bm Z$. Next, notice the $i^{th}$ column of $\bm Z$ is $\bm Z_{i, :} := \bm W_{i, :}^\top \bm X$. Therefore,
    \begin{align*}
        \| \bm Z_{i, :} \|_2^2 = \| \bm W_{i, :}^\top \bm X \|_2^2 = \bm W_{i, :}^\top \bm X \bm X^\top \bm W_{i, :} = \bm W_{i, :}^\top \bm W_{i, :} = \| \bm W_{i, :} \|_2^2.
    \end{align*}
    Since $\bm W$ has $\epsilon$-scaled orthonormal columns, we have $\| \bm W_{i, :} \|_2^2 \leq \epsilon^2$ for all $i \in [m]$. Therefore, 
    \begin{align*}
        \| \bm Z \odot \bm Z \|_{\infty} = \max_{i \in [m]} \| \bm Z_{i, :} \|_2^2 = \max_{i \in [m]} \| \bm W_{i, :}^\top \bm X \|_2^2 = \max_{i \in [m]}  \| \bm W_{i, :} \|_2^2 \leq \epsilon^2. 
    \end{align*}
    Similarly, the $j^{th}$ column of $\bm Z$ is $\bm Z_{:, j} = \bm W \bm X_{:, j}$. Therefore, 
    \begin{align*}
        \| \bm Z_{:, j} ||_2^2 = \| \bm W \bm X_{:, j} \|_2^2 = \bm X_{:, j}^\top \bm W^\top \bm W \bm X_{:, j} = \epsilon^2 \bm X_{:, j}^\top \bm X_{:, j} = \epsilon^2 \| \bm X_{:, j} \|_2^2.
    \end{align*}
    Since $\bm X$ is whitened, i.e., $\bm X \bm X^\top = \bm I_d$, we have $\| \bm X_{:, j} \|_2^2 \leq 1$ for all $j \in [N]$. Therefore, 
    \begin{align*}
        \| \bm Z \odot \bm Z \|_1 = \max_{j \in [N]} \| \bm Z_{:, j} \|_2^2 = \max_{j \in [N]} \| \bm W \bm X_{:, j} \|_2^2 = \max_{j \in [N]} \epsilon^2 \| \bm X_{:, j} \|_2^2 \leq \epsilon^2.
    \end{align*}
    This completes the proof. 
\end{proof}

\medskip 

Next, for two matrices $\bm U_1 \in \mbb{R}^{d \times r}, \bm U_2 \in \mbb{R}^{d \times r}$ with orthonormal columns, recall the definition of $\| \sin \Psi\left( \bm U_1, \bm U_2 \right) \|_2$ in \Cref{def:princ_angles}. Also define
% \begin{align}
% \label{eq:subspace_sin_theta}
%     \sin\Psi\left( \bm U_1, \bm U_2 \right) = \begin{bmatrix}
%         \sin\left( \arccos\left( \sigma_1\left( \bm U_1^\top \bm U_2 \right) \right) \right) &
%         \sin\left( \arccos\left( \sigma_2\left( \bm U_1^\top \bm U_2 \right) \right) \right) & \dots & \sin\left( \arccos\left( \sigma_r\left( \bm U_1^\top \bm U_2 \right) \right) \right)
%     \end{bmatrix}^\top  
% \end{align}
% and
\begin{align}
\label{eq:subspace_dist_def}
    \dist\left( \bm U_1, \bm U_2 \right) = \left\| \bm U_1 \bm U_1^\top - \bm U_2 \bm U_2^\top \right\|_F.
\end{align}
The next result relates $\| \sin \Psi\left( \bm U_1, \bm U_2 \right) \|_2$ and $\dist\left( \bm U_1, \bm U_2 \right)$. 

\begin{lemma}
\label{lem:aux_subspace_angles}
    Let $\bm U_1, \bm U_2 \in \mbb{R}^{d \times K}$ have orthonormal columns, and $\bm U_{1, \perp} \in \mbb{R}^{d \times (d - K)}$ have orthonormal columns that satisfy $\bm U_1^\top \bm U_{1, \perp} = \bm 0_{K \times (d - K)}$. Then, 
    \begin{align*}
        \dist\left( \bm U_1, \bm U_2 \right) = \sqrt{2} \cdot \| \sin\Psi\left( \bm U_1, \bm U_2 \right) \|_2 = \sqrt{2} \cdot \| \bm U_{1, \perp}^\top \bm U_2 \|_F.
    \end{align*}
\end{lemma}
\begin{proof}
    Let $\psi_1, \dots, \psi_K$ denote the $K$ principal angles between $\bm U_1$ and $\bm U_2$, which are exactly $\psi_i = \arccos\left( \sigma_i\left( \bm U_1^\top \bm U_2 \right) \right)$. Then, from \cite[Lemma 2]{xu2025understanding}, we have
    \begin{align*}
        \dist\left( \bm U_1, \bm U_2 \right) = \sqrt{2} \cdot \sqrt{\sum\limits_{i=1}^K \sin^2\left( \psi_i \right)} = \sqrt{2} \cdot \| \sin \Psi\left( \bm U_1, \bm U_2 \right) \|_2.
    \end{align*}
    It suffices to show that $\| \sin \Psi\left( \bm U_1, \bm U_2 \right) \|_2 = \| \bm U_{1, \perp}^\top \bm U_2 \|_F$. We can write
    \begin{align*}
        \bm U_2 = \bm U_1 \bm U_1^\top \bm U_2 + \bm U_{1, \perp} \bm U_{1, \perp}^\top \bm U_2. 
    \end{align*}
    Therefore,
    \begin{align*}
        &\| \bm U_{1, \perp}^\top \bm U_2 \|_F^2 = \| \bm U_2\|_F^2 - \| \bm U_1^\top \bm U_2 \|_F^2 = K - \sum\limits_{i=1}^K \sigma_i^2(\bm U_1^\top \bm U_2) \\
        &= K - \sum\limits_{i=1}^K \cos^2\left( \psi_i \right) = \sum\limits_{i=1}^K \left(1 - \cos^2\left( \psi_i \right) \right) = \sum\limits_{i=1}^K \sin^2(\psi_i) = \| \sin \Psi\left( \bm U_1, \bm U_2 \right) \|_2^2.
    \end{align*}
    This completes the proof. 
\end{proof}

\medskip 

Our next result upper bounds the alignment between two subspaces after being projected onto a third subspace. 
\begin{lemma}
\label{lem:subspace_align_after_proj}
    Let $\bm Q \in \mbb{R}^{m \times d}$ and $\bm U_1, \bm U_2 \in \mbb{R}^{m \times K}$ all have orthonormal columns. Define $\bm U_{1, \perp} \in \mbb{R}^{m \times (m - K)}$ to have orthonormal columns that satisfy $\bm U_1^\top \bm U_{1, \perp} = \bm 0_{K \times (m - K)}$, and $\widetilde{\bm U}_{1, \perp} \in \mbb{R}^{d \times (d - K)}$ to be an orthonormal basis for the subspace $\mc{R}^\perp\left( \bm Q^\top \bm U_1 \right)$. Then, we have
    \begin{align*}
        \left\| \bm U_2^\top \bm Q \widetilde{\bm U}_{1, \perp} \right\|_F \leq \| \bm U_2^\top \bm U_{1, \perp} \|_F.
    \end{align*}
\end{lemma}
\begin{proof}
    Since $\bm U_1 \in \mbb{R}^{m \times K}$ and $\bm U_2 \in \mbb{R}^{m \times (m - K)}$ each have orthonormal columns, and $\bm U_1^\top \bm U_{1, \perp} = \bm 0_{m \times (m -K)}$, we have
    \begin{align*}
        \bm I_m = \bm U_1 \bm U_1^\top + \bm U_{1, \perp} \bm U_{1, \perp}^\top.
    \end{align*}
    Therefore, 
    \begin{align*}
        &\bm U_2^\top \bm Q \widetilde{\bm U}_{1, \perp} = \bm U_2^\top \left( \bm U_1 \bm U_1^\top + \bm U_{1, \perp} \bm U_{1, \perp}^\top \right) \bm Q \widetilde{\bm U}_{1, \perp} \\
        &= \underbrace{\bm U_2^\top \bm U_1 \bm U_1^\top \bm Q \widetilde{\bm U}_{1, \perp}}_{(a)} + \bm U_2^\top \bm U_{1, \perp} \bm U_{1, \perp}^\top \bm Q \widetilde{\bm U}_{1, \perp}.
    \end{align*}
    Since $\widetilde{\bm U}_{1, \perp} \in \mc{R}^\perp \left( \bm Q^\top \bm U_1 \right)$, we have $\bm U_1^\top \bm Q \widetilde{\bm U}_{1, \perp} = \left( \bm Q^\top \bm U_1 \right)^\top \widetilde{\bm U}_{1, \perp} = \bm 0_{K \times (d - K)}$, and so $(a) = \bm 0_{K \times (d - K)}$. Thus, 
    \begin{align*}
        &\left\| \bm U_2^\top \bm Q \widetilde{\bm U}_{1, \perp} \right\|_F = \left\| \bm U_2^\top \bm U_{1, \perp} \bm U_{1, \perp}^\top \bm Q \widetilde{\bm U}_{1, \perp} \right\|_F  \\
        &\leq \left\| \bm U_2^\top \bm U_{1, \perp} \right\|_F \cdot \sigma_1\left( \bm U_{1, \perp}^\top \bm Q \widetilde{\bm U}_{1, \perp} \right) \leq \| \bm U_2^\top \bm U_{1, \perp} \|_F \cdot \sigma_1\left( \bm U_{1, \perp} \right) \cdot \sigma_1\left( \bm Q \right) \cdot \sigma_1\left( \widetilde{\bm U}_{1, \perp} \right) = \| \bm U_2^\top \bm U_{1, \perp} \|_F.
    \end{align*}
    This completes the proof. 
\end{proof}

\subsubsection{Gradient Bounds}
In this section, we provide bounds related to the gradient $\bm G_1(t)$. First, we upper bound $\| \bm G_1(t) \|_F$.
\begin{lemma}
\label{lem:smooth_grad_bound}
    For all $t \geq 0$, we have
    \begin{align*}
        \| \bm G_1(t) \|_F \leq \beta \cdot \sigma_1\left( \bm W_2 \right) \cdot \sqrt{2 \ell(t)}.
    \end{align*}
\end{lemma}
\begin{proof}
    Recall from \Cref{eq:grad_iter_t}, we have
    \begin{align*}
        \bm G_1(t) = \Big( \bm W_2^\top \bm \Delta_2(t) \odot \phi'\left( \bm Z_1(t) \right) \Big) \bm X^\top,
    \end{align*}
    and so 
    \begin{align*}
        &\| \bm G_1(t) \|_F = \left\| \Big( \bm W_2^\top \bm \Delta_2(t) \odot \phi'\left( \bm Z_1(t) \right) \Big) \bm X^\top \right\|_F \leq \left\| \bm W_2^\top \bm \Delta_2(t) \odot \phi'\left( \bm Z_1(t) \right) \right\|_F \cdot \sigma_1(\bm X) \\
        &\overset{(i)}{=} \left\| \bm W_2^\top \bm \Delta_2(t) \odot \phi'\left( \bm Z_1(t) \right) \right\|_F \overset{(ii)}{\leq} \beta \cdot \| \bm W_2^\top \bm \Delta_2(t) \|_F = \beta \cdot \sigma_1(\bm W_2) \cdot \| \bm \Delta_2(t) \|_F,
    \end{align*}
    where $(i)$ is because $\bm X \bm X^\top = \bm I_d$, and $(ii)$ is because $\phi$ is $\beta$-Lipschitz. Finally, note 
    \[
        \ell(t) = \frac{1}{2} \left\| \bm W_2 \phi\left( \bm W_2(t) \bm X \right) - \bm Y \right\|_F^2 = \frac{1}{2} \| \bm \Delta_2(t) \|_F^2 \implies \left\| \bm \Delta_2(t) \right\|_F = \sqrt{2 \ell(t)}.
    \] 
    Putting everything together yields
    \begin{align*}
        \| \bm G_1(t) \|_F \leq \beta \cdot \sigma_1(\bm W_2) \cdot \sqrt{2 \ell(t)},
    \end{align*}
    which completes the proof. 
\end{proof}

\medskip

Next, we show the gradient $\bm G_1(t)$ is locally Lipschitz.
\begin{lemma}
\label{lem:smooth_grad_lipschitz}
    Suppose $\bm W_1(0) \in \mbb{R}^{m \times d}$ is an arbitrary $\epsilon$-scaled semi-orthogonal matrix, i.e., $\bm W_1^\top(0) \bm W_1(0) = \epsilon^2 \bm I_d$. Define $\bm \Delta_2\left( \bm W_1 \right) := \bm W_2 \phi(\bm W_1 \bm X) - \bm Y, $ and $\mc{M} := \left\{ \bm W_1 \in \mbb{R}^{m \times d}: \| \bm \Delta_2(\bm W_1) \|_{\max} \leq  M \right \}$. Also define $\bm G_1\left( \bm W_1 \right)$ to be $\bm G_1$ at some $\bm W_1 \in \mbb{R}^{m \times d}$. Then, for all $\bm W_1, \widehat{\bm W}_1 \in \mc{M}$, we have
    \begin{align*}
        \| \bm G_1\left( \bm W_1 \right) - \bm G_1(\widehat{\bm W}_1) \|_F \leq  \left( \mu \cdot M \cdot \| \bm W_2 \|_1 + \beta^2 \cdot \sigma_1^2(\bm W_2) \right) \cdot \| \bm W_1 - \widehat{\bm W}_1 \|_F.
    \end{align*}
\end{lemma}
\begin{proof}
    Suppose $\bm W_1, \widehat{\bm W}_1 \in \mc{M}$. Let $\bm G_1(\bm W_1), \bm G_1(\widehat{\bm W}_1)$ denote $\bm G_1$ at $\bm W_1$ and $\widehat{\bm W}_1$, and let $\bm \Delta_2 := \bm \Delta_2(\bm W_1)$ and $\hat{\bm \Delta}_2 := \bm \Delta_2(\widehat{\bm W}_1)$. Then, we have
    \begin{align*}
        &\left \| \bm G_1(\bm W_1) - \bm G_1(\widehat{\bm W}_1) \right \|_F \\
        &= \left\| \left( \bm W_2^\top \bm \Delta_2 \odot \phi'\left( \bm W_1 \bm X \right) - \bm W_2^\top \hat{\bm \Delta}_2 \odot \phi'( \widehat{\bm W}_1 \bm X ) \right) \bm X^\top \right\|_F \\
        &\leq \left\| \bm W_2^\top \bm \Delta_2 \odot \phi'\left( \bm W_1 \bm X \right) - \bm W_2^\top \bm \Delta_2 \odot \phi'( \widehat{\bm W}_1 \bm X ) \right \|_F \\
        &= \left\| \bm W_2^\top \bm \Delta_2 \odot \phi'(\bm W_1 \bm X) - \bm W_2^\top \bm \Delta_2 \odot \phi'(\widehat{\bm W}_1 \bm X) + \bm W_2^\top \bm \Delta_2 \odot \phi'(\widehat{\bm W}_1 \bm X) - \bm W_2^\top \hat{\bm \Delta}_2 \odot \phi'(\widehat{\bm W}_1 \bm X)\right\|_F \\
        &\leq \left( \left\| \bm W_2^\top \bm \Delta_2 \odot \left( \phi'(\bm W_1 \bm X) - \phi'(\widehat{\bm W}_1 \bm X) \right) \right\|_F + \left\| \bm W_2^\top \left( \bm \Delta_2 - \hat{\bm \Delta}_2 \right) \odot \phi'(\widehat{\bm W}_1 \bm X) \right\|_F \right) \\
        &= \left( \left\| \bm W_2^\top \bm \Delta_2 \odot \left( \phi'(\bm W_1 \bm X) - \phi'(\widehat{\bm W}_1 \bm X) \right) \right\|_F + \left\| \bm W_2^\top \bm W_2 \left( \phi(\bm W_1 \bm X) - \phi(\widehat{\bm W}_1 \bm X) \right) \odot \phi'(\widehat{\bm W}_1 \bm X) \right\|_F \right) \\
        &\leq \left(  \| \bm W_2^\top \bm \Delta_2\|_{\max} \cdot \left \| \phi'( \bm W_1 \bm X ) - \phi'(\widehat{\bm W}_1 \bm X) \right\|_F + \beta \cdot \sigma_1^2(\bm W_2) \cdot \| \phi(\bm W_1 \bm X) - \phi(\widehat{\bm W}_1 \bm X) \|_F \right) \\
        &\overset{(i)}{\leq} \left( \mu \cdot \| \bm W_2 \|_1 \cdot \| \bm \Delta_2 \|_{\max} \cdot \| \bm W_1 \bm X - \widehat{\bm W}_1 \bm X \|_F + \beta^2 \cdot \sigma_1^2(\bm W_2) \cdot \| \bm W_1 \bm X - \widehat{\bm W}_1 \bm X \|_F \right) \\
        &\overset{(ii)}{\leq}  \left( \mu \cdot M \cdot \| \bm W_2 \|_1 + \beta^2 \cdot \sigma_1^2(\bm W_2) \right) \cdot \| \bm W_1 - \widehat{\bm W}_1 \|_F,
    \end{align*}
    where $(i)$ is because $\phi$ and $\beta$-Lipschitz and $\mu$-smooth, and $(ii)$ is because $\bm W_1, \widehat{\bm W}_1 \in \mc{M}$, which implies $\| \bm \Delta_2 \|_{\max} \leq M$ by definition. This completes the proof.
\end{proof}

\noindent From now on, we define $\gamma_L :=  \mu \cdot M \cdot \| \bm W_2 \|_1 + \beta^2 \cdot \sigma_1^2(\bm W_2)$. Note that by the Descent Lemma \cite{bertsekas1995nonlinear}, if $\eta \leq \frac{1}{\gamma_L}$, then the loss does not increase. Thus, we assume $\eta \leq \frac{1}{\gamma_L}$.

\subsubsection{Gradient Singular Values and Subspaces}
In this section, we provide results characterizing the singular values and subspaces of the gradient. Before proceeding, we provide some additional definitions. Let $\bm G_1(t) = \bm L_1(t) \bm \Sigma_1(t) \bm R_1^\top(t)$ denote an SVD of $\bm G_1(t)$, with $\bm L_{1, 1}(t) \in \mbb{R}^{m \times K}$, $\bm \Sigma_{1, 1}(t) \in \mbb{R}^{K \times K}$, and $\bm R_{1, 1}(t) \in \mbb{R}^{d \times K}$ denoting the top-$K$ components, and $\bm L_{1, 2}(t) \in \mbb{R}^{m \times (d - K)}$, $\bm \Sigma_{1, 2}(t) \in \mbb{R}^{(d - K) \times (d - K)}$, and $\bm R_{1, 2}(t) \in \mbb{R}^{d \times (d - K)}$ denoting the bottom $d - K$ components. 

\medskip 

\noindent Our first result shows under small initialization scale $\epsilon$, $\bm G_1(0)$ is approximately rank-$K$.
\begin{lemma}
\label{lem:smooth_grad_init_svals}
    Suppose Assumptions~\ref{assum:input_data} and \ref{assum:network} hold. 
    Define 
    \[
        r(\epsilon) = \epsilon \cdot \phi'(0) \cdot \sigma_1^2(\bm W_2) \cdot \left( \phi'(0) + \frac{\mu}{2} \cdot \epsilon \right) + \mu \cdot \| \bm W_1(0) \bm X \|_{\max} \cdot \sigma_1(\bm W_2) \cdot \left( \beta \cdot \sigma_1(\bm W_2) \cdot \sqrt{d} \cdot \epsilon + \left\| \bm Y \right \|_F \right).
    \]
    If  $\bm W_1^\top(0) \bm W_1(0) = \epsilon^2 \bm I_d$, where $\epsilon$ satisfies
    $r(\epsilon) < \frac{\phi'(0) \cdot \sigma_K\left( \bm W_2^\top \bm Y \bm X^\top \right)}{2}$,
    then, for all $i \in [d]$, we have
    \begin{align*}
        \sigma_i\left( \bm G_1(0) \right) \in \begin{cases}
            \left[ \phi'(0) \cdot \sigma_i(\bm W_2^\top \bm Y \bm X^\top) - r(\epsilon), \phi'(0) \cdot \sigma_i(\bm W_2^\top \bm Y \bm X^\top) + r(\epsilon) \right] & i = 1, \dots, K \\
            \hfil \left[0, r(\epsilon) \right] & i = K + 1, \dots, d
        \end{cases}
    \end{align*} 
\end{lemma}
\begin{proof} 
    Recall from \eqref{eq:grad_iter_t} that
    \begin{equation*}
        \bm G_1(0) = \left( \bm W_2^\top \bm \Delta_2(0) \odot \phi'(\bm W_1(0) \bm X) \right) \bm X^\top.
    \end{equation*}
    Now, define a term $\bm P(0) := \big( \bm W_2^\top \bm \Delta_2(0) \odot \phi'\big(\bm W_1(0) \bm X \big) \big) - \phi'(0) \cdot \bm W_2^\top \bm \Delta_2(0)$, and $\bm \Gamma(0) := \bm \Delta_2(0) \bm X^\top$. This means
    \begin{equation}
        \label{eq:smooth_perturbed_grad}
        \bm G_1(0) = \phi'(0) \cdot \bm W_2^\top \bm \Delta_2(0) \bm X^\top + \bm P(0) \bm X^\top = \phi'(0) \cdot \bm W_2^\top \bm \Gamma(0) + \bm P(0) \bm X^\top. 
    \end{equation}
    Here, $\phi'(0) \cdot \bm W_2^\top \bm \Gamma(0)$ is the linear component of $\bm G_1(0)$, while $\bm P(0) \bm X^\top$ is a perturbation of $\phi'(0) \cdot \bm W_2^\top \bm \Gamma(0)$.
    
    \paragraph{Singular values of $\bm W_2^\top \bm \Gamma(0)$.} We first analyze the singular values of $\bm W_2^\top \bm \Gamma(0)$, which is (at most) rank-$K$. Note
    \begin{align*}
        &\bm W_2^\top\bm \Gamma(0) = \bm W_2^\top \bm \Delta_2(0) \bm X^\top = \left( \bm W_2^\top \bm W_2 \phi(\bm W_1(0) \bm X) - \bm W_2^\top \bm Y \right) \bm X^\top = \big(\bm W_2^\top \bm W_2 \phi(\bm W_1(0) \bm X) \bm X^\top - \bm W_2^\top \bm Y \bm X^\top  \big). 
    \end{align*}
    Also note
    \begin{align*}
        &\sigma_1\big(\bm W_2^\top \bm W_2 \phi\left( \bm W_1(0) \bm X \right) \bm X^\top\big)  \leq \sigma_1\left (\bm W_2^\top \bm W_2 \phi\left( \bm W_1(0) \bm X \right) \right) \leq \sigma_1^2(\bm W_2) \cdot \sigma_1(\phi(\bm W_1(0) \bm X) ) \\
        &\leq \sigma_1^2(\bm W_2) \cdot \Big( \phi'(0) \cdot \sigma_1( \bm Z_1(0)) + \sigma_1(\bm R(\bm Z_1(0)) \Big),
    \end{align*}
    where the last inequality is from taking the Taylor expansion of $\phi(\bm W_1(0) \bm X)$ element-wise around $\bm 0_m \bm 0_N^\top$, with $\bm R(\bm Z_1(0))$ denoting the remainder, and then applying the triangle inequality. Since $\bm W_1(0)$ is an $\epsilon$-scaled semi-orthogonal matrix, and $\bm X \bm X^\top = \bm I_d$, we have $\sigma_1(\bm Z_1(0)) \leq \sigma_1\left( \bm W_1(0) \right) = \epsilon$. Next, since $\phi(\cdot)$ is $\mu$-smooth, by Taylor's Remainder Theorem \citep{spivak2006calculus}, we have
    \begin{align*}
        \left| \bm R(\bm Z_1(0))_{ij} \right| \leq \frac{\mu}{2} \left| \left(\bm Z_1(0) \right)_{ij} \right|^2 = \frac{\mu}{2} \left( \bm Z_1(0) \right)_{ij}^2,
    \end{align*}
    or equivalently,
    \begin{align*}
        \left| \bm R(\bm Z_1(0)) \right| \leq \frac{\mu}{2} \left( \bm Z_1(0) \odot \bm Z_1(0) \right),
    \end{align*}
    where $| \cdot |$ and $\leq$ are element-wise. We then have
    \begin{align*}
        &\sigma_1\Big( \bm R \big( \bm Z_1(0) \big) \Big) \overset{(i)}{\leq} \sigma_1\Big( \big| \bm R \big( \bm Z_1(0) \big) \big| \Big) \overset{(ii)}{\leq} \frac{\mu}{2} \cdot \sigma_1\Big( \bm Z_1(0) \odot \bm Z_1(0) \Big) \\
        &\leq \frac{\mu}{2} \cdot \sqrt{\left\| \bm Z_1(0) \odot \bm Z_1(0) \right\|_\infty \left\| \bm Z_1(0) \odot \bm Z_1(0) \right\|_1} \\
        &\overset{(iii)}{\leq} \frac{\mu}{2} \cdot \sqrt{\epsilon^2 \cdot \epsilon^2} = \frac{\mu}{2} \cdot \epsilon^2,
    \end{align*}
    where $(i)$ is from \Cref{lem:spectral_norm_matrix_absolute_value}, $(ii)$ is from \cite[Theorem 8.1.18]{horn2012matrix}, and $(iii)$ is from \Cref{lem:matrix_1_infty_norm_bounds}.
    In summary, 
    \begin{align*}
        &\sigma_1\left(\bm W_2^\top \bm W_2 \phi(\bm W_1(0) \bm X)  \right) \leq \sigma_1^2(\bm W_2) \cdot \big( \phi'(0) \cdot \sigma_1(\bm Z_1(0)) + \sigma_1\left( \bm R(\bm Z_1(0))\right) \big) \\
        &\leq \sigma_1^2(\bm W_2) \cdot \left( \phi'(0) \cdot \epsilon + \frac{\mu}{2} \cdot \epsilon^2 \right) = \epsilon \cdot \sigma_1^2(\bm W_2) \cdot \left( \phi'(0) + \frac{\mu}{2} \cdot \epsilon \right). 
    \end{align*}
    Thus, by Weyl's inequality \citep{weyl1949inequalities}:
    \begin{align}
    \label{eq:init_grad_linear_sval_diff}
        &\left| \sigma_i\big( \bm W_2^\top \bm \Gamma(0) \big) - \sigma_i\big( \bm W_2^\top \bm Y \bm X^\top \big) \right| \leq \sigma_1\big( \bm W_2^\top \bm W_2 \phi(\bm W_1(0) \bm X) \big) \leq \epsilon \cdot \sigma_1^2(\bm W_2) \cdot \left( \phi'(0) + \frac{\mu}{2} \cdot \epsilon \right).
    \end{align}
    Let $r_1(\epsilon) := \epsilon \cdot \phi'(0) \cdot \sigma_1^2\left( \bm W_2 \right) \cdot \left( \phi'(0) + \frac{\mu}{2} \cdot \epsilon \right)$. By rearranging \eqref{eq:init_grad_linear_sval_diff}, for all $i = 1, \dots, K$, we have
    \begin{align}
        \label{eq:W2_Gamma0_svals}
        \phi'(0) \cdot \sigma_i(\bm W_2^\top \bm Y \bm X^\top) - r_1(\epsilon) \leq \phi'(0) \cdot \sigma_i\big( \bm W_2^\top \bm \Gamma(0) \big) \leq \phi'(0) \cdot \sigma_i(\bm W_2^\top \bm Y \bm X^\top) + r_1(\epsilon) 
    \end{align}
    For $i = K + 1, \dots, d$, we have $\sigma_i\left( \phi'(0) \cdot \bm W_2^\top \bm \Gamma(0) \right) = 0$ since $\bm W_2^\top \bm \Gamma(0)$ is at most rank-$K$.
   
    \paragraph{Singular values of $\bm G_1(0)$.} We now analyze the singular values of $\bm G_1(0)$. Recall $\phi'(x)$ is $\mu$-Lipschitz over $\mbb{R}$, so $|\phi'(x) - \phi'(0)| \leq \mu \cdot |x|$ for all $x \in \mbb{R}$. %Furthermore, from \ax{assumption on $\| \bm W_1(0) \bm X \|_{\max}$} we have $\| \bm Z_1(0) \|_{\max} \leq \frac{R \cdot \| \bm X \|_{\max} \cdot \epsilon}{\sqrt{m}}$ for some sufficiently large constant $R$. 
    Therefore, 
    \begin{align*}
        &\left| \phi'(\bm W_1(0) \bm X)_{ij} - \phi'(0) \right| \leq \mu \left| \big( \bm W_1(0) \bm X \big)_{ij} \right| \leq \mu \cdot \| \bm W_1(0) \bm X \|_{\max} \\
        &\implies \phi'(0) - \mu \cdot \| \bm W_1(0) \bm X \|_{\max} \leq \phi'(\bm W_1(0) \bm X)_{ij}  \leq \phi'(0) + \mu \cdot \| \bm W_1(0) \bm X \|_{\max} \\
        &\implies -\mu \cdot \| \bm W_1(0) \bm X \|_{\max} \leq \phi'\left( \bm W_1(0) \bm X \right)_{ij} - \phi'(0) \leq \mu \cdot \| \bm W_1(0) \bm X \|_{\max}.
    \end{align*}
    Therefore, 
    \begin{align*}
        -\mu \cdot \| \bm W_1(0) \bm X \|_{\max} \cdot \left(\bm W_2^\top \bm \Delta_2(0)\right)_{ij} \leq \underbrace{\left( \bm W_2^\top \bm \Delta_2(0) \right)_{ij} \cdot \left( \phi'\left( \bm W_1(0) \bm X \right)_{ij} - \phi'(0) \right)}_{:= \left( \bm P(0) \right)_{ij}} \leq \mu \cdot \| \bm W_1(0) \bm X \|_{\max} \cdot \left( \bm W_2^\top \bm \Delta_2(0) \right)_{ij}.
    \end{align*}
    % \begin{align*}
    %     &\left( \phi'(0) - \frac{C \cdot \mu \cdot \epsilon}{\sqrt{\max\{N, m}\}} \right) \cdot \left(\bm W_2^\top \bm \Delta_2(0)\right)_{ij} \leq \bigg( \bm W_2^\top \bm \Delta_2(0) \odot \phi'\left( \bm W_1(0) \bm X \right) \bigg)_{ij} \leq \left( \phi'(0) + \frac{C \cdot \mu \cdot \epsilon}{\sqrt{\max\{N, m}\}} \right) \cdot \left(\bm W_2^\top \bm \Delta_2(0)\right)_{ij} \\
    %     &\implies - \frac{C \cdot \mu \cdot \epsilon}{\sqrt{\max\{N, m}\}} \cdot \Big( \bm W_2^\top \bm \Delta_2(0) \Big)_{ij} \leq \underbrace{\bigg( \bm W_2^\top \bm \Delta_2(0) \odot \phi'\left( \bm W_1(0) \bm X \right) - \phi'(0) \cdot \bm W_2^\top \bm \Delta_2(0) \bigg)_{ij}}_{= \big( \bm P(0) \big)_{ij}} \leq  \frac{C \cdot \mu \cdot \epsilon}{\sqrt{\max\{N, m}\}} \cdot \Big( \bm W_2^\top \bm \Delta_2(0) \Big)_{ij}. 
    % \end{align*}
    As a result, 
    \begin{align*}
        &\Big| \big( \bm P(0) \big)_{ij} \Big| \leq \mu \cdot \| \bm W_1(0) \bm X \|_{\max} \cdot \Big| \left( \bm W_2^\top \bm \Delta_2(0) \right)_{ij} \Big| \implies \| \bm P(0) \|_F \leq \mu \cdot \| \bm W_1(0) \bm X \|_{\max} \cdot \| \bm W_2^\top \bm \Delta_2(0) \|_F.
    \end{align*}
    Note $\| \bm W_2^\top \bm \Delta_2(0)\|_F \leq \sigma_1(\bm W_2) \cdot \| \bm \Delta_2(0)\|_F$, and so 
    \begin{align*}
        &\| \bm W_2^\top \bm \Delta_2(0)\|_F \leq \sigma_1(\bm W_2) \cdot \| \bm W_2 \phi(\bm W_1(0) \bm X) - \bm Y \|_F \leq \sigma_1(\bm W_2) \cdot \left( \beta \cdot \sigma_1(\bm W_2) \cdot \sqrt{d} \cdot \epsilon + \left\| \bm Y \right \|_F \right) 
    \end{align*}
    Again by Weyl's inequality, \Cref{eq:smooth_perturbed_grad}, and \Cref{eq:W2_Gamma0_svals},
    \begin{align*}
        &\left| \sigma_i(\bm G_1(0)) - \phi'(0) \cdot \sigma_i\big(\bm W_2^\top \bm \Gamma(0) \big) \right| \leq \sigma_1\big( \bm P(0) \bm X^\top \big) \leq \| \bm P(0) \bm X^\top \|_F \\
        &\leq \mu \cdot \| \bm W_1(0) \bm X \|_{\max} \cdot \| \bm W_2^\top \bm \Delta_2(0) \|_F \leq \mu \cdot \| \bm W_1(0) \bm X \|_{\max} \cdot \sigma_1(\bm W_2) \cdot \left( \beta \cdot \sigma_1(\bm W_2) \cdot \sqrt{d} \cdot \epsilon + \left\| \bm Y \right \|_F \right) .
    \end{align*} 
    Let $r_2(\epsilon) := \mu \cdot \| \bm W_1(0) \bm X \|_{\max} \cdot \sigma_1(\bm W_2) \cdot \left( \beta \cdot \sigma_1(\bm W_2) \cdot \sqrt{d} \cdot \epsilon + \left\| \bm Y \right \|_F \right) $. Therefore, for all $i \in [d]$, we have
    \begin{align*}
        &\phi'(0) \cdot \sigma_i\left( \bm W_2^\top \bm \Gamma(0) \right) - r_2(\epsilon) \leq \sigma_i\left( \bm G_1(0) \right) \leq \phi'(0) \cdot \sigma_i\left( \bm W_2^\top \bm \Gamma(0) \right) + r_2(\epsilon) \\
        &\implies \phi'(0) \cdot \sigma_i\left( \bm W_2^\top \bm Y \bm X^\top \right) - r_1(\epsilon) - r_2(\epsilon) \leq \sigma_i\left( \bm G_1(0) \right) \leq \phi'(0) \cdot \sigma_i\left( \bm W_2^\top \bm Y \bm X^\top \right) + r_1(\epsilon) + r_2(\epsilon).
    \end{align*}
    Define $r(\epsilon) = r_1(\epsilon) + r_2(\epsilon)$. Since $\bm W_2^\top \bm Y \bm X^\top$ is rank-$K$, for all $i \in [d]$, we have
    \begin{align*}
        \sigma_i\left( \bm G_1(0) \right) \in \begin{cases}
            \left[ \phi'(0) \cdot \sigma_i(\bm W_2^\top \bm Y \bm X^\top) - r(\epsilon), \phi'(0) \cdot \sigma_i(\bm W_2^\top \bm Y \bm X^\top) + r(\epsilon) \right] & i = 1, \dots, K \\
            \hfil \left[0, r(\epsilon) \right] & i = K + 1, \dots, d
        \end{cases}
    \end{align*}
    Finally, note that we require $\sigma_K(\bm G_1(0)) - \sigma_{K + 1}(\bm G_1(0)) > 0$, which implies
    \begin{align*}
        &\phi'(0) \cdot \sigma_K\left( \bm W_2^\top \bm Y \bm X^\top \right) - 2 \cdot r(\epsilon) > 0 \implies r(\epsilon) < \frac{\phi'(0) \cdot \sigma_K(\bm W_2^\top \bm Y \bm X^\top)}{2}, 
        % &\implies \phi'(0) \cdot \sigma_K(\bm W_2^\top \bm Y \bm X^\top) - 2 \cdot \epsilon \cdot \left( \phi'(0) \cdot \sigma_1^2(\bm W_2) \cdot \left(\phi'(0) + \epsilon \cdot \frac{\mu}{2} \right) + R \cdot \mu \cdot \sigma_1(\bm W_2) \cdot \left(1 + \epsilon \cdot \beta \cdot \sigma_1(\bm W_2) \right) \right) > 0.
    \end{align*}
    % Solving for $\epsilon$ yields
    % \begin{align*}
    %     \epsilon < \frac{\sqrt{\left( \phi'(0)^2 \sigma_1^2(\bm W_2) + R \mu \sigma_1(\bm W_2) \right)^2 + \left( \mu \sigma_1^2(\bm W_2) \right) \left( \phi'(0) + 2 R \beta  \right) \left( \phi'(0) \sigma_K(\bm W_2^\top \bm Y \bm X^\top) \right)} - \left( \phi'(0)^2 \sigma_1^2(\bm W_2) + R \mu \sigma_1(\bm W_2) \right) }{\mu \sigma_1^2(\bm W_2) \left( \phi'(0) + 2 R \beta \right)},
    % \end{align*}
    which completes the proof.
\end{proof}

\medskip

\noindent Our next result characterizes how the bottom $d - K$ left and singular subspaces change between $\bm G_1(t)$ and $\bm G_1(0)$.
\begin{lemma}
\label{lem:smooth_grad_subspace_change}
    Suppose Assumptions~\ref{assum:input_data} through \ref{assum:gradient} hold. Define
    \begin{align*}
        &\gamma_L := \mu \cdot M \cdot \| \bm W_2 \|_1 + \beta^2 \cdot \sigma_1^2\left( \bm W_2 \right), \quad \text{and} \\
        &r(\epsilon)= \epsilon \cdot \phi'(0) \cdot \sigma_1^2(\bm W_2) \cdot \left( \phi'(0) + \frac{\mu}{2} \cdot \epsilon \right) + \mu \cdot \| \bm W_1(0) \bm X \|_{\max} \cdot \sigma_1(\bm W_2) \cdot \left( \beta \cdot \sigma_1(\bm W_2) \cdot \sqrt{d} \cdot \epsilon + \left\| \bm Y \right \|_F \right).
    \end{align*}
    If $\bm W_1^\top(0) \bm W_1(0) = \epsilon^2 \bm I_d$ where $\epsilon$ satisfies $r(\epsilon) < \frac{\phi'(0) \cdot \sigma_K\left( \bm W_2^\top \bm Y \bm X^\top \right)}{2}$, and $\eta \leq \frac{1}{\gamma_L}$, then 
    for all $t \geq 1$, we have
    \begin{align*}
        &\left\| \bm L_{1, 1}^\top(t) \bm L_{1, 2}(0) \right\|_F %\dist\left( \bm W_1^\top(0) \bm L_{1, 1}(t) / \epsilon, \bm W_1^\top(0) \bm L_{1, 1}(0) / \epsilon \right) 
        \leq \frac{\gamma_L \cdot \beta \cdot \sigma_1\left( \bm W_2 \right) \cdot \sqrt{2 \ell(0)} \cdot \Theta(1) \cdot \left( 1 - \left(1 - \Theta(\eta) \right)^{\Theta(t)} \right)}{\phi'(0) \cdot \sigma_K\left( \bm W_2^\top \bm Y \bm X^\top \right) - r(\epsilon) - \sigma_{K + 1}\left( \bm G_1(t) \right)}, \quad \text{and} \\
        &\|\bm R_{1, 1}^\top(t) \bm R_{1, 2}(0) \|_F \leq \frac{\gamma_L \cdot \beta \cdot \sigma_1\left( \bm W_2 \right) \cdot \sqrt{2 \ell(0)} \cdot \Theta(1) \cdot \left( 1 - \left(1 - \Theta(\eta) \right)^{\Theta(t)} \right)}{\phi'(0) \cdot \sigma_K\left( \bm W_2^\top \bm Y \bm X^\top \right) - r(\epsilon) - \sigma_{K + 1}\left( \bm G_1(t) \right)}, %\frac{ 2\sqrt{2} \cdot \gamma_L \cdot \beta \cdot \sigma_1(\bm W_2) \cdot \sqrt{2 \ell(0)} \cdot \left( 1 - \left(1 - \Theta(\eta) \right)^{\Theta(t)} \right) }{ \gamma_{PL} \cdot \left( \sigma_K\left( \bm W_2^\top \bm Y \bm X^\top \right) - r(\epsilon) - s(t) \right) }
    \end{align*}
    which provides upper bounds on the change in the top-$K$ left and right singular subspaces of $\bm G_1(t)$ from $\bm G_1(0)$.
    % where
    % \begin{align*}
    %     \delta(t) := \frac{\gamma_L \cdot \beta \cdot \sigma_1\left( \bm W_2 \right) \cdot \sqrt{2 \ell(0)} \cdot \Theta(1) \cdot \left( 1 - \left(1 - \Theta(\eta) \right)^{\Theta(t)} \right)}{\phi'(0) \cdot \sigma_K\left( \bm W_2^\top \bm Y \bm X^\top \right) - r(\epsilon) - \sigma_{K + 1}\left( \bm G_1(t) \right)}.
    % \end{align*}

\begin{comment}
    for all $t \leq T_L$, we have
    \begin{align*}
        &\max\left\{ \dist\left( \bm L_{1, 2}(t), \bm L_{1, 2}(0)  \right), \dist\left( \bm R_{1, 2}(t), \bm R_{1, 2}(0) \right)  \right\} \leq \frac{ 2 \cdot \beta \cdot \sigma_1(\bm W_2) \cdot \sqrt{2 \ell(0)} \cdot \left( 1 - \left(1 - \Theta(\eta) \right)^{\Theta(t)} \right) }{ \sigma_K\left( \bm W_2^\top \bm Y \bm X^\top \right) - r(\epsilon) \left(1 + \left(1 + C_{L, 1}(\epsilon) \cdot \eta \right)^{C_{L, 2} \cdot t} \right) },
    \end{align*}
    and for all $t > T_L$:
    \begin{align*}
        &\max\left\{ \dist\left( \bm L_{1, 2}(t), \bm L_{1, 2}(0)  \right), \dist\left( \bm R_{1, 2}(t), \bm R_{1, 2}(0) \right)  \right\} \leq \frac{ 2 \cdot \beta \cdot \sigma_1(\bm W_2) \cdot \sqrt{2 \ell(0)} \cdot \left( 1 - \left(1 - \Theta(\eta) \right)^{\Theta(t)} \right) }{ \sigma_K\left( \bm W_2^\top \bm Y \bm X^\top \right) - r(\epsilon) - G_L },
    \end{align*}
\end{comment}

\end{lemma}
\begin{proof}
    Before proceeding, we note that at $t = 0$, we have $\dist\left( \bm W_1^\top(0) \bm L_{1, 1}(t) / \epsilon, \bm W_1^\top(0) \bm L_{1, 1}(0) / \epsilon \right) = \dist\left( \bm R_{1, 2}(t), \bm R_{1, 2}(0) \right) = 0$ exactly, so we consider $t \geq 1$. This is an direct application of Wedin's Sin Theorem \cite{wedin1972perturbation}. 
    We first upper bound $\| \bm G_1(t) - \bm G_1(0) \|_F$:
    \begin{align*}
        &\| \bm G_1(t) - \bm G_1(0) \|_F \overset{(i)}{\leq} \gamma_{L} \cdot \| \bm W_1(t) - \bm W_1(0) \|_F = \gamma_{L} \cdot \eta \cdot \left\| \sum\limits_{\tau=0}^{t-1} \bm G_1(\tau) \right\|_F \\
        &\leq \gamma_{L} \cdot \eta \cdot \sum\limits_{\tau=0}^{t-1} \| \bm G_1(\tau) \|_F \overset{(ii)}{\leq} \gamma_L \cdot \eta \cdot \| \bm G_1(0) \|_F \cdot \sum\limits_{\tau=0}^{t-1} \left(1 - \Theta(\eta) \right)^{\Theta(\tau)} \\
        &\overset{(iii)}{\leq} \gamma_L \cdot \beta \cdot \eta \cdot \sigma_1\left( \bm W_2 \right) \cdot \sqrt{2 \ell(0)} \cdot \sum\limits_{\tau=0}^{t-1} \left(1 - \Theta(\eta) \right)^{\Theta(\tau)} \\
        &\leq \gamma_L \cdot \beta \cdot \eta \cdot \sigma_1(\bm W_2) \cdot \sqrt{2 \ell(0)} \cdot \frac{1 - \left( 1 - \Theta(\eta) \right)^{\Theta(t)}}{\Theta(\eta)} \\
        &= \gamma_L \cdot \beta \cdot \sigma_1(\bm W_2) \cdot \sqrt{2 \ell(0)} \cdot \Theta(1) \cdot \left( 1 - \left(1 - \Theta(\eta) \right)^{\Theta(t)} \right),
    \end{align*}
    where $(i)$ is from \Cref{lem:smooth_grad_lipschitz}, $(ii)$ is from \Cref{assum:gradient}, and $(iii)$ is from \Cref{lem:smooth_grad_bound}.
    Then, directly from Wedin's Sin Theorem \cite{wedin1972perturbation} and \Cref{lem:aux_subspace_angles},
    \begin{align}
    \label{eq:wedin_R_subspace}
        &\| \bm R_{1, 1}^\top(t) \bm R_{1, 2}(0) \|_F \leq \frac{\| \bm G_1(t) - \bm G_1(0) \|_F}{\sigma_K\left( \bm G_1(0) \right) - \sigma_{K + 1}\left( \bm G_1(t) \right)} \leq \frac{\gamma_L \cdot \beta \cdot \sigma_1\left( \bm W_2 \right) \cdot \sqrt{2 \ell(0)} \cdot \Theta(1) \cdot \left( 1 - \left(1 - \Theta(\eta) \right)^{\Theta(t)} \right)}{\phi'(0) \cdot \sigma_K\left( \bm W_2^\top \bm Y \bm X^\top \right) - r(\epsilon) - \sigma_{K + 1}\left( \bm G_1(t) \right)},
    \end{align}
    which also applies to $\| \bm L_{1, 1}^\top(t) \bm L_{1, 2}(0)\|_F$. Note from \Cref{assum:gradient} and our condition on $r(\epsilon)$, the denominator is strictly positive. This completes the proof.
\begin{comment}
    From \Cref{lem:aux_subspace_angles}, we have
    \begin{align*}
        &\sqrt{2} \cdot \| \widetilde{\bm L}_{1, 2}^\top(t) \widetilde{\bm L}_{1, 1}(0) \|_F = \dist\left( \bm Q^\top \bm L_{1, 1}(0), \bm Q^\top \bm L_{1, 1}(t) \right) \\
        &= \left\| \bm Q^\top \bm L_{1, 1}(0) \bm L_{1, 1}^\top(0) \bm Q - \bm Q^\top \bm L_{1, 1}(t) \bm L_{1, 1}^\top(t) \bm Q \right\|_F \\
        &= \left\| \bm Q^\top \left( \bm L_{1, 1}(0) \bm L_{1, 1}^\top(0) - \bm L_{1, 1}(t) \bm L_{1, 1}^\top(t) \right) \bm Q \right\|_F \\
        &\leq \left\| \left( \bm L_{1, 1}(0) \bm L_{1, 1}^\top(0) - \bm L_{1, 1}(t) \bm L_{1, 1}^\top(t) \right) \right\|_F \cdot \sigma_1^2(\bm Q) \\
        &= \left\|\bm L_{1, 1}(0) \bm L_{1, 1}^\top(0) - \bm L_{1, 1}(t) \bm L_{1, 1}^\top(t) \right\|_F \\
        &= \dist\left( \bm L_{1, 1}(t), \bm L_{1, 1}(0) \right) = \sqrt{2} \cdot \| \bm L_{1, 1}^\top(t) \bm L_{1, 2}(0) \|_F.
    \end{align*}
    We can then apply Wedin's Sin Theorem \citep{wedin1972perturbation} to $\| \bm L_{1, 1}^\top(t) \bm L_{1, 2}(0) \|_F$ in the same manner as in \eqref{eq:wedin_R_subspace}, which completes the proof.
\end{comment}

\end{proof}

\medskip 

\subsection{Proof of \Cref{thm:smooth_main_result_main_body}}
\label{ssec:smooth_main_proof}
In this section, we provide a proof of \Cref{thm:smooth_main_result_main_body}. First, we provide a result that identifies a $p$-dimensional subspace where $\bm G_1(t)$ has a bounded component for all $t \geq 0$ --- \Cref{thm:smooth_main_result_main_body} follows straightforwardly after.

\begin{lemma}
\label{lem:smooth_main_result_small_helper}
   Suppose Assumptions~\ref{assum:input_data} through \ref{assum:residual} hold. 
   % Define
   %  \begin{align*}
   %      & \gamma_L = \mu \cdot M \cdot \| \bm W_2 \|_1 + \beta^2 \cdot \sigma_1^2\left( \bm W_2 \right), \\
   %      &r(\epsilon) = \epsilon \cdot \phi'(0) \cdot \sigma_1^2(\bm W_2) \cdot \left( \phi'(0) + \frac{\mu}{2} \cdot \epsilon \right) +  \mu \cdot \| \bm W_1(0) \bm X \|_{\max} \cdot \sigma_1(\bm W_2) \cdot \left( \beta \cdot \sigma_1(\bm W_2) \cdot \sqrt{d} \cdot \epsilon + \left\| \bm Y \right \|_F \right), \\
   %      &A(t) := \max\left\{ \left\| \sin\Psi\left( \bm L_{1, 1}(t), \bm L_{1, 1}(0) \right) \right\|_2, \left\| \sin\Psi\left( \bm R_{1, 1}(t), \bm R_{1, 1}(0) \right) \right\|_2 \right\}, \quad \text{and} \\
   %      &\delta(t) := \frac{\gamma_L \cdot \beta \cdot \sigma_1\left( \bm W_2 \right) \cdot \sqrt{2 \ell(0)} \cdot \Theta(1) \cdot \left( 1 - \left(1 - \Theta(\eta) \right)^{\Theta(t)} \right)}{\phi'(0) \cdot \sigma_K\left( \bm W_2^\top \bm Y \bm X^\top \right) - r(\epsilon) - \sigma_{K + 1}\left( \bm G_1(t) \right)}.
   %  \end{align*}
    If $\bm W_1^\top(0) \bm W_1(0) = \epsilon^2 \bm I_d$ where $\epsilon$ satisfies $r(\epsilon) < \frac{\phi'(0) \cdot \sigma_K\left( \bm W_2^\top \bm Y \bm X^\top \right)}{2}$, and $\eta \leq \frac{1}{\gamma_L}$,
    then there exist semi-orthogonal matrices $\bm V \in \mbb{R}^{d \times p}$ and $\bm U \in \mbb{R}^{m \times p}$ such that
    \begin{align*}
        &\bm W_1(0) \bm V = \epsilon \bm U, \quad \bm W_1^\top(0) \bm U = \epsilon \bm V, \quad \text{and} \\ %&\| \bm W_1(0) \bm V - \epsilon \bm U \|_F = \| \bm W_1^\top(0) \bm U - \epsilon \bm V \|_F = 0, \\
        &\max\left\{ \frac{\| \bm W_1(t + 1) \bm V - \bm W_1(t) \bm V \|_F}{\sqrt{p}}, \frac{\| \bm W_1^\top(t + 1) \bm U - \bm W_1^\top(t) \bm U \|_F}{\sqrt{p}} \right\} \leq \eta \cdot \sqrt{ \sigma_1^2\left( \bm G_1(t) \right) \cdot \frac{A^2(t)}{p} + \sigma_{K + 1}^2\left( \bm G_1(t) \right) }.
    \end{align*}   
\end{lemma}
\begin{proof}
    Define 
    \begin{align*}
        \mc{S} := \mc{R}\left( \bm R_{1, 2}(0) \right) \cap \mc{R}^\perp\left( \bm W_1^\top(0) \bm L_{1, 1}(0) / \epsilon \right),
    \end{align*}
    Both $\mc{R}\left( \bm R_{1, 2}(0) \right)$ and $\mc{R}^\perp\left( \bm W_1^\top(0) \bm L_{1, 1}(0) / \epsilon \right)$ are of dimension $d - K$, so $\dim\left( \mc{S} \right) = d - 2K := p$.
    Let $\bm V \in \mbb{R}^{d \times p}$ be an orthonormal basis for $\mc{S}$, and $\bm U := \bm W_1(0) \bm V / \epsilon \in \mbb{R}^{m \times p}$. From the definitions of $\bm U$ and $\bm V$, we trivially have
    \begin{align*}
        \bm W_1(0) \bm V - \epsilon \bm U = \bm 0_{m \times p} \quad \text{and} \quad \bm W_1^\top(0) \bm U - \epsilon \bm V = \bm 0_{d \times p}.
    \end{align*}
    Now consider an arbitrary $t$ for $t \geq 0$. From the GD update,
    \begin{align*}
        &\|\bm W_1(t + 1) \bm V - \bm W_1(t) \bm V\|_F = \eta \cdot \|\bm G_1(t) \bm V\|_F \quad \text{and} \\
        &\| \bm W_1^\top(t + 1) \bm U - \bm W_1^\top(t) \bm U\|_F = \eta \cdot \|\bm G_1^\top(t) \bm U\|_F.
    \end{align*}
    We first upper bound $\| \bm G_1(t) \bm V \|_F$, which then also applies for $\| \bm G_1^\top(t) \bm U \|_F$.
    
    \paragraph{Upper bounding $\| \bm G_1(t)\bm  V \|_F$ and $\| \bm G_1^\top(t) \bm U \|_F$.} Recall $\bm G_1(t) = \bm L_{1, 1}(t) \bm \Sigma_{1, 1}(t) \bm R^\top_{1, 1}(t) + \bm L_{1, 2}(t) \bm \Sigma_{1, 2}(t) \bm R^\top_{1, 2}(t)$. Therefore,
    \begin{align*}
        & \bm G_1(t) \bm V = \bm L_{1, 1}(t) \bm \Sigma_{1, 1}(t) \bm R^\top_{1, 1}(t) \bm V + \bm L_{1, 2}(t) \bm \Sigma_{1, 2}(t) \bm R^\top_{1, 2}(t) \bm V \\
        &\implies \| \bm G_1(t) \bm V \|_F = \sqrt{ \left\| \bm L_{1, 1}(t) \bm \Sigma_{1, 1}(t) \bm R^\top_{1, 1}(t) \bm V \right\|_F^2 + \left\| \bm L_{1, 2}(t) \bm \Sigma_{1, 2}(t) \bm R^\top_{1, 2}(t) \bm V \right\|_F^2 } \\
        &\quad\quad\quad\quad\quad\quad\quad\ \; \; = \sqrt{\left\| \bm \Sigma_{1, 1}(t) \bm R^\top_{1, 1}(t) \bm V \right\|_F^2 + \left\| \bm \Sigma_{1, 2}(t) \bm R^\top_{1, 2}(t) \bm V \right\|_F^2} \\
        &\quad\quad\quad\quad\quad\quad\quad\ \; \; \leq \sqrt{\sigma_1^2\left( \bm G_1(t) \right) \cdot \left\| \bm R^\top_{1, 1}(t) \bm V \right\|_F^2 + \sigma_{K + 1}^2 \left( \bm G_1(t) \right) \cdot \left\| \bm R^\top_{1, 2}(t) \bm V \right\|_F^2} \\
        &\quad\quad\quad\quad\quad\quad\quad\ \; \; \leq \sqrt{\sigma_1^2\left( \bm G_1(t) \right) \cdot \left\| \bm R^\top_{1, 1}(t) \bm R_{1, 2}(0) \right\|_F^2 + \sigma_{K + 1}^2 \left( \bm G_1(t) \right) \cdot \left\| \bm R^\top_{1, 2}(t) \bm V \right\|_F^2} \\
        &\quad\quad\quad\quad\quad\quad\quad\ \; \; \overset{(i)}{=} \sqrt{\sigma_1^2\left( \bm G_1(t) \right) \cdot \left\| \sin \Psi\left( \bm R_{1, 1}(t), \bm R_{1, 1}(0) \right) \right\|_2^2 + \sigma_{K + 1}^2 \left( \bm G_1(t) \right) \cdot \left\| \bm R^\top_{1, 2}(t) \bm V \right\|_F^2}.
    \end{align*}
    where $(i)$ is from \Cref{lem:aux_subspace_angles}. Next, define $\bm W_1(0) := \epsilon \bm Q \implies \bm Q := \bm W_1(0) / \epsilon$ and $\widetilde{\bm L}_{1, 2}(0)$ to be an orthonormal basis for $\mc{R}^{\perp}\left( \bm W_1^\top(0) \bm L_{1, 1}(0) / \epsilon \right) = \mc{R}^\perp\left( \bm Q^\top \bm L_{1, 1}(0) \right)$. Recall by definition, we have $\bm V \in \mc{R}\left( \widetilde{\bm L}_{1, 2}(0) \right)$. Therefore, 
    \begin{align*}
        &\bm G_1^\top(t) \bm U = \bm R_{1, 1}(t) \bm \Sigma^\top_{1, 1}(t) \bm L_{1, 1}^\top(t) \bm U + \bm R_{1, 2}(t) \bm \Sigma^\top_{1, 2}(t) \bm L_{1, 2}^\top(t) \bm U \\
        &\implies \| \bm G_1^\top(t) \bm U \|_F = \sqrt{ \| \bm R_{1, 1}(t) \bm \Sigma_{1, 1}^\top(t) \bm L_{1, 1}^\top(t) \bm U \|_F^2 + \| \bm R_{1, 2}(t) \bm \Sigma_{1, 2}^\top(t) \bm L_{1, 2}^\top(t) \bm U \|_F^2} \\
        &\quad\quad\quad\quad\quad\quad\quad\; \; \; = \sqrt{\| \bm \Sigma_{1, 1}^\top(t) \bm L_{1, 1}^\top(t) \bm U \|_F^2 + \| \bm \Sigma_{1, 2}^\top(t) \bm L_{1, 2}^\top(t) \bm U \|_F^2} \\
        &\quad\quad\quad\quad\quad\quad\quad\; \; \; \leq \sqrt{\sigma_1^2\left( \bm G_1(t) \right) \cdot \| \bm L_{1, 1}^\top(t) \bm U \|_F^2 + \sigma_{K+1}^2\left( \bm G_1(t) \right) \cdot \| \bm L_{1, 2}^\top(t) \bm U \|_F^2} \\
        &\quad\quad\quad\quad\quad\quad\quad\; \; \; = \sqrt{ \sigma_1^2\left( \bm G_1(t) \right) \cdot \| \bm L_{1, 1}^\top(t) \bm W_1(0) \bm V / \epsilon \|_F^2 + \sigma_{K+1}^2\left( \bm G_1(t) \right) \cdot \| \bm L_{1, 2}^\top(t) \bm W_1(0) \bm V / \epsilon \|_F^2 } \\
        &\quad\quad\quad\quad\quad\quad\quad\; \; \; := \sqrt{ \sigma_1^2\left( \bm G_1(t) \right) \cdot \| \bm L_{1, 1}^\top(t) \bm Q \bm V \|_F^2 + \sigma_{K+1}^2\left( \bm G_1(t) \right) \cdot \| \bm L_{1, 2}^\top(t) \bm Q \bm V \|_F^2 } \\
        &\quad\quad\quad\quad\quad\quad\quad\; \; \; \leq  \sqrt{ \sigma_1^2\left( \bm G_1(t) \right) \cdot \| \bm L_{1, 1}^\top(t) \bm Q \widetilde{\bm L}_{1, 2}(0) \|_F^2 + \sigma_{K+1}^2\left( \bm G_1(t) \right) \cdot \| \bm L_{1, 2}^\top(t) \bm Q \bm V \|_F^2 } \\
        &\quad\quad\quad\quad\quad\quad\quad\; \; \; \overset{(ii)}{\leq}  \sqrt{ \sigma_1^2\left( \bm G_1(t) \right) \cdot \| \bm L_{1, 1}^\top(t) \bm L_{1, 2}(0) \|_F^2 + \sigma_{K+1}^2\left( \bm G_1(t) \right) \cdot \| \bm L_{1, 2}^\top(t) \bm Q \bm V \|_F^2 } \\
        &\quad\quad\quad\quad\quad\quad\quad\; \; \; \overset{(iii)}{=} \sqrt{ \sigma_1^2\left( \bm G_1(t) \right) \cdot \| \sin \Psi \left( \bm L_{1, 1}(t), \bm L_{1, 1}(0) \right) \|_F^2 + \sigma_{K+1}^2\left( \bm G_1(t) \right) \cdot \| \bm L_{1, 2}^\top(t) \bm Q \bm V \|_F^2 },
        % &\quad\quad\quad\quad\quad\quad\quad\; \; \; \leq \sqrt{ \sigma_1^2\left( \bm G_1(t) \right) \cdot \| \bm L_{1, 1}^\top(t) \bm Q \bm Q^\top \bm L_{1, 1}(0) \|_F^2 + \sigma_{K+1}^2\left( \bm G_1(t) \right) \cdot \| \bm L_{1, 2}^\top(t) \bm Q \bm V \|_F^2 } \\
        % &\quad\quad\quad\quad\quad\quad\quad\; \; \; \overset{(ii)}{\leq} \sqrt{ \sigma_1^2\left( \bm G_1(t) \right) \cdot \| \bm L_{1, 1}^\top(t) \bm L_{1, 2}(0) \|_F^2 + \sigma_{K+1}^2\left( \bm G_1(t) \right) \cdot \| \bm L_{1, 2}^\top(t) \bm Q \bm V \|_F^2 } \\
        % &\quad\quad\quad\quad\quad\quad\quad\; \; \; \overset{(iii)}{=} \sqrt{\underbrace{\sigma_1^2\left( \bm G_1(t) \right) \cdot \| \sin \Psi\left( \bm L_{1, 1}(t), \bm L_{1, 1}(0) \right) \|_2^2}_{(c)} + \underbrace{\sigma_{K+1}^2\left( \bm G_1(t) \right) \cdot \| \bm L_{1, 2}^\top(t) \bm Q \bm V \|_F^2 }_{(d)}}
    \end{align*}
    where $(ii)$ is from \Cref{lem:subspace_align_after_proj}, and $(iii)$ is from \Cref{lem:aux_subspace_angles}. The definition of $A(t)$, combined with the fact that $\| \bm R_{1, 2}^\top(t) \bm V \|_F^2 \leq p$ and $\|\bm L_{1, 2}^\top(t) \bm Q \bm V \|_F^2 \leq p$, gives us
    \begin{align*}
        \max\left\{ \| \bm G_1(t) \bm V \|_F, \| \bm G_1^\top(t) \bm U \|_F \right\} \leq \sqrt{\sigma_1^2\left( \bm G_1(t) \right) \cdot A^2(t) + \sigma_{K + 1}^2 \left( \bm G_1(t) \right) \cdot p} := \rho(t) \cdot \sqrt{p}.
    \end{align*}
    This completes the proof. 
\end{proof}

\medskip 

\noindent We now provide a proof of \Cref{thm:smooth_main_result_main_body}. We re-state the (formal) Theorem statement here for convenience. 
\begin{theorem}[Formal version of \Cref{thm:smooth_main_result_main_body}]
\label{thm:smooth_main_result_appendix}
    Suppose Assumptions~\ref{assum:input_data} through \ref{assum:residual} hold. Define
    \begin{align*}
        &\gamma_L := \mu \cdot M \cdot \| \bm W_2 \|_1 + \beta^2 \cdot \sigma_1^2\left( \bm W_2 \right), \\
        &r(\epsilon) := \epsilon \cdot \phi'(0) \cdot \sigma_1^2(\bm W_2) \cdot \left( \phi'(0) + \frac{\mu}{2} \epsilon \right) + \mu \cdot \| \bm W_1(0) \bm X \|_{\max} \cdot \sigma_1(\bm W_2) \cdot \left( \beta \cdot \sigma_1(\bm W_2) \cdot \sqrt{d} \cdot \epsilon + \left\| \bm Y \right \|_F \right), \\&A(t) := \max\left\{ \left\| \sin\Psi\left( \bm L_{1, 1}(t), \bm L_{1, 1}(0) \right) \right\|_2, \left\| \sin\Psi\left( \bm R_{1, 1}(t), \bm R_{1, 1}(0) \right) \right\|_2 \right\}. 
        % \quad \text{and} \\
        % &\delta(t) := \frac{\gamma_L \cdot \beta \cdot \sigma_1\left( \bm W_2 \right) \cdot \sqrt{2 \ell(0)} \cdot \Theta(1) \cdot \left( 1 - \left(1 - \Theta(\eta) \right)^{\Theta(t)} \right)}{\phi'(0) \cdot \sigma_K\left( \bm W_2^\top \bm Y \bm X^\top \right) - r(\epsilon) - \sigma_{K + 1}\left( \bm G_1(t) \right)}.
    \end{align*}
    If $\bm W_1^\top(0) \bm W_1(0) = \epsilon^2 \bm I_d$ where $\epsilon$ satisfies $r(\epsilon) < \frac{\phi'(0) \cdot \sigma_K\left( \bm W_2^\top \bm Y \bm X^\top \right)}{2}$, and  $\eta \leq \frac{1}{\gamma_L}$,
    then there exist orthogonal matrices $\bm U \in \mbb{R}^{m \times m}$ and $\bm V \in \mbb{R}^{d \times d}$ such that, for all $t \geq 0$, $\bm W_1(t)$ admits the following decomposition:
    \begin{align*}
        \bm W_1(t) = \bm U \widetilde{\bm W}_1(t) \bm V^\top = \bm U \begin{bmatrix}
            \widetilde{\bm W}_{1, 1}(t) & \widetilde{\bm W}_{1, 2}(t) \\
            \widetilde{\bm W}_{1, 3}(t) & \widetilde{\bm W}_{1, 4}(t)
        \end{bmatrix} \bm V^\top,
    \end{align*}
    where $\widetilde{\bm W}_{1, 1} \in \mbb{R}^{(m - p) \times 2K}$, $\widetilde{\bm W}_{1, 2} \in \mbb{R}^{( m - p) \times p}$, $\widetilde{\bm W}_{1, 3} \in \mbb{R}^{p \times 2K}$, and $\widetilde{\bm W}_{1, 4} \in \mbb{R}^{p \times p}$, with
    \begin{align*}
        &\widetilde{\bm W}_{1, 2}(0) = \bm 0_{(m - p) \times p}, \quad  \widetilde{\bm W}_{1, 3}(0) = \bm 0_{p \times 2K}, \quad \frac{\| \widetilde{\bm W}_{1, 4}(0)\|_F}{\sqrt{p}} = \epsilon,
    \end{align*}
    and
    \begin{align*}
        &\frac{\| \widetilde{\bm W}_{1, i}(t + 1) - \widetilde{\bm W}_{1, i}(t) \|_F}{\sqrt{p}} \leq \eta \cdot \sqrt{ \sigma_1^2\left( \bm G_1(t) \right) \cdot \frac{A^2(t)}{p} + \sigma_{K + 1}^2\left( \bm G_1(t) \right) } 
    \end{align*}
    for all $i \in \{2, 3, 4\}$.
\end{theorem}
\begin{proof}
    Define $$\rho(t) := \sqrt{ \sigma_1^2\left( \bm G_1(t) \right) \cdot A^2(t) + p \cdot \sigma_{K + 1}^2\left( \bm G_1(t) \right) }.$$ From \Cref{lem:smooth_main_result_small_helper}, there exist $\bm V_2 \in \mbb{R}^{d \times p}$ and $\bm U_2 \in \mbb{R}^{m \times p}$ with orthonormal columns such that for any $t \geq 0$:
    \begin{align*}
        \bm W_1(t + 1) \bm V_2 = \bm W_1(t) \bm V_2 + \bm T_1(t), 
    \end{align*}
    where $\| \bm T_1(t) \|_F \leq \eta \cdot \rho(t)$, and $\bm W_1(0) \bm V_2 = \epsilon \bm U_2$. 
    Similarly, 
    \begin{align*}
        \bm W_1^\top(t + 1) \bm U_2 = \bm W_1^\top(t) \bm U_2 + \bm T_2(t),
    \end{align*}
    where again $\| \bm T_2(t) \|_F \leq \eta \cdot \rho(t)$ and $\bm W_1^\top(0) \bm U = \epsilon \bm V$. Complete $\bm V_2$ and $\bm U_2$ to be orthogonal matrices, defined as $\bm V := \begin{bmatrix}
        \bm V_1 & \bm V_2
    \end{bmatrix} \in \mbb{R}^{d \times d}$ and $\bm U := \begin{bmatrix}
        \bm U_1 & \bm U_2
    \end{bmatrix} \in \mbb{R}^{m \times m}$. Then, we have
    \begin{align*}
        &\bm U_1^\top \bm W_1(t + 1) \bm V_2 - \bm U_1^\top \bm W_1(t) \bm V_2 =  \bm U_1^\top \bm T_1(t) \\
        &\implies \| \bm U_1^\top \bm W_1(t + 1) \bm V_2 - \bm U_1^\top \bm W_1(t) \bm V_2\|_F = \| \bm U_1^\top \bm T_1(t) \|_F \leq \| \bm T_1(t) \|_F \leq \eta \cdot \rho(t),
    \end{align*}
    where
    \begin{align*}
        \bm U_1^\top \bm W_1(0) \bm V_2 = \epsilon \bm U_1^\top \bm U_2 = \bm 0_{2K \times p} \implies \|\bm U_1^\top \bm W_1(0) \bm V_2\|_F = 0.
    \end{align*}
    Similarly, we have
    \begin{align*}
        \| \bm V_1^\top \bm W_1(t + 1) \bm U_2 - \bm V_1^\top \bm W_1(t) \bm U_2 \|_F = \|\bm V_1^\top \bm T_2(t)\|_F \leq  \|\bm T_2(t)\|_F \leq \eta \cdot \rho(t),
    \end{align*}
    where
    \begin{align*}
        \bm V_1^\top \bm W_1^\top(0) \bm U_2 = \epsilon \bm V_1^\top \bm V_2 = \bm 0_{2K \times p} \implies \|\bm V_1^\top \bm W_1(0) \bm U_2\|_F = 0.
    \end{align*}
    Finally, we have
    \begin{align*}
        \| \bm U_2^\top \bm W_1(t + 1) \bm V_2 - \bm U_2^\top \bm W_1(t) \bm V_2 \|_F = \| \bm U_2^\top \bm T_1(t)\|_F \leq \eta \cdot \rho(t),
    \end{align*}
    where 
    \begin{align*}
        \bm U_2^\top \bm W_1(0) \bm V_2 = \epsilon \bm U_2^\top \bm U_2 = \epsilon \bm I_p \implies \| \bm U_2^\top \bm W_1(0) \bm V_2 \|_F = \epsilon \sqrt{p}.
    \end{align*}
    Putting everything together:
    \begin{align*}
        \bm U^\top \bm W_1(t) \bm V = \begin{bmatrix}
            \bm U_1^\top \\
            \bm U_2^\top
        \end{bmatrix} \bm W_1(t) \begin{bmatrix}
            \bm V_1 & \bm V_2
        \end{bmatrix} = \begin{bmatrix}
            \bm U_1^\top \bm W_1(t) \bm V_1 & \bm U_1^\top \bm W_1(t) \bm V_2 \\
            \bm U_2^\top \bm W_1(t) \bm V_1 & \bm U_2^\top \bm W_1(t) \bm V_2
        \end{bmatrix} := \begin{bmatrix}
            \widetilde{\bm W}_{1, 1}(t) & \widetilde{\bm W}_{1, 2}(t) \\
            \widetilde{\bm W}_{1, 3}(t) & \widetilde{\bm W}_{1, 4}(t)
        \end{bmatrix},
    \end{align*}
    where 
    \begin{align*}
        &\| \widetilde{\bm W}_{1, 2}(0) \|_F = \| \bm U_1^\top \bm W_1(0) \bm V_2\|_F = 0 \\
        &\| \widetilde{\bm W}_{1, 3}(0) \|_F= \| \bm U_2^\top \bm W_1(0) \bm V_1\|_F = 0, \quad \text{and} \\
        &\| \widetilde{\bm W}_{1, 4}(0) \|_F = \| \bm U_2^\top \bm W_1(0) \bm V_2\|_F = \epsilon \sqrt{p}, 
    \end{align*}
    as well as
    \begin{align*}
        &\| \widetilde{\bm W}_{1, 2}(t + 1) - \widetilde{\bm W}_{1, 2}(t) \|_F = \| \bm U_1^\top \bm W_1(t + 1) \bm V_2 - \bm U_1^\top \bm W_1(t) \bm V_2 \|_F \leq \eta \cdot \rho(t), \\
        &\| \widetilde{\bm W}_{1, 3}(t + 1) - \widetilde{\bm W}_{1, 3}(t) \|_F = \| \bm U_2^\top \bm W_1(t + 1) \bm V_1 - \bm U_2^\top \bm W_1(t) \bm V_1 \|_F \leq \eta \cdot \rho(t), \quad \text{and} \\
        &\| \widetilde{\bm W}_{1, 4}(t + 1) - \widetilde{\bm W}_{1, 4}(t) \|_F = \| \bm U_2^\top \bm W_1(t + 1) \bm V_2 - \bm U_2^\top \bm W_1(t) \bm V_2 \|_F \leq \eta \cdot \rho(t).
    \end{align*}
    Dividing everything by $\sqrt{p}$ completes the proof. 
\end{proof}

\subsection{Proof of \Cref{coro:quantify_thm_bound}}
\label[appendix]{ssec:coro_proof}
Here, we provide a proof of \Cref{coro:quantify_thm_bound}. We re-state the result here for convenience. 

\boundcorollary*

\smallskip

\begin{proof}
    Define 
    \begin{align*}
        \rho(t) := \sqrt{ \sigma_1^2\left( \bm G_1(t) \right) \cdot A^2(t) + p \cdot \sigma_{K + 1}^2\left( \bm G_1(t) \right) }
    \end{align*}
    and 
    \begin{align*}
        \delta(t) := \frac{\gamma_L \cdot \beta \cdot \sigma_1\left( \bm W_2 \right) \cdot \sqrt{2 \ell(0)} \cdot \Theta(1) \cdot \left( 1 - \left(1 - \Theta(\eta) \right)^{\Theta(t)} \right)}{\phi'(0) \cdot \sigma_K\left( \bm W_2^\top \bm Y \bm X^\top \right) - r(\epsilon) - \sigma_{K + 1}\left( \bm G_1(t) \right)}.
    \end{align*}
    We now upper bound $\rho(t)$. First, note that at $t = 0$, we have $A^2(t) = 0$, and $p \cdot \sigma_{K + 1}^2\left( \bm G_1(0) \right) = p \cdot r^2(\epsilon)$ via \Cref{lem:smooth_grad_init_svals}. Thus, we focus on $t \geq 1$.
    Recall by definition, $\bm V \in \mc{R}\left( \bm R_{1, 2}(0) \right)$. From \Cref{lem:smooth_grad_subspace_change}, we have
    \begin{align*}
        &A(t) \leq \delta(t) \implies \sigma_1^2\left( \bm G_1(t) \right) \cdot A^2(t) \leq \sigma_1^2\left( \bm G_1(t) \right) \cdot \delta^2(t).
    \end{align*}
    Then, we have
    \begin{align*}
        &\frac{\sigma_1\left( \bm G_1(t) \right)}{\sigma_1\left( \bm G_1(0) \right)} \leq G_2 \cdot \frac{\| \bm G_1(t) \|_F}{\| \bm G_1(0) \|_F} \leq G_2 \cdot G_1 \cdot \left(1 - \Theta(\eta) \right)^{\Theta(t)} \\
        &\overset{(i)}{\implies} \sigma_1(\bm G_1(t)) \leq \bar{G} \cdot \left(1 - \Theta(\eta) \right)^{\Theta(t)} \cdot \sigma_1\left( \bm G_1(0) \right) \overset{(ii)}{\leq} \bar{G} \cdot \left(1 - \Theta(\eta) \right)^{\Theta(t)} \cdot \delta(t) \cdot \left( \phi'(0) \cdot \sigma_1\left( \bm W_2^\top \bm Y \bm X^\top\right) + r(\epsilon) \right),
    \end{align*}
    where $\bar{G} := G_2 \cdot G_1$, $(i)$ is from \Cref{assum:gradient}, and $(ii)$ is from \Cref{lem:smooth_grad_init_svals}. 
    In summary, 
    \begin{align*}
        \rho(t) \leq \begin{cases}
            \hfil r(\epsilon) \cdot \sqrt{p} & t = 0 \\
            \sqrt{\bar{G}^2 \cdot \left(1 - \Theta(\eta) \right)^{\Theta(t)} \cdot \delta^2(t) \cdot \left(\phi'(0) \cdot \sigma_1\left( \bm W_2^\top \bm Y \bm X^\top \right) + r(\epsilon) \right)^2 + \sigma_{K + 1}^2\left( \bm G_1(t) \right) \cdot p} & t \geq 1
        \end{cases}.
    \end{align*}
    Dividing both sides by $\sqrt{p}$ completes the proof. 
\end{proof}
    \section{Proof of \Cref{prop:nonsmooth_result}}
\label[appendix]{sec:nonsmooth_proofs}
In this section, we provide a proof of \Cref{prop:nonsmooth_result}. We use the same notation as in \Cref{sec:smooth_proofs}. We re-state \Cref{prop:nonsmooth_result} below for convenience.

\begin{proposition}[\Cref{prop:nonsmooth_result} re-stated]
    Suppose $d = N$, $\bm W_1(0)_{ij} \overset{iid}{\sim} \mc{N}\left(0, \frac{\epsilon^2}{m} \right)$, and $\phi = \relu$. With probability at least $1 - \delta$ w.r.t. the randomness in $\bm W_1(0)$, 
    \begin{align*}
        \sigma_{d - K}\left( \bm G_1(0) \right) \geq  \sqrt{\frac{\lambda_{\min} \left( \bm V^\top \bm D \bm V \right)}{4}  - \left( \frac{R'}{6} \cdot \log\left( \frac{2(d - K)}{\delta} \right) + \sqrt{2 \cdot \log\left( \frac{2(d - K)}{\delta} \right) \cdot \frac{R' \cdot \lambda_{\max}\left( \bm V^\top \bm D \bm V \right) }{16}} \right)},
    \end{align*}
    where $\bm V$ is an orthonormal basis for $\mc{N}\left( \bm W_2^\top \bm \Delta_2(0) \right)$, $\bm D := \mathrm{diag}\left( \left\| \left( \bm W_2^\top \bm \Delta_2(0) \right)_{:, 1} \right\|_2^2, \dots, \left\| \left( \bm W_2^\top \bm \Delta_2(0) \right)_{:, N} \right\|_2^2 \right)$, and $R' := \max\limits_{i \in [m]} \left\| \left( \bm W_2^\top \bm \Delta_2(0) \right)_{i, :} \right\|_2^2 $.
\end{proposition}
\begin{proof}
    Recall that 
    \begin{align*}
        \bm G_1(0) = \big( \bm W_2^\top \bm \Delta_2(0) \odot \phi'\left( \bm W_1(0) \bm X \right) \big) \bm X^\top.
    \end{align*}
    Since $d = N$ and $\bm X \bm X^\top = \bm I_d$, $\bm X$ is exactly an orthogonal matrix. Therefore, we have
    \begin{align*}
        \sigma_{d - K}\left( \bm G_1(0) \right) = \sigma_{d - K}\left(\bm W_2^\top \bm \Delta_2(0) \odot \phi'\left( \bm W_1(0) \bm X \right) \right), 
    \end{align*}
    so we aim to lower bound $\sigma_{K + 1}\left(\bm W_2^\top \bm \Delta_2(0) \odot \phi'\left( \bm W_1(0) \bm X \right) \right)$. First, note that $\bm W_2^\top \bm \Delta_2(0)$ is rank-$K$. Next, define $\bm M := \phi'\left( \bm W_1(0) \bm X \right)$. Notice that $\bm M$ contains iid Bernoulli entries with probability parameter $q = 0.5$, as $\left( \bm W_1(0) \bm X \right)_{ij} \overset{iid}{\sim} \mc{N}\left(0,  \frac{\epsilon^2}{m} \right)$, and $\phi'\left( \bm W_1(0) \bm X \right)_{ij} = \mbb{I}\left[ \left( \bm W_1(0) \bm X \right)_{ij} > 0\right]$, where $\mbb{I}$ denotes the indicator function. Finally, note $\bm M = q \bm 1_m \bm 1_N^\top + \left( \bm M - q \bm 1_m \bm 1_N^\top \right)$. Combining everything yields
    \begin{align*}
        \bm W_2^\top \bm \Delta_2(0) \odot \bm M = \underbrace{q \bm W_2^\top \bm \Delta_2(0)}_{:= \bm S} + \underbrace{\bm W_2^\top \bm \Delta_2(0) \odot \left( \bm M - q \bm 1_m \bm 1_N^\top \right)}_{:= \bm N}.
    \end{align*}
    Therefore,
    \begin{align*}
        &\sigma_{d - K}\left(\bm W_2^\top \bm \Delta_2(0) \odot \phi'\left( \bm W_1(0) \bm X \right) \right) = \sigma_{d - K}\left(\bm S + \bm N \right) \\
        &\overset{(i)}{=} \max\limits_{\mc{T} \subset \mbb{R}^N: \dim\left( \mc{T} \right) = d - K} \min\limits_{\bm x \in \mc{T}, \| \bm x \|_2 = 1} \left \| \left( \bm S + \bm N \right) \bm x \right\|_2 \\
        &\overset{(ii)}{\geq} \min\limits_{\bm x \in \mc{N}\left( \bm S \right), \| \bm x \|_2 = 1} \left \| \left( \bm S + \bm N \right) \bm x \right\|_2 = \min\limits_{\bm x \in \mc{N}\left( \bm S \right), \| \bm x \|_2 = 1} \left \| \bm N \bm x \right\|_2.
    \end{align*}
    where $(i)$ is from the Courant-Fischer Min-Max Theorem for singular values, i.e., \citet[Theorem 4.2.2]{horn2012matrix} applied to $\left( \bm S + \bm N \right)^\top \left( \bm S + \bm N \right)$, and $(ii)$ is using the fact that $\dim\left( \mc{N}\left( \bm S \right) \right) = d - K$. Let $\bm V \in \mbb{R}^{N \times (d - K)}$ be an orthonormal basis for $\mc{N}\left( \bm S \right)$. Then,
    \begin{align*}
        \min\limits_{\bm x \in \mc{N}\left( \bm S \right), \| \bm x \|_2 = 1} \left \| \bm N \bm x \right\|_2 = \min\limits_{\bm z \in \mbb{R}^{d - K}: \| \bm z \|_2 = 1} \left\| \bm N \bm V \bm z \right\|_2 = \sigma_{\min}\left( \bm N \bm V \right).
    \end{align*}
    Now, it suffices to lower bound $\sigma_{\min}\left( \bm N \bm V \right)$. Notice $\bm N \bm V \in \mbb{R}^{m \times (d - K)}$ with $m \geq N > d - K$. Therefore, $\sigma_{\min}^2\left( \bm N \bm V \right) = \lambda_{\min} \left( \bm V^\top \bm N^\top \bm N \bm V \right)$, so we aim to lower bound $\lambda_{\min} \left( \bm V^\top \bm N^\top \bm N \bm V \right)$. Let $\widetilde{\bm M} := \bm M - q \bm 1_m \bm 1_N^\top$ and $\bm E = \bm W_2^\top \bm \Delta_2(0)$, so $\bm N = \bm E \odot \widetilde{\bm M}$. Then,
    \begin{align*}
        \bm V^\top \bm N^\top \bm N \bm V = \mbb{E}_{\bm N}\left[ \bm V^\top \bm N^\top \bm N \bm V \right] + \left( \bm V^\top \bm N^\top \bm N \bm V - \mbb{E}_{\bm N}\left[  \bm V^\top \bm N^\top \bm N \bm V \right] \right), 
    \end{align*}
    and so by Weyl's inequality \citep{weyl1949inequalities},
    \begin{align*}
        &\lambda_{\min}\left( \bm V^\top \bm N^\top \bm N \bm V \right) = \lambda_{\min} \Big( \mbb{E}_{\bm N}\left[ \bm V^\top \bm N^\top \bm N \bm V \right] + \left( \bm V^\top \bm N^\top \bm N \bm V - \mbb{E}_{\bm N}\left[  \bm V^\top \bm N^\top \bm N \bm V \right] \right) \Big) \\
        &\geq \lambda_{\min}\left( \mbb{E}_{\bm N}\left[ \bm V^\top \bm N^\top \bm N \bm V \right] \right) + \lambda_{\min} \left( \bm V^\top \bm N^\top \bm N \bm V - \mbb{E}_{\bm N}\left[  \bm V^\top \bm N^\top \bm N \bm V \right] \right) \\
        &\geq \underbrace{\lambda_{\min}\left( \mbb{E}_{\bm N}\left[ \bm V^\top \bm N^\top \bm N \bm V \right] \right)}_{(a)} - \underbrace{\sigma_1\left( \bm V^\top \bm N^\top \bm N \bm V - \mbb{E}_{\bm N}\left[  \bm V^\top \bm N^\top \bm N \bm V \right] \right)}_{(b)}
    \end{align*}
    We analyze $(a)$ and $(b)$ individually. To do so, we first determine $\mbb{E}_{\bm N}\left[\bm V^\top \bm N^\top \bm N \bm V\right]$.
    
    \paragraph{Expectation of $\bm V^\top \bm N^\top \bm N \bm V$.} First, since $\bm V$ is deterministic w.r.t. $\bm N$, we have
    \begin{align*}
        &\mbb{E}_{\bm N}\left[ \bm V^\top \bm N^\top \bm N \bm V \right] = \bm V^\top \mbb{E}_{\widetilde{\bm M}}\left[ \left( \bm E \odot \widetilde{\bm M} \right)^\top \left( \bm E \odot \widetilde{\bm M} \right) \right] \bm V.
    \end{align*}
    Therefore,
    \begin{align*}
        &\bigg( \left( \bm E \odot \widetilde{\bm M} \right)^\top \left( \bm E \odot \widetilde{\bm M} \right) \bigg)_{ij} = \left( \bm E \odot \widetilde{\bm M} \right)_{:, i}^\top \left( \bm E \odot \widetilde{\bm M} \right)_{:, j} = \sum\limits_{k=1}^m \left( \bm E \odot \widetilde{\bm M} \right)_{ki} \left( \bm E \odot \widetilde{\bm M} \right)_{kj} = \sum\limits_{k=1}^m \bm E_{ki} \bm E_{kj} \cdot \widetilde{\bm M}_{ki} \widetilde{\bm M}_{kj} \\
        &\implies \mbb{E}_{\widetilde{\bm M}}\left[ \bigg( \left( \bm E \odot \widetilde{\bm M} \right)^\top \left( \bm E \odot \widetilde{\bm M} \right) \bigg)_{ij} \right] = \sum\limits_{k=1}^m \bm E_{ki} \bm E_{kj} \cdot \mbb{E}\left[ \widetilde{\bm M}_{ki} \widetilde{\bm M}_{kj} \right] = \begin{cases}
            0 & i \neq j \\
            q^2 \cdot \| \bm E_{:, i} \|_2^2 & i = j
        \end{cases} \\
        &\implies \mbb{E}_{\widetilde{\bm M}}\left[ \left( \bm E \odot \widetilde{\bm M} \right)^\top \left( \bm E \odot \widetilde{\bm M} \right) \right] = q^2 \cdot \mathrm{diag}\left( \| \bm E_{:, 1} \|_2^2, \dots, \| \bm E_{:, N} \|_2^2  \right) := q^2 \cdot \bm D \\
        &\implies \mbb{E}_{\bm N}\left[ \bm V^\top \bm N^\top \bm N \bm V \right] = q^2 \cdot \bm V^\top \bm D \bm V,
    \end{align*}
    where $\bm A_{:, i}$ denotes the $i^{th}$ column in the matrix $\bm A$. Thus, we have \begin{align*}
        (a) = \lambda_{\min}\left( \mbb{E}_{\bm N}\left[ \bm V^\top \bm N^\top \bm N \bm V\right] \right) = q^2 \cdot \lambda_{\min}\left( \bm V^\top \bm D \bm V \right),
    \end{align*}
    so now we aim to lower bound $(b)$.

    \paragraph{Bounding $(b)$.} We aim to use the Matrix Bernstein inequality, i.e., \citet[Theorem 6.1.1]{tropp2015introduction} to bound $(b)$.
    First, define $\widebar{\bm N} := \bm V^\top \bm N^\top \bm N \bm V - \mbb{E}_{\bm N}\left[ \bm V^\top \bm N^\top \bm N \bm V \right]$. Note we can write $\widebar{\bm N}$ as a sum of independent, zero-mean, symmetric random matrices:
    \begin{align*}
        &\bm V^\top \bm N^\top \bm N \bm V - \mbb{E}_{\bm N}\left[ \bm V^\top \bm N^\top \bm N \bm V \right] = \sum\limits_{i=1}^m \bm V^\top \left( \bm E \odot \widetilde{\bm M} \right)_{i, :} \left( \bm E \odot \widetilde{\bm M} \right)_{i, :}^\top \bm V - \mbb{E}\left[ \bm V^\top \left( \bm E \odot \widetilde{\bm M} \right)_{i, :} \left( \bm E \odot \widetilde{\bm M} \right)_{i, :}^\top \bm V \right] 
    \end{align*}
    where $\bm A_{i, :}$ denotes the $i^{th}$ row in $\bm A$, but written as a column vector. The terms in the sum are independent since the rows of $\widetilde{\bm M}$ are independent. Next, define $\bm n_i := \bm V^\top \left( \bm E \odot \widetilde{\bm M} \right)_{i, :}$ and $\widebar{\bm N}_i = \bm n_i \bm n_i^\top - \mbb{E}\left[ \bm n_i \bm n_i^\top \right]$. Then, we have $\widebar{\bm N} = \sum\limits_{i=1}^m \left( \bm n_i \bm n_i^\top - \mbb{E}\left[ \bm n_i \bm n_i^\top \right] \right) = \sum\limits_{i=1}^m \widebar{\bm N}_i$. To leverage the Matrix Bernstein inequality, we must bound $\sigma_1\left( \widebar{\bm N}_i \right)$ and $\nu\left( \widebar{\bm N} \right) := \sigma_1\left( \mbb{E}\left[ \widebar{\bm N}^2 \right] \right)$.

    \paragraph{Bounding $\sigma_1\left( \widebar{\bm N}_i \right)$.} We first bound $\sigma_1\left( \widebar{\bm N}_i \right)$. First, recall $\widebar{\bm N}_i = \bm n_i \bm n_i^\top - \mbb{E}\left[ \bm n_i \bm n_i^\top \right]$, and notice it is the difference of two symmetric matrices, and thus is itself symmetric. Therefore,
    \begin{align*}
        \sigma_1\left( \widebar{\bm N}_i \right) = \max\left\{ \lambda_{\max}\left( \bm n_i \bm n_i^\top - \mbb{E}\left[ \bm n_i \bm n_i^\top \right] \right), \left| \lambda_{\min}\left( \bm n_i \bm n_i^\top - \mbb{E}\left[ \bm n_i \bm n_i^\top \right] \right) \right| \right\}.
    \end{align*}
    Next, $\bm n_i \bm n_i^\top \succeq 0$ with probability $1$, which implies $\mbb{E}\left[ \bm n_i \bm n_i^\top \right] \succeq 0$. Therefore,
    \begin{align*}
        \lambda_{\max}\left( \bm n_i \bm n_i^\top - \mbb{E}\left[ \bm n_i \bm n_i^\top \right] \right) \leq \lambda_{\max}\left( \bm n_i \bm n_i^\top \right) = \| \bm n_i \|_2^2,
    \end{align*}
    and
    \begin{align*}
        &\lambda_{\min}\left( \bm n_i \bm n_i^\top - \mbb{E}\left[ \bm n_i \bm n_i^\top \right] \right) \geq \lambda_{\min}\left( -\mbb{E}\left[ \bm n_i \bm n_i^\top \right] \right) = -\lambda_{\max}\left( \mbb{E}\left[ \bm n_i \bm n_i^\top \right] \right) \\
        &\implies \left| \lambda_{\min}\left( \bm n_i \bm n_i^\top - \mbb{E}\left[ \bm n_i \bm n_i^\top \right] \right) \right| \leq \lambda_{\max}\left( \mbb{E}\left[ \bm n_i \bm n_i^\top \right] \right) \overset{(i)}{\leq} \mbb{E}\left[ \lambda_{\max}\left( \bm n_i \bm n_i^\top \right) \right] = \mbb{E}\left[ \| \bm n_i \|_2^2 \right],
    \end{align*}
    where $(i)$ is from Jensen's inequality. Thus, we focus on analyzing $\| \bm n_i \|_2^2$. Note
    \begin{align*}
        \bm n_i = \bm V^\top \left( \bm E \odot \widetilde{\bm M} \right)_{i, :} = \bm V^\top \mathrm{diag}\left( \bm E_{i, :} \right) \widetilde{\bm M}_{i, :},
    \end{align*}
    where $\mathrm{diag}(\bm a)$ is a diagonal matrix whose entries are the elements of the vector $\bm a$. Therefore,
    \begin{align*}
        &\| \bm n_i \|_2^2 = \bm n_i^\top \bm n_i = \widetilde{\bm M}_{i, :}^\top \mathrm{diag}\left( \bm E_{i, :} \right)^\top \bm V \bm V^\top \mathrm{diag}\left( \bm E_{i, :} \right) \widetilde{\bm M}_{i, :} \overset{(ii)}{\leq} \widetilde{\bm M}_{i, :}^\top \mathrm{diag}\left( \bm E_{i, :} \right)^\top \mathrm{diag}\left( \bm E_{i, :} \right) \widetilde{\bm M}_{i, :} = q^2 \| \bm E_{i, :} \|_2^2 \\
        &\implies \| \bm n_i \|_2^2 \leq q^2 \cdot \max_{i \in [m]} \| \bm E_{i, :} \|_2^2 := q^2 \cdot R'
    \end{align*}
    with probability $1$, where $(ii)$ is because $\bm V \bm V^\top \preceq \bm I_N$.
    Therefore, for all $i \in [m]$,
    \begin{align*}
        &\sigma_1\left( \widebar{\bm N}_i \right) = \max\left\{ \lambda_{\max}\left( \bm n_i \bm n_i^\top - \mbb{E}\left[ \bm n_i \bm n_i^\top \right] \right), \left| \lambda_{\min}\left( \bm n_i \bm n_i^\top - \mbb{E}\left[ \bm n_i \bm n_i^\top \right] \right) \right| \right\} \leq \max\left\{ \| \bm n_i \|_2^2, \mbb{E}\left[ \| \bm n_i \|_2^2 \right]  \right\} \leq q^2 \cdot  R'
    \end{align*}
    with probability $1$.

    \paragraph{Bounding $\nu\left( \widebar{\bm N} \right)$.} Next, we bound $\nu\left( \widebar{\bm N} \right) = \sigma_1\left( \mbb{E}\left[ \widebar{\bm N}^2 \right]  \right)$. Since each $\widebar{\bm N}_i$ is zero mean, we have
    \begin{align*}
       \sigma_1\left( \mbb{E}\left[ \widebar{\bm N}^2 \right]  \right) = \sigma_1\left( \sum\limits_{i=1}^m \mbb{E}\left[ \widebar{\bm N}_i^2 \right] \right).
    \end{align*}
    For an arbitrary $i \in [m]$:
    \begin{align*}
        &\mbb{E}\left[ \widebar{\bm N}_i^2 \right] = \mbb{E}\left[ \left( \bm n_i \bm n_i^\top - \mbb{E}\left[ \bm n_i \bm n_i \right]  \right)^2 \right] = \mbb{E}\left[ \left( \bm n_i \bm n_i^\top \right)^2 \right] - \mbb{E}\left[ \bm n_i \bm n_i^\top \right]^2 \\
        &= \mbb{E}\left[ \| \bm n_i \|_2^2 \cdot \bm n_i \bm n_i^\top \right] - \mbb{E}\left[ \bm n_i \bm n_i^\top \right]^2 \overset{(ii)}{\preceq} q^2 \cdot R' \cdot \mbb{E}\left[ \bm n_i \bm n_i^\top \right] - \mbb{E}\left[ \bm n_i \bm n_i^\top \right]^2 \preceq q^2 \cdot R' \cdot \mbb{E}\left[ \bm n_i \bm n_i^\top \right]
    \end{align*}
    where $(ii)$ is from the fact that $\| \bm n_i \|_2^2 \leq q^2 \cdot R'$ with probability $1$. Therefore,
    \begin{align*}
        \sum\limits_{i=1}^m \mbb{E}\left[ \widebar{\bm N}_i^2 \right] \preceq q^2 \cdot R' \cdot \sum\limits_{i=1}^m \mbb{E}\left[ \bm n_i \bm n_i^\top \right] \overset{(iii)}{=} q^2 \cdot R' \cdot \mbb{E}_{\bm N}\left[ \bm V^\top \bm N^\top \bm N \bm V \right] = q^4 \cdot R' \cdot \bm V^\top \bm D \bm V,
    \end{align*}
    where $(iii)$ is because $\bm n_i := \bm V^\top \left( \bm E \odot \widetilde{\bm M} \right)_{i, :} = \bm V^\top \bm N_{i, :}$, and so $\sum\limits_{i=1}^m \bm n_i \bm n_i^\top = \bm V^\top \bm N^\top \bm N \bm V$. As a result,
    \begin{align*}
        \sigma_1\left( \mbb{E}\left[ \widebar{\bm N}^2 \right] \right) = \sigma_1\left( \sum\limits_{i=1}^m \mbb{E}\left[ \widebar{\bm N}_i^2 \right] \right) \leq q^4 \cdot R' \cdot \lambda_{\max}\left( \bm V^\top \bm D \bm V \right)
    \end{align*}

    \paragraph{Matrix Bernstein inequality.} From \citet[Theorem 6.1.1]{tropp2015introduction}, for all $\tau \geq 0$:
    \begin{align*}
        &\mbb{P}\left( \sigma_1\left( \widebar{\bm N} \right) \geq \tau \right) \leq 2 (d - K) \cdot \exp\left( \frac{-\tau^2/2}{\nu\left( \widebar{\bm N} \right) + \max\limits_i \sigma_1\left( \widebar{\bm N}_i \right) \cdot \tau / 3} \right) \\
        &\leq 2(d - K) \cdot \exp\left( \frac{-\tau^2 / 2}{q^4 \cdot R' \cdot \lambda_{\max} \left(\bm V^\top \bm D \bm V \right) + q^2 \cdot R' \cdot \tau / 3} \right) \leq \delta,
    \end{align*}
    which implies
    \begin{align*}
        \tau \geq &\frac{q^2 \cdot R'}{3} \cdot \log\left( \frac{2(d - K)}{\delta} \right) + \sqrt{\left( \frac{q^2 \cdot R'}{3} \cdot \log\left( \frac{2(d - K)}{\delta} \right) \right)^2 + 2 \cdot \log\left( \frac{2(d - K)}{\delta} \right) \cdot \left( q^4 \cdot R' \cdot \lambda_{\max}\left( \bm V^\top \bm D \bm V \right) \right)}.
    \end{align*}
    Using $\sqrt{a + b} \leq \sqrt{a} + \sqrt{b}$, a sufficient condition is 
    \begin{align*}
        \tau \geq \frac{2q^2 \cdot R'}{3} \cdot \log\left( \frac{2(d - K)}{\delta} \right) + \sqrt{2 \cdot \log\left( \frac{2(d - K)}{\delta} \right) \cdot \left( q^4 \cdot R' \cdot \lambda_{\max}\left( \bm V^\top \bm D \bm V \right) \right)}.
    \end{align*}
    Therefore, with probability at least $1 - \delta$,
    \begin{align*}
        \sigma_1\left( \widebar{\bm N} \right) \leq \frac{2q^2 \cdot R'}{3} \cdot \log\left( \frac{2(d - K)}{\delta} \right) + \sqrt{2 \cdot \log\left( \frac{2(d - K)}{\delta} \right) \cdot \left( q^4 \cdot R' \cdot \lambda_{\max}\left( \bm V^\top \bm D \bm V \right) \right)}.
    \end{align*}

    \paragraph{Final result.} Putting everything together, for any $\delta \in (0, 1)$, with probability at least $1 - \delta$,
    \begin{align*}
        \sigma_{d - K}\left( \bm G_1(0) \right) \geq \sqrt{q^2 \cdot \lambda_{\min} \left( \bm V^\top \bm D \bm V \right) - \left( \frac{2q^2 \cdot R'}{3} \cdot \log\left( \frac{2(d - K)}{\delta} \right) + \sqrt{2 \cdot \log\left( \frac{2(d - K)}{\delta} \right) \cdot \left( q^4 \cdot R' \cdot \lambda_{\max}\left( \bm V^\top \bm D \bm V \right) \right)} \right) }.
    \end{align*}
    Substituting $q = 0.5$ completes the proof.

\end{proof}

    \section{Low-Rank MLP Ablations}
In this section, we provide ablation studies on the initialization of the $\widetilde{\bm U}$ and $\widetilde{\bm V}$ factors, as well as the width parameter $r$. We ran all experiments in this section using \texttt{PyTorch} on an NVIDIA A40 GPU.

% \subsection{Additional Details for FashionMNIST Experiments}
% \label[appendix]{ssec:additional_fmnist_details}
% Here, we provide additional experimental details for \Cref{sssec:fashion_mnist}. We considered $L = 4$ layer networks with $\gelu$ and $\silu$ activations, setting $m = d = 784$ and $r = 2K$. We initialized each $\bm W_l$ and $\widetilde{\bm W}_l$ as $\epsilon$-scaled orthogonal matrices, with $\epsilon = 0.1$, and trained all networks on squared-error loss using full-batch GD. We considered initializing $\widetilde{\bm U}$ and $\widetilde{\bm V}$ in two different ways: 1) the $\mc{S}_{big}$ initialization scheme, and 2) as random semi-orthogonal matrices. We trained all networks on all $N = 5 \times 10^4$ training images for $T = 1500$ epochs using full-batch GD on (total) squared-error loss, which aligns with the training algorithm in our theoretical setting. We set and $\eta = 10^{-5}$ and used a cosine annealing scheduler. 

\subsection{Subspace Initialization}
\label[appendix]{ssec:angle_ablation}

\begin{figure}[t]
    \centering
    \includegraphics[width=0.8\linewidth]{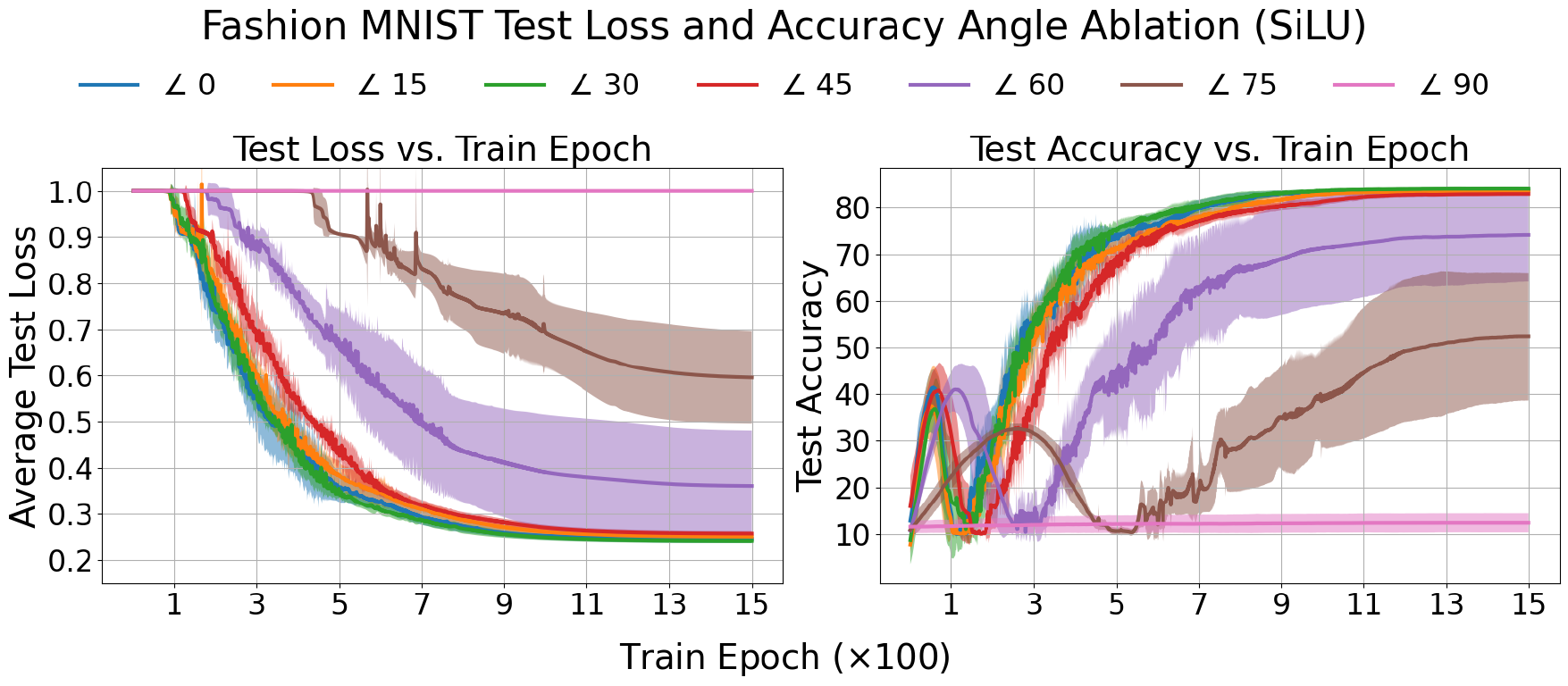}
    \caption{For all low-rank MLPs with $\widetilde{\bm U}$ and $\widetilde{\bm V}$ initialized close to ($\psi \leq 45$ degrees) $\widetilde{\bm U}_{big}$ and $\widetilde{\bm V}_{big}$, GD reaches similar solutions. Otherwise, GD gets stuck in progressively worse local minima as $\psi$ increases.}
    \label{fig:angle_ablation}
\end{figure}

From \Cref{fig:fashion_mnist,fig:cifar10} in \Cref{sssec:fashion_mnist,sssec:cifar10} respectively, we observe that if we train the low-rank MLP in \eqref{eq:low_rank_mlp} after initializing $\widetilde{\bm U}$ and $\widetilde{\bm V}$ as random semi-orthogonal matrices (green), then (S)GD gets stuck in a much worse local minimum compared to the other MLPs (blue and orange). In this section, we ablate the initialization of $\widetilde{\bm U}$ and $\widetilde{\bm V}$. 

\paragraph{Experimental details.} We repeated the experiments as described in \Cref{sssec:fashion_mnist}. We used the exact same training algorithm (full batch GD), hyperparameters, and loss function (total squared-error loss). However, we only used $\phi = \silu$, and we varied the initialization of $\widetilde{\bm U}$ and $\widetilde{\bm V}$ as follows. First, we define $\widetilde{\bm U}_{big} \in \mbb{R}^{d \times r}$ and $\widetilde{\bm V}_{big} \in \mbb{R}^{d \times r}$ to be the result of the $\mc{S}_{big}$ initialization scheme described in \Cref{sec:low_rank_param}, where we set $r = 2K = 20$. Next, we define $\widetilde{\bm U}_{big}^\perp, \widetilde{\bm V}_{big}^\perp \in \mbb{R}^{d \times r}$ to be completely orthogonal to $\widetilde{\bm U}_{big}$ and $\widetilde{\bm V}_{big}$, respectively. Finally, we initialized $\widetilde{\bm U}$ and $\widetilde{\bm V}$ to be some angle $\psi$ degrees away from $\widetilde{\bm U}_{big}$ and $\widetilde{\bm V}_{big}$ via
\begin{align}
\label{eq:U_tilde_V_tilde_angle_init}
    \widetilde{\bm U} = \cos\left( \psi\right) \cdot \widetilde{\bm U}_{big} + \sin \left( \psi \right) \cdot \widetilde{\bm U}_{big}^\perp \quad \text{and} \quad \widetilde{\bm V} = \cos(\psi) \cdot \widetilde{\bm V}_{big} + \sin(\psi) \cdot \widetilde{\bm V}_{big}^\perp.
\end{align}
In our experiments, we swept through $\psi \in \{0, 15, 30, 45, 60, 75, 90\}$ degrees. 

\paragraph{Results.} \Cref{fig:angle_ablation} shows the resulting test loss and accuracy curves during training averaged over $5$ trials for each $\psi$. The test loss and accuracy curves are nearly identical for all $\psi \in \{0, 15, 30, 45\}$ degrees, although at $\psi = 45$, GD required slightly more iterations for the test loss to begin decreasing from its initial value. The performance noticeably deteriorates when $\psi > 45$. For $\psi \in \{60, 75\}$ degrees, the test loss still decreases from their initial values, but GD gets stuck in noticeably worse local minima compared to when $\psi < 45$. Finally, at $\psi = 90$, the initialized $\widetilde{\bm U}$ and $\widetilde{\bm V}$ are completely orthogonal to $\widetilde{\bm U}_{big}$ and $\widetilde{\bm V}_{big}$. The test loss stays at its initial value for all training epochs, and the low-rank MLP does not escape the random guessing stage. Overall, if $\widetilde{\bm U}$ / $\widetilde{\bm V}$ in the low-rank MLP are initialized to be closer to $\widetilde{\bm U}_{big}$ /  $\widetilde{\bm V}_{big}$ than their orthogonal complements, then the low-rank MLP can eventually nearly match the performance of the full-rank counterpart. Otherwise, GD gets stuck in progressively worse local minima as $\psi \to 90$ degrees.

% \subsection{Additional Details for CIFAR-10 Experiments}
% \label[appendix]{ssec:additional_cifar10_details}
% Here, we provide additional experimental details for \Cref{sssec:cifar10}. We trained all models on cross-entropy loss using SGD with momentum for $T = 250$ epochs. We set $\eta = 5 \times 10^{-3}$ with a cosine annealing scheduler, the batch size to $128$, the momentum to $0.9$, and weight decay to $5 \times 10^{-4}$.

\subsection{Width Ablation}
\begin{figure}[t]
    \centering
    \includegraphics[width=0.85\linewidth]{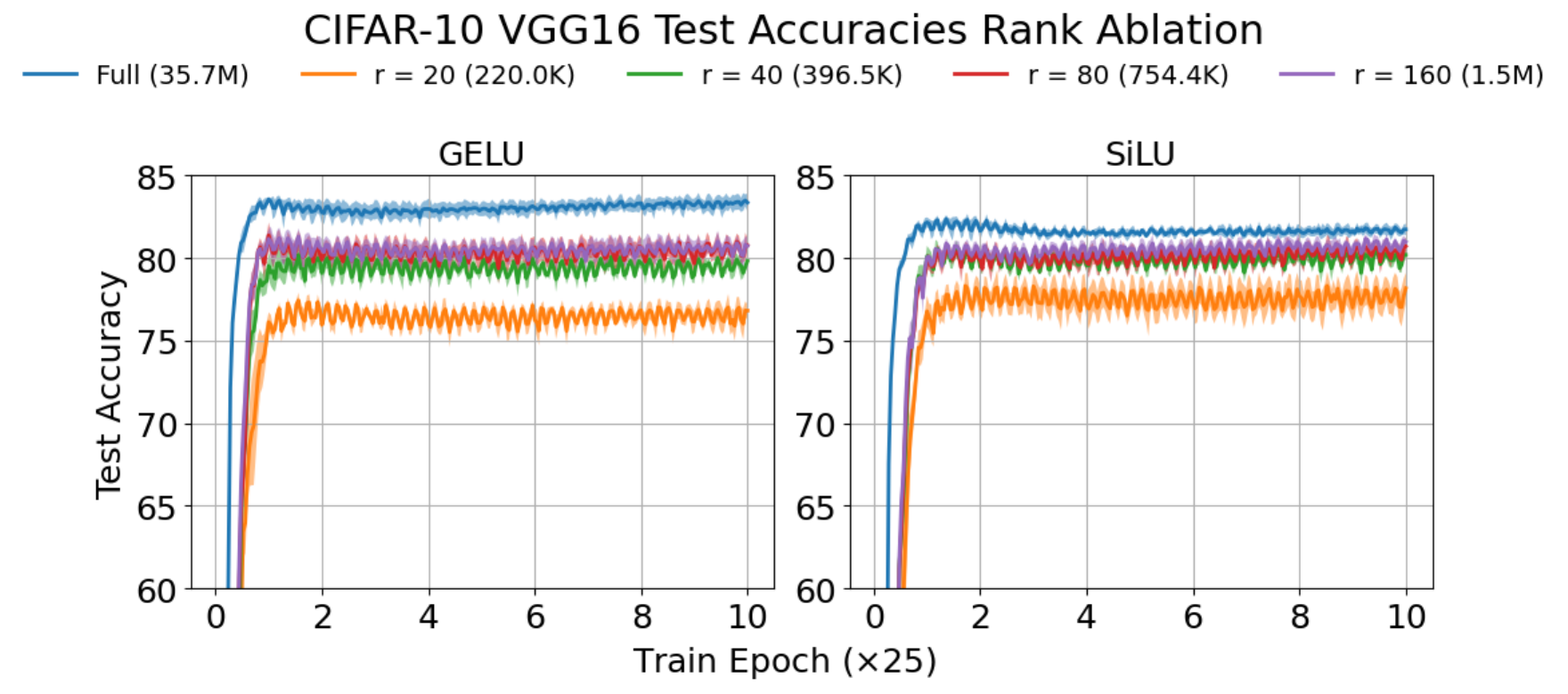}
    \caption{Increasing the width $r$ of the low-rank MLP initially leads to a noticeable increase in test accuracy, but further increasing $r$ leads to diminishing returns in test accuracy gains. Here, we only trained the classifier head in VGG-16, and kept the convolutional layers frozen at their ImageNet pre-trained weights.}
    \label{fig:rank_ablation}
\end{figure}

\label[appendix]{ssec:rank_ablation}
From the classifier-only results in \Cref{sssec:cifar10} (\Cref{fig:cifar10}), there is a $5 - 10\%$ gap in test accuracy between the low-rank MLP with $\mc{S}_{big}$ initialization (orange) and the fully parameterized MLP (blue). In this section, we investigate how much does this gap close by increasing the width $r$.

\paragraph{Experimental details.} We repeated the classifier-only training experiments as described in \Cref{sssec:cifar10} with both $\phi = \gelu$ and $\phi = \silu$, again using the exact same training algorithm (SGD with momentum), hyperparameters, and loss function (cross-entropy loss). However, now we vary the width $r \in \{2K, 4K, 8K, 16K\}$ in the low-rank MLP in \eqref{eq:low_rank_mlp}. We initialized $\widetilde{\bm U}$ and $\widetilde{\bm V}$ using the $\mc{S}_{big}$ initialization scheme as described in \Cref{sec:low_rank_param}. For $r > 2K$, we initialized the last $r - 2K$ columns of $\widetilde{\bm U}$ and $\widetilde{\bm V}$ to lie in the orthogonal complements of their first $2K$ columns. 

\paragraph{Results.} \Cref{fig:rank_ablation} shows the test accuracy at every training epoch for each $r$. Even going from $r = 2K$ to $r = 4K$ noticeably closes the gap between the low-rank and fully-parameterized MLPs, especially for the $\silu$ classifier. However, continuing to double $r$ further leads to diminishing returns in performance gains. Nevertheless, with the appropriate initialization on $\widetilde{\bm U}$ and $\widetilde{\bm V}$, a width of $r = \Theta(K)$ appears to be sufficient in achieving a test accuracy that is only about $2 - 3\%$ lower than that of the fully parameterized classifier head.

\fi

\end{document}